\documentclass{article} 
\usepackage{iclr2026_conference,times}

\usepackage{hyperref}
\usepackage{url}
\usepackage{dsfont}
\usepackage{amsmath,amssymb,amsfonts,amsthm}
\usepackage{graphicx}
\usepackage{todonotes}
\usepackage{marginnote}
\usepackage{booktabs}
\usepackage{multirow}
\usepackage{wrapfig}
\usepackage{xspace}
\usepackage{adjustbox}
\usepackage{xcolor}
\usepackage{tabularx}
\usepackage{colortbl}
\usepackage[most]{tcolorbox}
\usepackage{amsthm}
\usepackage[hang,flushmargin]{footmisc}
\newtheorem*{proposition*}{Proposition}
\newtheorem{theorem}{Theorem}
\newtheorem*{theorem*}{Theorem}
\newcommand{\multirowcell}[1]{\begin{tabular}[c]{@{}c@{}}#1\end{tabular}}
\newcommand{\HC}{\mathcal{H}}

\newcommand{\DC}{\mathcal{D}}

\DeclareMathOperator{\Var}{Var}
\newcommand{\EE}{\mathbb{E}}

\title{Don't Throw Away Your Beams: \\ Improving Consistency-based Uncertainties \\ in LLMs via Beam Search}

\author{Ekaterina Fadeeva\textsuperscript{1}\\
\And
Maiya Goloburda\textsuperscript{2}\\
\And
Aleksandr Rubashevskii\textsuperscript{2}\\
\And
Roman Vashurin\textsuperscript{2}\\
\And
Artem Shelmanov\textsuperscript{2}\\
\And
Preslav Nakov\textsuperscript{2}\\
\And
Mrinmaya Sachan\textsuperscript{1}\\
\And
Maxim Panov\textsuperscript{2}\\
}

\newcommand{\llama}{Llama 3.1 8B\xspace}
\newcommand{\gemma}{Gemma 3 4B\xspace}
\newcommand{\qwen}{Qwen 3 8B\xspace}

\newcommand{\llamab}{Llama 3.1 8B base\xspace}
\newcommand{\gemmab}{Gemma 3 4B base\xspace}
\newcommand{\qwenb}{Qwen 3 8B base\xspace}

\newcommand{\llamait}{Llama 3.1 8B instruct\xspace}
\newcommand{\gemmait}{Gemma 3 4B instruct\xspace}
\newcommand{\qwenit}{Qwen 3 8B instruct\xspace}

\def\BC{\mathcal{B}}
\def\bv{\mathbf{b}}
\def\xv{\mathbf{x}}
\def\yv{\mathbf{y}}
\def\tv{\mathbf{t}}
\def\uv{\mathbf{u}}
\def\vv{\mathbf{v}}
\def\EE{\mathbb{E}}

\iclrfinalcopy 
\begin{document}
\definecolor{TodoColor}{rgb}{1,0.7,0.6}
\definecolor{aquamarine}{rgb}{0.5, 1.0, 0.83}
\definecolor{some_blue}{rgb}{0.4, 0.4, 1.0}
\newcommand{\todonote}[3][]{\todo[color=#2,size=\scriptsize,fancyline,caption={},#1]{#3}}
\newcommand{\todox}[2][]{\todonote[#1]{TodoColor}{\textbf{TODO:} #2}}
\newcommand{\mrinmaya}[2][]{\todonote[#1]{pink}{\textbf{Mrinmaya:} #2}}
\newcommand{\rediska}[2][]{\todonote[#1]{aquamarine}{\textbf{rediska:} #2}}
\newcommand{\maxim}[2][]{\todonote[#1]{some_blue}{\textbf{Maxim:} #2}}
\newcommand{\Mrinmaya}[2][]{\mrinmaya[inline,#1]{#2}}

\newcommand{\TODO}[2][]{\todox[inline,#1]{#2}}
\newcommand{\TODOMARK}{\textcolor{black}{\sethlcolor{TodoColor} \small \hl{\textbf{TODO}}}\xspace}
\newcommand{\CITEME}{\textcolor{black}{\small \hl{\textbf{CITEME}}}\xspace}

\newcommand\comet[2][]{Comet$^{#2}_\textrm{#1}$\xspace}
\newcommand{\hrefEmail}[2]{\href{mailto:#1}{\color{black}{#2}}}
\newcommand{\whitezero}{\textcolor{white}{0}}
\newcommand{\offsetminus}{\hspace{-1.2mm}-}

\newcommand{\hlc}[2][yellow]{{%
    \colorlet{foo}{#1}%
    \sethlcolor{foo}\hl{#2}}%
}
\definecolor{FindingsColor}{gray}{0.85}
\newcommand{\hlfinding}[1]{\hlc[FindingsColor]{#1}}

\makeatletter\def\Hy@Warning#1{}\makeatother
\let\svthefootnote\thefootnote
\newcommand\blankfootnote[1]{%
  \let\thefootnote\relax\footnotetext{#1}%
  \let\thefootnote\svthefootnote%
}


\newcommand{\legend}[3]{
\null\hspace{3mm}
\makebox[22mm][l]{
    \textcolor[HTML]{#1}{
    \rule[3pt]{10pt}{1.5pt}
    \hspace{-13pt}
    \raisebox{0.5pt}{\scalebox{1.5}{$\bullet$}}}
    #2
}
\makebox[8mm][l]{#3}
}

\newcommand{\legendShort}[2]{
\null\hspace{1mm}
\makebox[21mm][l]{
    \textcolor[HTML]{#1}{
    \rule[3pt]{10pt}{1.5pt}
    \hspace{-13pt}
    \raisebox{0.5pt}{\scalebox{1.5}{$\bullet$}}}
    #2
}
}

\newcommand\anonymized{\texttt{\bf[anonymized]}\xspace}

\newtcolorbox{todobox}[1][]{colback=orange!10,
  colframe=orange!70!black,
  title=TODO,
  fonttitle=\bfseries,
  coltitle=black,
  #1}
  
\newcommand\freefootnote[1]{%
  \let\thefootnote\relax%
  \footnotetext{#1}%
  \let\thefootnote\svthefootnote%
}

\maketitle

\vspace{-1.5em}
\begin{center}
    \hfill \textsuperscript{1} ETH Zurich \hfill \textsuperscript{2} MBZUAI \hfill\null
\end{center}
\vspace{0.5em}

\begin{abstract}
  Consistency-based methods have emerged as an effective approach to uncertainty quantification (UQ) in large language models. These methods typically rely on several generations obtained via multinomial sampling, measuring their agreement level. However, in short-form QA, multinomial sampling is prone to producing duplicates due to peaked distributions, and its stochasticity introduces considerable variance in uncertainty estimates across runs. We introduce a new family of methods that employ beam search to generate candidates for consistency-based UQ, yielding improved performance and reduced variance compared to multinomial sampling. We also provide a theoretical lower bound on the beam set probability mass under which beam search achieves a smaller error than multinomial sampling. We empirically evaluate our approach on six QA datasets and find that its consistent improvements over multinomial sampling lead to state-of-the-art UQ performance. 
\end{abstract}

\freefootnote{Correspondence to: \texttt{efadeeva@ethz.ch}, \texttt{maxim.panov@mbzuai.ac.ae}}
\freefootnote{Code available at: \url{https://github.com/IINemo/lm-polygraph/tree/beam-uncertainty}}


\section{Introduction}
\label{introduction}

  \begin{wrapfigure}{r}{0.5\textwidth}
    \centering
    \includegraphics[width=\linewidth, trim=0 2em 0 4em]{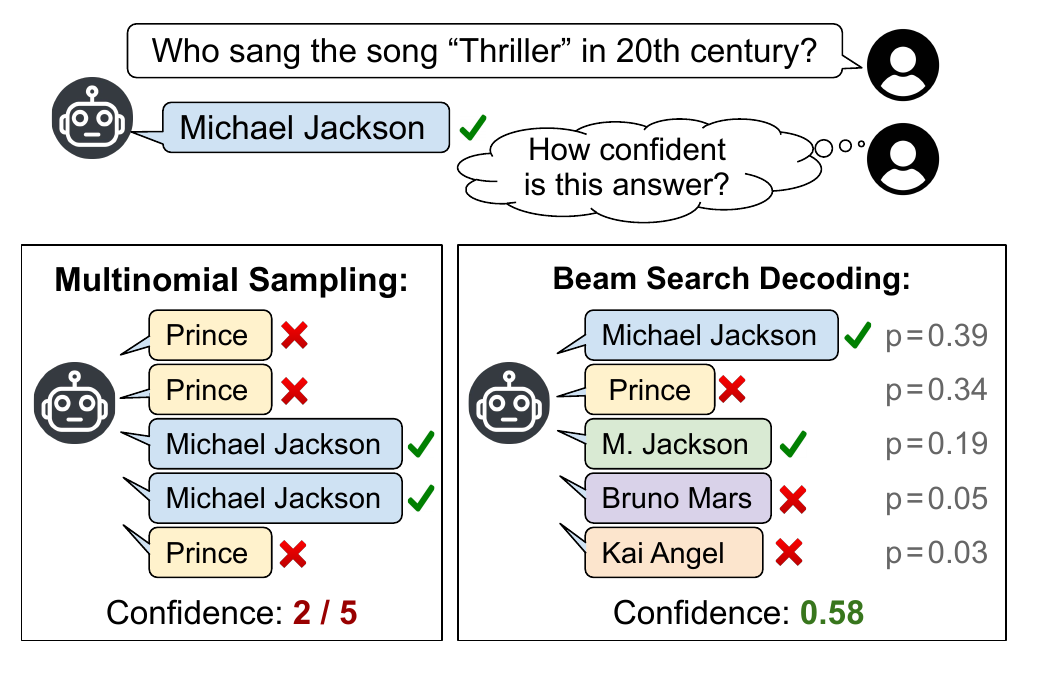}
    \caption{Beam Search vs Multinomial Sampling. Sampling produces multiple identical generations resulting in noisy confidence estimate, while beam search covers top answers from LLM distribution resulting in a better confidence score.}
    \label{fig:sampling_vs_beamsearch}
  \end{wrapfigure}
  
  Today, large language models (LLMs) are increasingly being adapted in various safety-critical domains, including medicine~\citep{medicine}, education~\citep{education}, and law~\citep{law}. This rapid adoption has led to a growing body of work focused on the assessment of the quality and reliability of LLM outputs. An important research direction in this field is Uncertainty Quantification (UQ; \citealp{Xiao2019,Baan2023UncertaintyIN, xia-etal-2025-survey}), which measures the LLM's confidence in their responses.

  UQ methods can be categorized into several distinct groups. These include information-based methods that rely on token likelihoods produced by the LLM~\citep{fomicheva-etal-2020-unsupervised}; verbalization approaches that prompt models to provide a confidence score~\citep{tian-etal-2023-just}; density-based methods that utilize embeddings~\citep{yoo-etal-2022-detection}; and last but not least, consistency-based measures that evaluate agreement between sampled outputs~\citep{lin2023generating}.

  Consistency-based UQ methods are of particular interest due to not only their strong performance but also their applicability to black-box settings~\citep{vashurin-etal-2025-benchmarking}. 
  
  Moreover, in white-box settings too, it was shown that combining information-based and consistency-based methods yields state-of-the-art performance for a variety of tasks~\citep{kuhn2023semantic,duan-etal-2024-shifting}. A key component of these methods is sampling, which serves as a practical means of approximating the full probability space of all potential model outputs.

  Most existing UQ approaches rely on multinomial sampling from the model's output distribution. However, in short-form QA, multinomial sampling is prone to producing similar or even identical generations, due to its bias towards higher-probability tokens during decoding; see Figure~\ref{fig:sampling_vs_beamsearch}. Furthermore, since each run produces a different set of candidate outputs, sample-based uncertainty estimates exhibit high variance, undermining their robustness. This limits their effectiveness as a representation of the full output space, especially since, for computational efficiency, studies typically rely on a small number of samples. 

  To address this problem, we propose computing output consistency based on samples generated using beam search. Beam search facilitates the exploration of alternative decoding paths, which in turn allows one to generate distinct candidate outputs that better capture the model's output space in short-form QA.
  Our approach includes weighting beam search outputs 
  by their probabilities rather than uniformly, thereby preventing the overrepresentation of low-probability outputs.
  Particularly, when beam search is employed for decoding, uncertainty estimates are obtained at essentially no additional cost.
  We show that replacing multinomial sampled outputs with those generated via beam search improves the robustness and accuracy of existing consistency-based methods, as well as hybrid methods relying on both output consistency and token likelihoods.

  Our main \textbf{contributions} are as follows.
  \begin{itemize}
    \item We identify key limitations of existing consistency-based uncertainty quantification methods based on multinomial sampling; see Section~\ref{sec:background}.

    \item We propose a new family of UQ methods that employ an importance-weighted estimator of consistency-based uncertainty with beam search output candidates; see Section~\ref{sec:methodology}.

    \item We provide a distribution-free sufficient condition ensuring that the beam-weighted estimator achieves a lower error than the expected error of the multinomial sampler; see Section~\ref{subsec:theor_analysis}.

    \item We show that applying a beam search-based estimator to existing consistency-based UQ approaches improves their performance on short-form QA tasks, achieving state-of-the-art results; see Section~\ref{sec:experiments}.
  \end{itemize}


\section{Background and Motivation}
\label{sec:background}

\subsection{Language Model Decoding}
  Autoregressive LLMs produce text sequentially, generating one token at a time. At each step $i$, the model samples a token $y_i \sim p(\cdot \mid \yv_{<i}, \xv)$, where $\yv_{<i}$ denotes the sequence of previously generated tokens. The probability of generating an output sequence $\yv$ is:
  \begin{equation}
    p(\yv \mid \xv) = \prod_{i=1}^{|\yv|} p(y_i \mid \yv_{<i}, \xv).
  \label{eq:autoreg}
  \end{equation}
  At each step, the model outputs a probability distribution over the entire vocabulary $\mathcal{V}$ conditioned on the prompt $\xv$ and the partial sequence $\yv_{<i}$.

\paragraph{Decoding strategies.}
  Since the model defines a probability distribution, a concrete output must be obtained at inference time by applying a decoding strategy. Common decoding strategies include: (i) greedy decoding that selects maximum probability tokens at each step; (ii) multinomial sampling where tokens are drawn according to $p(y_i \mid \yv_{<i}, \xv)$; and (iii) beam search, which maintains the top-$k$ most likely partial sequences at each step. Several other variants of decoding approaches have been proposed, such as top-$p$ nucleus sampling or temperature scaling~\citep{Holtzman2020The, Vijayakumar2018}. Each decoding strategy offers different trade-offs between output quality and diversity.

\subsection{Uncertainty Quantification for LLMs}

  The objective of uncertainty quantification is to measure the level of uncertainty introduced by LLM when generating output sequence $\yv_*$ conditioned on input sequence $\xv$, denoted by $U(\yv_* \mid \xv)$. 
  Existing approaches to UQ can be broadly categorized into three main groups.

  \begin{wrapfigure}[21]{r}{0.45\textwidth} 
    \centering
    \includegraphics[width=\linewidth, trim=0 1em 0 1em]{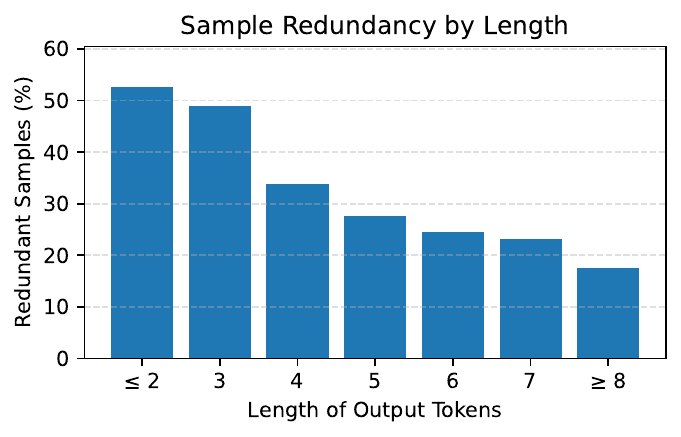}
    \caption{Mean percentage of redundant samples (i.e., outputs already seen among earlier generations) as a function of greedy output length. Results were obtained from 2,000 questions from the TriviaQA dataset using the \gemmab model and 10 candidate generations. Redundancy is especially high for short answers, leading to wasted computation.}
    \label{fig:redundancy}
  \end{wrapfigure}
  
  \textit{Information-based} methods rely on a single forward pass of the model and compute statistics over the token-level probability distributions to quantify uncertainty. 
  Examples include Sequence Probability, Mean Token Entropy, Perplexity~\citep{fomicheva-etal-2020-unsupervised}, and CCP~\citep{fadeeva-etal-2024-fact}. 

  \textit{Reflexive} methods query the model directly about its confidence in a generated answer using specially designed prompts. 
  A representative example is \textit{P(True)}~\citep{kadavath2022language}, which measures the probability that the model outputs ``True'' when asked whether its generated answer $\yv_*$ is correct.

  \textit{Sampling-based} methods draw multiple samples from the model's output distribution and evaluate their semantic or lexical similarity to assess uncertainty. 
  Lexical Similarity~\citep{fomicheva-etal-2020-unsupervised} computes mean pairwise similarity between generated texts; other examples include Semantic Entropy~\citep{kuhn2023semantic}, SAR~\citep{duan-etal-2024-shifting}, and black-box uncertainty measures from~\citep{lin2023generating}.

\paragraph{Consistency-based UQ methods.}
  A notable subset of sampling-based methods is \textit{consistency-based UQ}~\citep{vashurin2025uncertainty}. 
  These methods estimate uncertainty \emph{with respect to a particular generated output} $\yv_* \sim p(\cdot \mid \xv)$, rather than the overall uncertainty of the model's predictive distribution for the input $\xv$. 
  This distinction makes consistency-based UQ particularly suited for evaluating confidence in a specific prediction rather than overall model uncertainty, and \citet{vashurin2025uncertainty} empirically demonstrate that such methods outperform other sampling-based approaches in practice.

  Let us consider the most straightforward consistency–based method for predictive uncertainty quantification: measuring how semantically different alternative generations are from the produced answer $\yv_*$. We refer to this score as \emph{Dissimilarity} and formalize it as the expected semantic dissimilarity between the produced answer $\yv_*$ and \emph{all} potential alternatives drawn from the model:
  \begin{equation}
    U_D(\yv_* \mid \xv) = \EE_{\yv \sim p(\cdot \mid \xv)}\bigl[1 - s(\yv, \yv_*)\bigr].
  \label{eq:uncertainty_def}
  \end{equation}
  Here, $s(\yv', \yv'') \in [0, 1]$ is a function that measures semantic similarity between two generations \(\yv'\) and \(\yv''\). A higher value of $U_D(\yv_* \mid \xv)$ indicates lower consistency between the chosen answer and alternative candidate outputs, and thus reflects greater predictive uncertainty. 

  The corresponding Monte Carlo estimator introduced by~\citep{lin2023generating} draws $M$ i.i.d.\ samples $\yv^{(1)}, \dots, \yv^{(M)} \sim p(\cdot \mid \xv)$ and computes uncertainty in the following way:
  \begin{equation}
    \widehat{U}_{D}^{MC}(\yv_* \mid \xv) = \frac{1}{M}\sum_{i=1}^M \bigl(1 - s(\yv^{(i)}, \yv_*)\bigr).
  \label{eq:mc_estimator}
  \end{equation}

\paragraph{Challenges of consistency-based UQ methods.}
  
  A natural intuition is that, for consistency-based methods, samples should be generated in a distinct, high-probability, and stable manner. Most existing methods use multinomial sampling, which, especially for shorter generations and small sample sizes, does not satisfy these criteria.  

  Figure~\ref{fig:redundancy} shows the effect of multinomial sampling on the percentage of duplicates depending on the length of generations.
  The resulting samples contain many duplicates, with the issue being particularly pronounced for shorter generations, where 30–50\% of the outputs are duplicates. 
  
  This not only contributes to wasted computation, but also leads to high variance estimates. Moreover, drawing $M$ full generations solely for uncertainty estimation can be costly. 


  Thus, while multinomial sampling is widely used, it does not best serve the goals of consistency-based uncertainty estimation.


\section{Uncertainty Quantification based on Consistency of Beam Search Candidates}
\label{sec:methodology}
  To address the problems outlined above, we propose to utilize an alternative decoding strategy for generating candidate outputs: beam search. Beam search (i) guarantees distinct candidate outputs, (ii) reduces variance (see Section~\ref{subsec:theor_analysis}) and (iii) provides uncertainty estimates essentially ``for free'' as the beam already provides a distribution over candidate outputs.

\subsection{Replacing Multinomial Sampling}
  A simple way to approximate dissimilarity from beam-generated candidates would be to reuse equation~\eqref{eq:mc_estimator}, treating the beam outputs as if they were drawn uniformly. While this offers a plausible alternative, treating the candidates produced by beam search in a uniform manner would overemphasize lower-probability outputs. To better reflect the model distribution while avoiding repeated multinomial draws, we form a probability-weighted estimator over the beam set.

  For this purpose, we use beam search with width $M$ to obtain distinct candidates $\BC_M(\xv)=\{\bv^{(1)}, \dots, \bv^{(M)}\}$ and their sequence probabilities $\{p(\bv^{(i)} \mid \xv)\}_{i=1}^M$. To perform an estimation of $U_D(\yv_* \mid \xv)$ in equation~\eqref{eq:uncertainty_def} with the help of samples $b^{(i)}$, one needs to perform importance weighting. Thus, we define the restricted (top-$M$) normalized masses $w_i$ as:
  \begin{equation}
    w_i = \frac{p(\bv^{(i)} \mid \xv)}{\sum_{j=1}^{M} p(\bv^{(j)} \mid \xv)}, \qquad i = 1, \dots, M.
  \label{eq:restricted-mass}
  \end{equation}

  The resulting importance-weighted estimator of equation~\eqref{eq:uncertainty_def} is
  \begin{equation}
    \widehat{U}_{D}^{b}(\yv_* \mid \xv) = \sum_{i=1}^M w_i \bigl(1 - s(\bv^{(i)}, \yv_*)\bigr).
  \label{eq:beam-estimator}
  \end{equation}
  This top-$M$ truncation introduces a small bias relative to full multinomial sampling but typically reduces variance and duplication on peaked distributions, yielding more stable estimates per unit budget. In the next section we are going to explore the benefits of beam search-based estimator $\widehat{U}_{D}^{b}(\yv_* \mid \xv)$ from a theoretical perspective.

\subsection{Theoretical Analysis}
\label{subsec:theor_analysis}

  We compare the multinomial Monte Carlo estimator $\widehat{U}_D^{MC}$~\eqref{eq:mc_estimator} with the beam-weighted estimator $\widehat{U}_D^{b}$~\eqref{eq:beam-estimator} for the dissimilarity $U_D$ defined in equation~\eqref{eq:uncertainty_def}.

  \medskip
  \begin{theorem}[Comparison condition for beam-weighted and Monte Carlo estimators]
  \,\\
  \label{theorem:comp_condition}
    Let $\BC_M(\xv)=\{\bv^{(1)},\dots,\bv^{(M)}\}$ be the beam set, $m_{\BC} = \sum_{i=1}^M p(\bv^{(i)}\mid\xv)$ be its total probability mass, and define $\mu_{\BC}$ and $\mu_{\overline{\BC}}$ as dissimilarity inside and outside the beam set $\BC_M$ correspondingly:
    \begin{equation*}
      \mu_{\BC} = \EE_{\yv \sim p(\cdot \mid \xv)} \left[1 - s(\yv, \yv_*) \mid \yv \in \BC_M(\xv)\right],
      \qquad
      \mu_{\overline{\BC}} = \EE_{\yv \sim p(\cdot \mid \xv)} \left[ 1 - s(\yv, \yv_*) \mid \yv \notin \BC_M(\xv)\right].
    \end{equation*}

    Then the beam-weighted estimator $\widehat{U}_D^{b}$ achieves smaller mean-squared error than the Monte Carlo estimator $\widehat{U}_D^{MC}$ whenever
    \begin{equation}
      (1 - m_{\BC}) \bigl|\mu_{\BC} - \mu_{\overline{\BC}}\bigr| < \sigma / \sqrt{M},
      \label{eq:general_condition}
    \end{equation}
    where $\sigma^2 = \Var_{\yv \sim p(\cdot \mid \xv)} (1 - s(\yv, \yv_*))$.
    The corresponding distribution-free sufficient condition is
    \begin{equation}
      m_{\BC} > 1 - \tfrac{1}{2\sqrt{M}}.
      \label{eq:distribution_free_condition}
    \end{equation}
    \end{theorem}
    \begin{proof}
    The Monte Carlo estimator averages $M$ i.i.d.\ samples $\yv^{(i)} \sim p(\cdot \mid \xv)$, so it is unbiased with
    $\EE[\widehat{U}_D^{MC}] = U_D(\yv_* \mid \xv)$ and
    $\mathrm{MSE}(\widehat{U}_D^{MC}) = \Var(\widehat{U}_D^{MC}) = \sigma^2/M$.
    By Popoviciu's inequality, any random variable supported on $[0,1]$ has variance at most $1/4$, hence $\sigma^2 \le 1/4$.

    By the law of total expectation, the true dissimilarity $U_D$ decomposes as:
    \begin{equation*}
      U_D(\yv_* \mid \xv) = m_{\BC}\mu_{\BC} + (1-m_{\BC})\mu_{\overline{\BC}},
      \qquad
      \widehat{U}_D^{b} = \mu_{\BC},
    \end{equation*}
    so squared error of the beam-weighted estimator $\widehat{U}_D^{b}$ is deterministic:
    \begin{equation*}
      \mathrm{SE}(\widehat{U}_D^{b}) = \bigl(\widehat{U}_D^{b} - U_D\bigr)^2 = (1 - m_{\BC})^2 \bigl(\mu_{\BC} - \mu_{\overline{\BC}}\bigr)^2.
    \end{equation*}
    Beam-weighted estimation is therefore more accurate whenever
    \begin{equation*}
      (1 - m_{\BC})^2 \bigl(\mu_{\BC} - \mu_{\overline{\BC}}\bigr)^2 < \sigma^2 / M,
    \end{equation*}
    which yields the stated condition~\eqref{eq:general_condition}.
    A distribution-free bound~\eqref{eq:distribution_free_condition} follows from $\bigl|\mu_{\BC} - \mu_{\overline{\BC}}\bigr| \le 1$ and $\sigma^2 \le 1/4$.
  \end{proof}

  \begin{wrapfigure}[14]{r}{0.5\textwidth}
    \centering
    \includegraphics[width=\linewidth, trim=0 2.5em 0 3.5em]{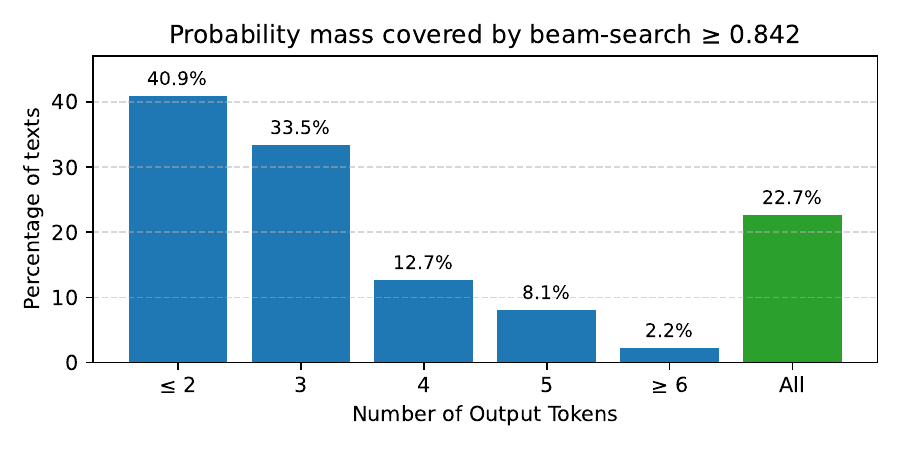}
    \caption{Percentage of texts meeting the sufficient condition (Theorem~\ref{theorem:comp_condition}). Results are based on 2,000 TriviaQA questions, \gemmab and $M=10$. The green ``All'' bar shows the overall percentage across all lengths.}
    \label{fig:beam_mass_threshold}
  \end{wrapfigure}

  From Theorem~\ref{theorem:comp_condition}, beam-weighted estimator is more accurate than Monte Carlo estimator whenever total beam probability mass $m_{\BC}$ exceeds $1-\frac{1}{2\sqrt{M}}$.
  For $M=10$, the threshold is $m_{\BC} > 0.842$. 
  Thus, when the top-$10$ beam hypotheses capture at least $\sim84\%$ of the model's probability mass, beam search provides a lower-error estimator than multinomial sampling with the same sample budget.

  In practice, \textit{this condition is frequently satisfied}. 
  On the TriviaQA dataset, Figure~\ref{fig:beam_mass_threshold} shows that $22.7\%$ of examples meet the sufficient condition overall, and up to $30$-$40\%$ for very short generations ($\le3$ output tokens), where probability mass is highly concentrated on the top beams. 
  When the inside-outside gap $\delta = |\mu_{\BC}-\mu_{\overline{\BC}}| < 1$, the break-even requirement~\eqref{eq:general_condition} relaxes to $(1-m_{\BC}) \delta < \sigma / \sqrt{M}$, allowing beam search to outperform even when $m_{\BC} < 0.842$. 
  Although $\mu_{\overline{\BC}}$ is not directly computable due to the combinatorial output space, our experiments consistently show beam search outperforming multinomial sampling, suggesting that $\delta$ is modest in practice and that the effective threshold is often lower than $0.842$.

\subsection{Adapting Other UQ Methods to Beam Search}
  In a similar manner, other consistency-based methods can be adapted to utilize beam search-based samples in their formulation. 

\paragraph{Eccentricity.}
  Eccentricity is a method introduced by~\citet{lin2023generating}. Unlike dissimilarity, which uses only the similarities between the produced answer $\yv_*$ and each alternative sample, Eccentricity aggregates the \emph{joint} pairwise relationships among all samples.

  In this method, we first construct a similarity matrix of size $(M + 1) \times (M + 1)$ for the $M$ samples and the produced answer $\yv^{(M + 1)} = \yv_*$:
  \begin{equation}
    W_{ij} = s\bigl(\yv^{(i)}, \yv^{(j)}\bigr), \quad 1 \le i, j \le M + 1.
  \end{equation}
  Then we compute the degree matrix $D$:
  \begin{equation}
    D_{ij} =
    \begin{cases}
      \sum\limits_{k=1}^{M + 1} W_{ik}, & i = j, \\
      0, & i \neq j,
    \end{cases}
  \end{equation}
  and obtain the eigendecomposition of the Graph Laplacian $L = I - D^{-1/2} W D^{-1/2}$, yielding eigenpairs $\{\lambda_i, \uv_i\}_{i=1}^{M + 1}$. 
  
  Smaller eigenvalues (close to zero) capture meaningful semantic structure, whereas larger eigenvalues tend to reflect noise. We therefore retain the eigenvectors whose eigenvalues satisfy $\lambda_i < \alpha$, yielding $K$ vectors in total; $K$ is thus determined by the threshold $\alpha > 0$.

  Semantic embeddings are formed as $\vv_j = [\uv_{1j}, \uv_{2j}, \dots, \uv_{Kj}]$. For $1 \le j \le M$, $\vv_j$ represents the embedding of $\yv^{(j)}$, and $\vv_* = \vv_{M + 1}$ corresponds to $\yv_*$. The confidence score is the distance between the embedding of the produced answer and the mean embedding of the samples:
  \begin{equation}
    \widehat{U}_{Ecc}(\yv_* \mid \xv) = \biggl\|\vv_{*} - \frac{1}{M}\sum_{i=1}^M \vv_i \biggr\|_2^2,
  \label{eq:eccentricity}
  \end{equation}
  where higher values indicate higher uncertainty.

  With beam-generated candidates, we weight embeddings by the normalized masses $w_i$ from equation~\eqref{eq:restricted-mass} to better reflect the model distribution while avoiding duplicate generations:
  \begin{equation}
    \widehat{U}_{Ecc}^{b}(\yv_* \mid \xv) = \biggl\|\vv_{*}^{b} - \sum_{i=1}^M w_i \vv_{i}^{b}\biggr\|_2^2.
  \end{equation}

\paragraph{CoCoA.}
  A white-box approach CoCoA~\citep{vashurin2025uncertainty} combines a model probabilities-based uncertainty with the sample-consistency signal: 
  \begin{equation}
    \widehat{U}_{CoCoA}(\yv_* \mid \xv) = u(\yv_* \mid \xv) \cdot \widehat{U}_D^{MC} (\yv_* \mid \xv) = u(\yv_* \mid \xv) \cdot \frac{1}{M} \sum\limits_{i=1}^M \bigl(1 - s(\yv^{(i)}, \yv_*)\bigr),
  \end{equation}
  where $u(\yv \mid \xv)$ is a model-based uncertainty measure for the sequence (e.g., $-\log p(\yv \mid \xv)$). 

  For a beam-weighted estimator, we utilize~\eqref{eq:beam-estimator} as sample-consistency signal:
  \begin{equation}
    \widehat{U}_{CoCoA}^{b}(\yv_* \mid \xv) = u(\yv_* \mid \xv) \cdot \widehat{U}_D^{b} (\yv_* \mid \xv).
  \end{equation}

\paragraph{Eigenvectors Dissimilarity.}
  Both Dissimilarity and Eccentricity produce confidence scores for the generated answer $\yv_*$. Dissimilarity compares $\yv_*$ to each sample using the base similarity function $s$, while Eccentricity measures the distance from $\yv_*$ to the centroid in the Laplacian embedding space; see equation~\eqref{eq:eccentricity}. To bridge these views, we measure dissimilarity within the embedding space itself, averaging the distances from the embedding of $\yv_*$ to the embeddings of individual samples. This retains the joint-pairwise smoothing of Eccentricity and also reflects the variance among samples, rather than only the centroid. The sampling-based estimate is
  \begin{equation}
    \widehat{U}_{EigVecD}(\yv_* \mid \xv) = \frac{1}{M} \sum \limits_{i=1}^M \bigl\|\vv_{*} - \vv_i\bigr\|_2^2,
  \end{equation}
  and the beam-guided, probability-weighted version is
  \begin{equation}
    \widehat{U}_{EigVecD}^{b}(\yv_* \mid \xv) = \sum\limits_{i=1}^M w_i \bigl\|\vv_{*}^{b} - \vv_i^{b}\bigr\|_2^2,
  \end{equation}
  where the embeddings $\vv_i$ (and $\vv_i^{b}$) are obtained from the Graph Laplacian as in Eccentricity, and $w_i$ are the normalized masses from equation~\eqref{eq:restricted-mass}. This estimator increases both when $\yv_*$ moves away from the bulk and when the samples themselves are more dispersed; by contrast, Eccentricity focuses on the single distance to the weighted centroid.



\section{Experiments}
\label{sec:experiments}

\subsection{Experimental Setup}

  \begin{table*}[t]

\caption{Test dataset settings and statistics.}

\centering
\small
\begin{tabular}{lcccccc}
\toprule
 &
\multicolumn{2}{c}{Closed-Book QA} &
\multicolumn{2}{c}{Open-Book QA} &
\multicolumn{2}{c}{Multiple Choice} \\
\cmidrule(lr){2-3}\cmidrule(lr){4-5}\cmidrule(lr){6-7}
 & TriviaQA & \multirowcell{Web\\Questions} & CoQA & HotpotQA & \multirowcell{Common\\senceQA} & \multirowcell{ARC-\\Challenge} \\
\midrule
\# Questions & 2000 & 1490 & 2000 & 2000 & 1221 & 447 \\
\midrule
\# few-shot examples & 5 & 5 & all preceding & 0 & 2 & 2 \\
\midrule
\multirowcell{Max new tokens} & 20 & 20 & 20 & 20 & 10 & 20 \\
\bottomrule
\end{tabular}
\label{tab:test_datasets}
\end{table*}

\paragraph{Datasets.}
  We evaluate our approach on six QA datasets in total. Those include two closed-book datasets: \textit{TriviaQA}~\citep{joshi-etal-2017-triviaqa} and \textit{Web Questions}~\citep{berant-etal-2013-semantic}, two open-book datasets: \textit{CoQA}~\citep{reddy-etal-2019-coqa} and \textit{HotpotQA}~\citep{yang2018hotpotqa} and two multiple-choice datasets: \textit{CommonsenceQA}~\citep{talmor-etal-2019-commonsenseqa} and \textit{ARC-Challenge}~\citep{allenai:arc}. For each dataset, we randomly sampled several questions from the test set. The statistics for those datasets are available in Table~\ref{tab:test_datasets}. Prompt details and examples of questions are provided in Appendix~\ref{appendix:prompts}.

\paragraph{Models.}
  We use base and instruct versions of 3 models: \gemma~\citep{gemmateam2025gemma3technicalreport}, \llama~\citep{DBLP:journals/corr/abs-2407-21783}, and \qwen~\citep{qwen3technicalreport}.





    
    
    
    
    
    





\begin{table*}[t]

\caption{Summary of baseline UQ methods.}

\centering
\small

    \begin{tabular}{cl}
    
    \toprule
    Category & Uncertainty Quantification Method \\
    \midrule
    
    \multirow{4}{*}{\multirowcell{Information-based}} & Sequence Probability (Prob) \\
    & Mean Token Entropy (MTE) \\
    & Perplexity \\
    & CCP \citep{fadeeva-etal-2024-fact}  \\
    
    \midrule
    
    Reflexive & P(True)~\citep{kadavath2022language} \\
    
    \midrule
    
    \multirow{5}{*}{\multirowcell{Sampling-based}} & Semantic Entropy~\citep{kuhn2023semantic} \\
    & Shifting Attention to Relevance (SAR)~\citep{duan-etal-2024-shifting} \\
    & Lexical Similarity ~\citep{fomicheva-etal-2020-unsupervised} \\
    & Sum of Eigenvalues of Laplacian (EigValLaplacian)~\citep{lin2023generating} \\
    & Number of Semantic Sets (NumSemSets)~\citep{lin2023generating} \\

    \bottomrule

\end{tabular}

\label{tab:lmpoly_methods}

\end{table*}

\paragraph{Metrics.}
  Following best uncertainty benchmarking practices~\citep{vashurin-etal-2025-benchmarking}, we adopt the Prediction–Rejection Ratio (PRR) as our primary evaluation metric.
  Consider a test dataset $\DC = \{(\xv_j, \tv_j)\}$, where $\tv_j$ denotes target output. Then, we can obtain an output $\yv_{j}^*$ generated by an LLM for input $\xv_j$ and the associated uncertainty score $u_j = U(\yv_j^* \mid \xv_j)$.
  
  Based on these we can build rejection curve that captures how the average quality $Q(\yv_j^*, \tv_j)$ over all $\{(\yv_j^*, \tv_j)\colon u_j < \tau\}$ changes  with the rejection threshold $\tau$. An oracle rejection curve can be defined by substituting $u_j = - Q(\yv_j^*, \tv_j)$, giving best possible rejection order where lowest-quality outputs are rejected first. A baseline for rejection can be obtained by rejecting outputs uniformly at random. PRR is then defined as the ratio of the area between UQ rejection curve and a random rejection baseline to the area between oracle rejection and the same random baseline:
  \begin{equation}
    \label{eq:prr}
    PRR = \frac{\text{AUC}_{\text{unc}}-\text{AUC}_{\text{rnd}}}{\text{AUC}_{\text{oracle}}-\text{AUC}_{\text{rnd}}}.
  \end{equation}
  A higher PRR indicates a more effective uncertainty score. Following~\cite{vashurin-etal-2025-benchmarking}, we use AlignScore~\citep{zha-etal-2023-alignscore} as the quality metric $Q$.
  While PRR serves as our main evaluation measure, we additionally report ROC-AUC and PR-AUC in Appendix~\ref{appendix:rocauc_and_prauc}.

\paragraph{Baselines.}
  We evaluate four main methods, Dissimilarity, Eccentricity, Eigenvectors Dissimilarity, and CoCoA, under multinomial sampling and their beam-guided, probability-weighted variants. For CoCoA, we consider both \textit{CocoaMSP} based on unnormalized log-probability: 
  \begin{equation}
      u(\yv_* \mid \xv) = -\log p(\yv_* \mid \xv),
  \end{equation}
  and \textit{CocoaPPL} based on perplexity: 
  \begin{equation}
      u(\yv_* \mid \xv) = -\tfrac{1}{|\yv_*|} \log p(\yv_* \mid \xv).
  \end{equation}

  In addition, we compare against several state-of-the-art UQ baselines summarized in Table~\ref{tab:lmpoly_methods}, using implementations from LM-Polygraph~\citep{fadeeva-etal-2023-lm}. The simplest baseline, \textit{Sequence Probability}, calculates $-\log p(\yv_* \mid \xv)$. For detailed descriptions of other methods see Appendix~\ref{sec:appendix_methods}.

  All experiments use $M=10$ candidates for both multinomial sampling and beam search. We adopt the entailment probability from the DeBERTa-large model fine-tuned on the MNLI task~\citep{he2021deberta} for similarity function $s$, following~\cite{lin2023generating}.

\subsection{Results and Discussion}

  \begin{table}[t!]
    \caption{PRR ($\uparrow$ is better) averaged over 6 datasets. For each model, the top-1 method is \textbf{bold} and the second-best is \underline{underlined}. For beam-guided variants, we mark $\uparrow$ when the variant improves over its original multinomial-sampling counterpart.}
    \label{tab:qa_mean_prr}
    \centering
    \small
    \resizebox{0.85\textwidth}{!}{
      \begin{tabular}{lcccccc}
\toprule
Method & \multirowcell{Llama 3.1 8B\\base} & \multirowcell{Llama 3.1 8B\\instruct} & \multirowcell{Gemma 3 4B\\base} & \multirowcell{Gemma 3 4B\\instruct} & \multirowcell{Qwen 3 8B\\base} & \multirowcell{Qwen 3 8B\\instruct} \\
\midrule
 \rowcolor{gray!20}
 \multicolumn{7}{c}{\textit{Baseline UQ Methods}} \\
\midrule
Prob & .410 \tiny{$\pm$ .019} & .344 \tiny{$\pm$ .031} & .471 \tiny{$\pm$ .023} & .292 \tiny{$\pm$ .022} & .376 \tiny{$\pm$ .03} & .289 \tiny{$\pm$ .067} \\
MTE & .422 \tiny{$\pm$ .016} & .364 \tiny{$\pm$ .026} & .476 \tiny{$\pm$ .022} & .317 \tiny{$\pm$ .028} & .407 \tiny{$\pm$ .032} & .297 \tiny{$\pm$ .064} \\
Perplexity & .452 \tiny{$\pm$ .02} & .323 \tiny{$\pm$ .027} & .525 \tiny{$\pm$ .024} & .288 \tiny{$\pm$ .025} & .372 \tiny{$\pm$ .03} & .276 \tiny{$\pm$ .058} \\
CCP & .401 \tiny{$\pm$ .02} & .364 \tiny{$\pm$ .029} & .492 \tiny{$\pm$ .022} & .331 \tiny{$\pm$ .026} & .355 \tiny{$\pm$ .034} & .291 \tiny{$\pm$ .06} \\
SAR & .352 \tiny{$\pm$ .02} & .385 \tiny{$\pm$ .029} & .386 \tiny{$\pm$ .026} & .239 \tiny{$\pm$ .024} & .363 \tiny{$\pm$ .033} & .292 \tiny{$\pm$ .052} \\
P(True) & .015 \tiny{$\pm$ .023} & .072 \tiny{$\pm$ .03} & .093 \tiny{$\pm$ .026} & -.096 \tiny{$\pm$ .024} & .110 \tiny{$\pm$ .03} & -.114 \tiny{$\pm$ .055} \\
Semantic Entropy & .414 \tiny{$\pm$ .019} & .376 \tiny{$\pm$ .025} & .401 \tiny{$\pm$ .023} & .293 \tiny{$\pm$ .024} & .319 \tiny{$\pm$ .031} & .299 \tiny{$\pm$ .058} \\
Lexical Similarity & .411 \tiny{$\pm$ .02} & .366 \tiny{$\pm$ .029} & .426 \tiny{$\pm$ .025} & .247 \tiny{$\pm$ .023} & .425 \tiny{$\pm$ .034} & .237 \tiny{$\pm$ .055} \\
EigValLaplacian & .426 \tiny{$\pm$ .016} & .371 \tiny{$\pm$ .028} & .437 \tiny{$\pm$ .03} & .233 \tiny{$\pm$ .025} & .406 \tiny{$\pm$ .03} & .265 \tiny{$\pm$ .056} \\
NumSemSets & .396 \tiny{$\pm$ .018} & .319 \tiny{$\pm$ .031} & .418 \tiny{$\pm$ .024} & .238 \tiny{$\pm$ .023} & .365 \tiny{$\pm$ .033} & .253 \tiny{$\pm$ .052} \\
\midrule
 \rowcolor{gray!20}
 \multicolumn{7}{c}{\textit{Consistency-based UQ: multinomial vs. beamsearch versions}} \\
\midrule
Dissimilarity & .505 \tiny{$\pm$ .018} & .379 \tiny{$\pm$ .028} & .630 \tiny{$\pm$ .021} & .206 \tiny{$\pm$ .019} & \underline{.477} \tiny{$\pm$ .037} & .327 \tiny{$\pm$ .066} \\
Dissimilarity + beamsearch & \textbf{.543} $\uparrow$\tiny{$\pm$ .019} & \underline{.417} $\uparrow$\tiny{$\pm$ .026} & \textbf{.650} $\uparrow$\tiny{$\pm$ .022} & .252 $\uparrow$\tiny{$\pm$ .022} & \textbf{.478} $\uparrow$\tiny{$\pm$ .031} & \underline{.355} $\uparrow$\tiny{$\pm$ .062} \\
\midrule
Eccentricity & .453 \tiny{$\pm$ .016} & .368 \tiny{$\pm$ .029} & .563 \tiny{$\pm$ .021} & .231 \tiny{$\pm$ .025} & .396 \tiny{$\pm$ .035} & .251 \tiny{$\pm$ .058} \\
Eccentricity + beamsearch & .505 $\uparrow$\tiny{$\pm$ .017} & .397 $\uparrow$\tiny{$\pm$ .029} & .603 $\uparrow$\tiny{$\pm$ .023} & .285 $\uparrow$\tiny{$\pm$ .024} & .410 $\uparrow$\tiny{$\pm$ .03} & .345 $\uparrow$\tiny{$\pm$ .061} \\
\midrule
EigVecDissimilarity & .463 \tiny{$\pm$ .019} & .370 \tiny{$\pm$ .028} & .561 \tiny{$\pm$ .026} & .236 \tiny{$\pm$ .025} & .425 \tiny{$\pm$ .035} & .256 \tiny{$\pm$ .051} \\
EigVecDissimilarity + beamsearch & .510 $\uparrow$\tiny{$\pm$ .021} & .414 $\uparrow$\tiny{$\pm$ .028} & .598 $\uparrow$\tiny{$\pm$ .022} & .301 $\uparrow$\tiny{$\pm$ .019} & .450 $\uparrow$\tiny{$\pm$ .033} & \textbf{.376} $\uparrow$\tiny{$\pm$ .057} \\
\midrule
CocoaMSP & .505 \tiny{$\pm$ .018} & .404 \tiny{$\pm$ .025} & .587 \tiny{$\pm$ .023} & .314 \tiny{$\pm$ .024} & .461 \tiny{$\pm$ .031} & .334 \tiny{$\pm$ .054} \\
CocoaMSP + beamsearch & .521 $\uparrow$\tiny{$\pm$ .019} & \textbf{.426} $\uparrow$\tiny{$\pm$ .024} & .615 $\uparrow$\tiny{$\pm$ .021} & \textbf{.345} $\uparrow$\tiny{$\pm$ .026} & .473 $\uparrow$\tiny{$\pm$ .03} & .347 $\uparrow$\tiny{$\pm$ .061} \\
\midrule
CocoaPPL & .523 \tiny{$\pm$ .017} & .397 \tiny{$\pm$ .026} & .628 \tiny{$\pm$ .024} & .312 \tiny{$\pm$ .023} & .461 \tiny{$\pm$ .034} & .327 \tiny{$\pm$ .055} \\
CocoaPPL + beamsearch & \underline{.536} $\uparrow$\tiny{$\pm$ .02} & .412 $\uparrow$\tiny{$\pm$ .027} & \underline{.649} $\uparrow$\tiny{$\pm$ .026} & \underline{.339} $\uparrow$\tiny{$\pm$ .021} & .461 $\uparrow$\tiny{$\pm$ .035} & .337 $\uparrow$\tiny{$\pm$ .057} \\
\bottomrule
\end{tabular}
    }
  \end{table}

  Table~\ref{tab:qa_mean_prr} presents PRR results for six models, averaged over six datasets.
  Across all models, incorporating beam search consistently improves the performance of consistency-based uncertainty scores.
  Moreover, in almost every case, beam search–based methods achieve either the best or second-best PRR compared to both baselines and the original consistency-based approaches.
  In particular, Dissimilarity + Beam Search achieves the best PRR scores for all base models and the second-best scores for \llamait and \qwenit.
  Similarly, CocoaMSP + Beam Search achieves the best results for \llamait and \gemmait, while CocoaPPL + Beam Search ranks second-best for \llamab, \gemmab, and \gemmait.
  We further provide separate results for each dataset in Appendix~\ref{appendix:other_llms}.

\subsection{Ablations}
  In this section, we study sensitivity to (i) the number of candidates $M$, (ii) output length, and (iii) rejection rate in PRR curves.

\subsubsection{Effect of Sample Count}

  We vary the sample count $M \in \{1, \dots, 15\}$ for Dissimilarity, Eccentricity, and EigVecDissimilarity under multinomial sampling and beam search. Figure~\ref{fig:abl_nsamples} shows that beam search generally achieves higher PRR across all budgets $M \ge 2$. Notably, beam search reaches high PRR at small budgets (3-5 samples) and saturates quickly, while multinomial sampling improves more gradually and remains below beam search throughout.
  
  For $M=1$, beam search reduces to greedy decoding, causing Dissimilarity to be nearly zero because it compares two identical greedy outputs. In contrast, the sampling variant compares greedy decoding to a stochastic sample, yielding a more informative value.

\subsubsection{Effect of Output Length}

  \begin{figure}[t!]
    \centering
    \includegraphics[width=\linewidth, trim=0 2.0em 0 2.2em]{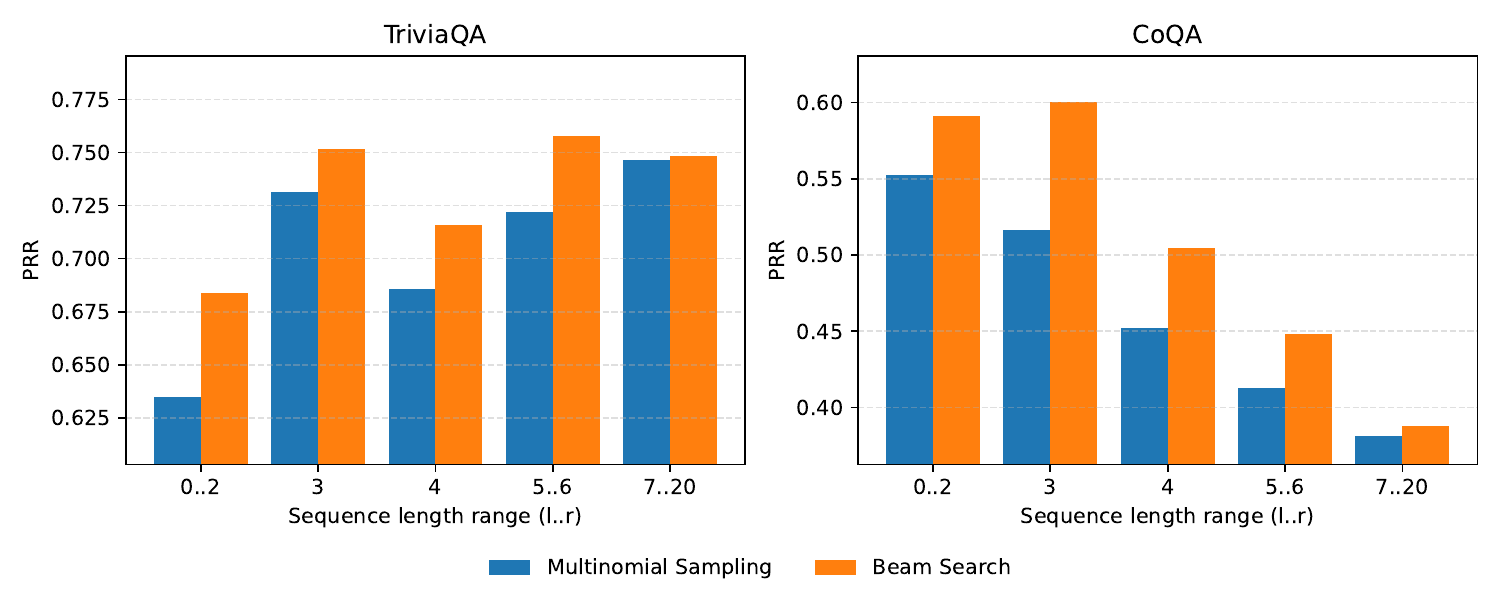}
    \caption{PRR (\(\uparrow\) is better) for Dissimilarity under beam search (with probability weights) vs. multinomial sampling, for different output lengths. Each dataset (TriviaQA, CoQA) with \gemmab is partitioned into five approximately equal-size bins token length of greedy output.}
    \label{fig:abl_length}
  \end{figure}

  Beam-guided estimators outperform sampling-based ones most clearly when generations are short. As shown earlier in Figure~\ref{fig:redundancy}, duplicate rates under multinomial sampling are high for 2-4 tokens ($\sim 30–50\%$) and drop to $\sim 17\%$ for outputs of 8+ tokens. To quantify the impact, we compute PRR for Dissimilarity using beam search (with weights from equation~\eqref{eq:restricted-mass}) and multinomial sampling (no weights) across five length bins of approximately equal size on TriviaQA and CoQA with \gemmab; see Figure~\ref{fig:abl_length}. Within each bin, beam search consistently beats multinomial sampling for short outputs; the gap narrows and becomes negligible for lengths of about 7 tokens and above, where duplication is less pronounced.

\subsubsection{Prediction-Rejection Curves}

  \begin{figure}[t!]
    \centering
    \includegraphics[width=\linewidth, trim=0 3.0em 0 2.5em]{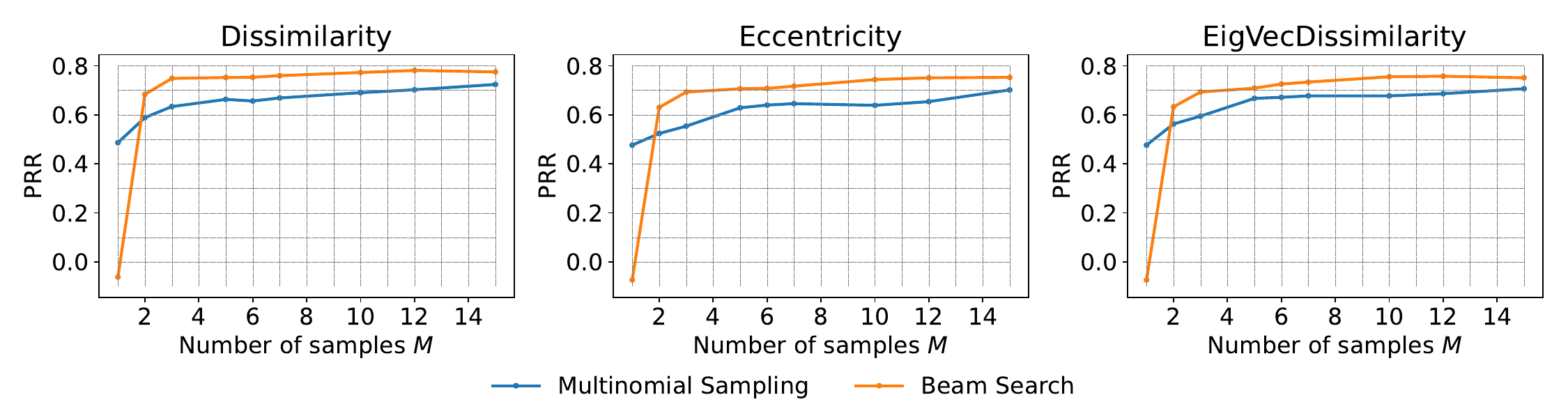}
    \caption{PRR (\(\uparrow\) is better) as a function of the number of candidates \(M\) on TriviaQA with \gemmab. Each panel reports one estimator (Dissimilarity, Eccentricity, EigVecDissimilarity). Curves compare multinomial sampling and beam search (with probability weights from equation~\eqref{eq:restricted-mass}).}
    \label{fig:abl_nsamples}
  \end{figure}

  \begin{figure}[t!]
    \centering
    \includegraphics[width=\linewidth, trim=0 2.0em 0 2.2em]{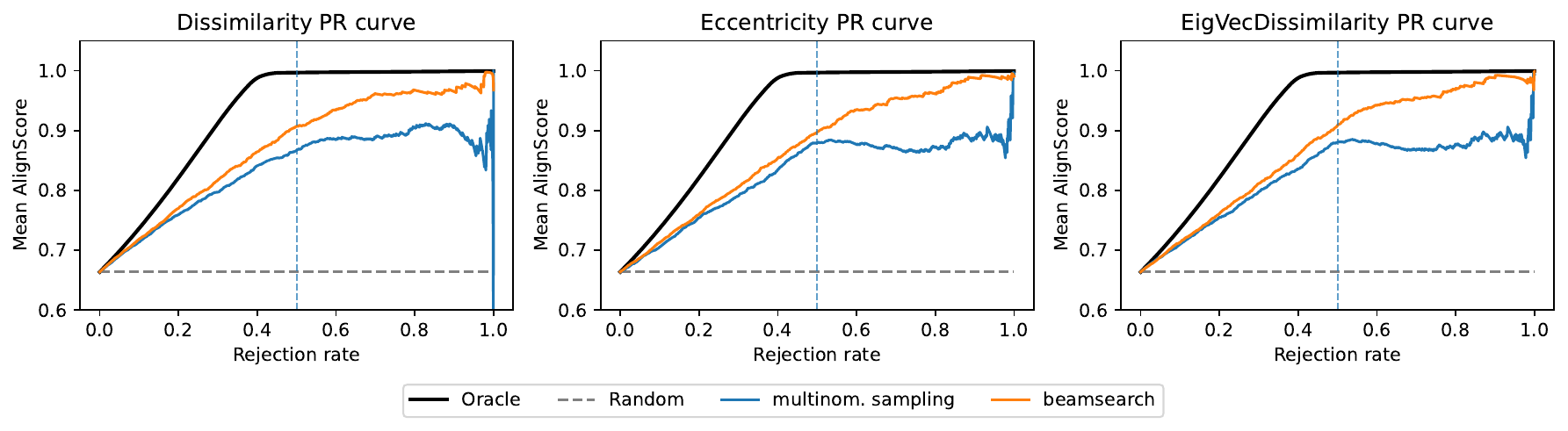}
    \caption{Prediction-Rejection curves for \textit{Dissimilarity}, \textit{Eccentricity}, and \textit{EigVecDissimilarity} on TriviaQA with \llamab, comparing multinomial sampling (blue) and beam search with weights (orange). Oracle (black) and random (gray dashed) baselines are shown. The vertical dashed line marks the maximum rejection rate used in AUC calculations.}
    \label{fig:prr_curves}
  \end{figure}

  Figure~\ref{fig:prr_curves} compares full Prediction-Rejection curves for Dissimilarity, Eccentricity, and EigVecDissimilarity on TriviaQA with \llamab. 
  Across all estimators, beam search consistently dominates multinomial sampling for nearly the entire rejection range. 
  The becomes increasingly pronounced as the rejection rate grows, where beam-guided estimates remain stable while multinomial ones flatten or even degrade. 
  This indicates that beam search is especially beneficial in the high-rejection regime, where distinguishing between stronger and weaker candidates is most critical.

\subsubsection{Additional Ablations}

  Additional ablations are deferred to the appendices: Appendix~\ref{appendix:abl_sampling} compares candidate–generation strategies including Diverse Beam Search, temperature sampling, and a hybrid multinomial–beam sampling. Appendix~\ref{appendix:abl_normalization} investigates restricted-mass normalization and shows that introducing a small probability floor $\epsilon$ can stabilize the weighting of low-mass beams. Appendix~\ref{appendix:abl_other_methods} evaluates other sampling-based objectives (Semantic Entropy, Degree Matrix) under beam generation with probability-weighted formulations. 
  Appendix~\ref{appendix:topbeam_scoring} examines top-1 beam decode as the produced answer $\yv_*$ (instead of greedy), a natural choice when beam search has already been run.


\section{Related Work}
\label{sec:rel_work}

  
\paragraph{Consistency-based uncertainty estimation.} 
  In a black-box setting, consistency-based methods are especially relevant, as they do not require access to the model internals. \citet{lin2023generating} introduce several methods that estimate confidence based on a similarity matrix, where each entry represents the similarity between a pair of sampled generations.
  \citet{fomicheva-etal-2020-unsupervised} present Lexical Similarity, a metric that evaluates the average similarity of words or phrases between each pair of responses. In a white-box setting, consistency signals can be combined with model token-probabilities-based confidence. These hybrid methods, such as Semantic Entropy~\citep{kuhn2023semantic}, CoCoA~\citep{vashurin2025uncertainty} and SAR~\citep{duan-etal-2024-shifting} explore different ways of combining these signals and achieve state-of-the-art performance. However, these works are primarily concerned with the introduction of new methods for uncertainty quantification and use multinational sampling as a way to approximate a variety of consistency-based measures.

\paragraph{Uncertainty and decoding.}
  There were some efforts focused on examining the interaction between decoding strategies and uncertainty quantification. In particular, \citet{hashimoto2025decodinguncertaintyimpactdecoding} explores the impact of decoding strategies on the performance of token probabilities-based UQ methods, namely Sequence Probability and Mean Token Entropy. The authors find that these scores produced with beam search can sometimes under perform compared to greedy or contrastive search. While this work offers interesting insights, no experiments with stochastic decoding strategies or non-likelihood based methods were conducted. Conversely, other research focused on making the decoding itself uncertainty-aware. For example, \citet{daheim2025uncertaintyaware} propose Minimum Bayes Risk (MBR) decoding, which incorporates model uncertainty into the MBR objective for improved generation quality. \citet{garces-arias-etal-2024-adaptive} and~\citet{lee-etal-2025-uncertainty} incorporate uncertainty into contrastive search decoding. Lastly, \cite{ding-etal-2025-guard} combines global entropy trends and local deviations to guide a self-adaptive decoding. These works integrate uncertainty into the decoding process to improve the quality of the generation, rather than improving the performance of the uncertainty itself. Although some uncertainty-aware decoding methods have also demonstrated improved uncertainty quantification performance, they are generally not evaluated with consistency-based metrics.


\section{Conclusion}
\label{sec:conclusion}
  We present a new family of uncertainty quantification methods for LLMs that employ a beam-weighted estimator of consistency-based uncertainty. 
  Compared to multinomial sampling, commonly used in existing approaches, our method yields lower variance in dissimilarity and greater diversity of candidate answers. 
  We also derive a theoretical lower bound on the beam set probability mass under which the error of the multinomial Monte Carlo estimator is guaranteed to be larger.
  Finally, we evaluate our approach on six QA datasets and six different models, demonstrating state-of-the-art performance.

\section*{Limitations}
  Although our method provides an improvement over existing consistency-based estimators, several important considerations remain.
  First, we evaluated our methods in white-box settings, as they require access to the model's probability distributions. Nonetheless, we argue that developing methods tailored for white-box settings continues to be of great importance given their continued relevance and usage. Moreover, the methods could be extended to the black-box settings using empirical probability estimates. 

  Second, our experiments are limited to short-form QA datasets, and the generalizability of our findings to longer-form generation remains an open question.

  Lastly, our implementation and evaluation relies on existing neural metrics: AlignScore is used to score the quality of the generation, and pre-trained NLI model is utilized as a measure of consistency. Although widely used in previous work, certain more specialized tasks might require different sample similarity measures and quality metrics.




\bibliography{iclr2026_conference}

@article{reddy-etal-2019-coqa,
    title = "{C}o{QA}: A Conversational Question Answering Challenge",
    author = "Reddy, Siva  and
      Chen, Danqi  and
      Manning, Christopher D.",
    journal = "Transactions of the Association for Computational Linguistics",
    volume = "7",
    year = "2019",
    address = "Cambridge, MA",
    publisher = "MIT Press",
    url = "https://aclanthology.org/Q19-1016",
    doi = "10.1162/tacl_a_00266",
    pages = "249--266",
}

@inproceedings{talmor-etal-2019-commonsenseqa,
    title = "{C}ommonsense{QA}: A Question Answering Challenge Targeting Commonsense Knowledge",
    author = "Talmor, Alon  and
      Herzig, Jonathan  and
      Lourie, Nicholas  and
      Berant, Jonathan",
    editor = "Burstein, Jill  and
      Doran, Christy  and
      Solorio, Thamar",
    booktitle = "Proceedings of the 2019 Conference of the North {A}merican Chapter of the Association for Computational Linguistics: Human Language Technologies, Volume 1 (Long and Short Papers)",
    month = jun,
    year = "2019",
    address = "Minneapolis, Minnesota",
    publisher = "Association for Computational Linguistics",
    url = "https://aclanthology.org/N19-1421/",
    doi = "10.18653/v1/N19-1421",
    pages = "4149--4158"
}

@article{kadavath2022language,
  author       = {Saurav Kadavath and
                  Tom Conerly and
                  Amanda Askell and
                  Tom Henighan and
                  Dawn Drain and
                  Ethan Perez and
                  Nicholas Schiefer and
                  Zac Hatfield{-}Dodds and
                  Nova DasSarma and
                  Eli Tran{-}Johnson and
                  Scott Johnston and
                  Sheer El Showk and
                  Andy Jones and
                  Nelson Elhage and
                  Tristan Hume and
                  Anna Chen and
                  Yuntao Bai and
                  Sam Bowman and
                  Stanislav Fort and
                  Deep Ganguli and
                  Danny Hernandez and
                  Josh Jacobson and
                  Jackson Kernion and
                  Shauna Kravec and
                  Liane Lovitt and
                  Kamal Ndousse and
                  Catherine Olsson and
                  Sam Ringer and
                  Dario Amodei and
                  Tom Brown and
                  Jack Clark and
                  Nicholas Joseph and
                  Ben Mann and
                  Sam McCandlish and
                  Chris Olah and
                  Jared Kaplan},
  title        = {Language Models (Mostly) Know What They Know},
  journal      = {arXiv preprint arXiv:2207.05221},
  year         = {2022},
  url          = {https://doi.org/10.48550/arXiv.2207.05221},
  doi          = {10.48550/arXiv.2207.05221},
  eprinttype    = {arXiv},
  eprint       = {2207.05221},
  bibsource    = {dblp computer science bibliography, https://dblp.org}
}

@article{liu2020roberta,
  title={{RoBERTa}: A Robustly Optimized {BERT} Pretraining Approach},
  author={Liu, Yinhan and Ott, Myle and Goyal, Naman and Du, Jingfei and Joshi, Mandar and Chen, Danqi and Levy, Omer and Lewis, Mike and Zettlemoyer, Luke and Stoyanov, Veselin},
  journal={arXiv preprint arXiv:1907.11692},
  year={2019},
  url={https://arxiv.org/abs/1907.11692}
}

@inproceedings{he2021deberta,
  title = {{DeBERTa}: Decoding-Enhanced {BERT} with Disentangled Attention},
  author = {He, Pengcheng and Liu, Xiaodong and Gao, Jianfeng and Chen, Weizhu},
  booktitle = {Proceedings of the International Conference on Learning Representations},
  year = {2021},
  url = {https://openreview.net/forum?id=XPZIaotutsD}
}

@inproceedings{fadeeva-etal-2024-fact,
    title = "Fact-Checking the Output of Large Language Models via Token-Level Uncertainty Quantification",
    author = "Fadeeva, Ekaterina  and
      Rubashevskii, Aleksandr  and
      Shelmanov, Artem  and
      Petrakov, Sergey  and
      Li, Haonan  and
      Mubarak, Hamdy  and
      Tsymbalov, Evgenii  and
      Kuzmin, Gleb  and
      Panchenko, Alexander  and
      Baldwin, Timothy  and
      Nakov, Preslav  and
      Panov, Maxim",
    editor = "Ku, Lun-Wei  and
      Martins, Andre  and
      Srikumar, Vivek",
    booktitle = "Findings of the Association for Computational Linguistics: ACL 2024",
    month = aug,
    year = "2024",
    address = "Bangkok, Thailand",
    publisher = "Association for Computational Linguistics",
    url = "https://aclanthology.org/2024.findings-acl.558/",
    doi = "10.18653/v1/2024.findings-acl.558",
    pages = "9367--9385"
}

@article{lin2023generating,
    title={Generating with Confidence: Uncertainty Quantification for Black-box Large Language Models},
    author={Zhen Lin and Shubhendu Trivedi and Jimeng Sun},
    journal={Transactions on Machine Learning Research},
    issn={2835-8856},
    year={2024},
    url={https://openreview.net/forum?id=DWkJCSxKU5},
    note={}
}

@inproceedings{fadeeva-etal-2023-lm,
    title = "{LM}-Polygraph: Uncertainty Estimation for Language Models",
    author = "Fadeeva, Ekaterina  and
      Vashurin, Roman  and
      Tsvigun, Akim  and
      Vazhentsev, Artem  and
      Petrakov, Sergey  and
      Fedyanin, Kirill  and
      Vasilev, Daniil  and
      Goncharova, Elizaveta  and
      Panchenko, Alexander  and
      Panov, Maxim  and
      Baldwin, Timothy  and
      Shelmanov, Artem",
    editor = "Feng, Yansong  and
      Lefever, Els",
    booktitle = "Proceedings of the 2023 Conference on Empirical Methods in Natural Language Processing: System Demonstrations",
    month = dec,
    year = "2023",
    address = "Singapore",
    publisher = "Association for Computational Linguistics",
    url = "https://aclanthology.org/2023.emnlp-demo.41/",
    doi = "10.18653/v1/2023.emnlp-demo.41",
    pages = "446--461"
}

@inproceedings{zha-etal-2023-alignscore,
    title = "{A}lign{S}core: Evaluating Factual Consistency with A Unified Alignment Function",
    author = "Zha, Yuheng  and
      Yang, Yichi  and
      Li, Ruichen  and
      Hu, Zhiting",
    editor = "Rogers, Anna  and
      Boyd-Graber, Jordan  and
      Okazaki, Naoaki",
    booktitle = "Proceedings of the 61st Annual Meeting of the Association for Computational Linguistics (Volume 1: Long Papers)",
    month = jul,
    year = "2023",
    address = "Toronto, Canada",
    publisher = "Association for Computational Linguistics",
    url = "https://aclanthology.org/2023.acl-long.634/",
    doi = "10.18653/v1/2023.acl-long.634",
    pages = "11328--11348"
}

@article{vashurin-etal-2025-benchmarking,
    title = "Benchmarking Uncertainty Quantification Methods for Large Language Models with {LM}-{P}olygraph",
    author = "Vashurin, Roman  and
      Fadeeva, Ekaterina  and
      Vazhentsev, Artem  and
      Rvanova, Lyudmila  and
      Vasilev, Daniil  and
      Tsvigun, Akim  and
      Petrakov, Sergey  and
      Xing, Rui  and
      Sadallah, Abdelrahman  and
      Grishchenkov, Kirill  and
      Panchenko, Alexander  and
      Baldwin, Timothy  and
      Nakov, Preslav  and
      Panov, Maxim  and
      Shelmanov, Artem",
    journal = "Transactions of the Association for Computational Linguistics",
    volume = "13",
    year = "2025",
    address = "Cambridge, MA",
    publisher = "MIT Press",
    url = "https://aclanthology.org/2025.tacl-1.11/",
    doi = "10.1162/tacl_a_00737",
    pages = "220--248"
}

@misc{qwen3technicalreport,
  title={Qwen3 Technical Report}, 
  author={Qwen Team},
  year={2025},
  eprint={2505.09388},
  archivePrefix={arXiv},
  primaryClass={cs.CL},
  url={https://arxiv.org/abs/2505.09388}, 
}

@misc{gemmateam2025gemma3technicalreport,
  title={Gemma 3 Technical Report}, 
  author={Gemma Team},
  year={2025},
  eprint={2503.19786},
  archivePrefix={arXiv},
  primaryClass={cs.CL},
  url={https://arxiv.org/abs/2503.19786}, 
}

@article{DBLP:journals/corr/abs-2407-21783,
  publtype={informal},
  author={Abhimanyu Dubey and Abhinav Jauhri and Abhinav Pandey and Abhishek Kadian and Ahmad Al-Dahle and Aiesha Letman and Akhil Mathur and Alan Schelten and Amy Yang and Angela Fan},
  title={The {L}lama 3 Herd of Models},
  year={2024},
  cdate={1704067200000},
  journal={CoRR},
  volume={abs/2407.21783},
  url={https://doi.org/10.48550/arXiv.2407.21783}
}

@article{allenai:arc,
  author    = {Peter Clark  and Isaac Cowhey and Oren Etzioni and Tushar Khot and
                Ashish Sabharwal and Carissa Schoenick and Oyvind Tafjord},
  title     = {Think you have Solved Question Answering? {T}ry {ARC}, the {AI}2 Reasoning Challenge},
  journal={arXiv preprint arXiv:1803.05457},
  year      = {2018},
  url={https://arxiv.org/abs/1803.05457}
}

@inproceedings{yang2018hotpotqa,
  title={{HotpotQA}: A Dataset for Diverse, Explainable Multi-hop Question Answering},
  author={Yang, Zhilin and Qi, Peng and Zhang, Saizheng and Bengio, Yoshua and Cohen, William W. and Salakhutdinov, Ruslan and Manning, Christopher D.},
  booktitle={Conference on Empirical Methods in Natural Language Processing ({EMNLP})},
  year={2018},
  address = {Brussels, Belgium},
  url={https://aclanthology.org/D18-1259/}
}

@inproceedings{berant-etal-2013-semantic,
    title = "Semantic Parsing on {F}reebase from Question-Answer Pairs",
    author = "Berant, Jonathan  and
      Chou, Andrew  and
      Frostig, Roy  and
      Liang, Percy",
    editor = "Yarowsky, David  and
      Baldwin, Timothy  and
      Korhonen, Anna  and
      Livescu, Karen  and
      Bethard, Steven",
    booktitle = "Proceedings of the 2013 Conference on Empirical Methods in Natural Language Processing",
    month = oct,
    year = "2013",
    address = "Seattle, Washington, USA",
    publisher = "Association for Computational Linguistics",
    url = "https://aclanthology.org/D13-1160/",
    pages = "1533--1544"
}

@inproceedings{joshi-etal-2017-triviaqa,
    title = "{T}rivia{QA}: A Large Scale Distantly Supervised Challenge Dataset for Reading Comprehension",
    author = "Joshi, Mandar  and
      Choi, Eunsol  and
      Weld, Daniel  and
      Zettlemoyer, Luke",
    editor = "Barzilay, Regina  and
      Kan, Min-Yen",
    booktitle = "Proceedings of the 55th Annual Meeting of the Association for Computational Linguistics (Volume 1: Long Papers)",
    month = jul,
    year = "2017",
    address = "Vancouver, Canada",
    publisher = "Association for Computational Linguistics",
    url = "https://aclanthology.org/P17-1147/",
    doi = "10.18653/v1/P17-1147",
    pages = "1601--1611"
}

@inproceedings{vashurin2025uncertainty,
    title={{CoCoA}: A Minimum {B}ayes Risk Framework Bridging Confidence and Consistency for Uncertainty Quantification in LLMs},
    author={Vashurin, Roman and Goloburda, Maiya and Ilina, Albina and Rubashevskii, Aleksandr and Nakov, Preslav and Shelmanov, Artem and Panov, Maxim},
    booktitle={The Thirty-ninth Annual Conference on Neural Information Processing Systems},
    year={2025},
    url={https://openreview.net/forum?id=H1NGlLNaVC}
}

@article{fomicheva-etal-2020-unsupervised,
    title = "Unsupervised Quality Estimation for Neural Machine Translation",
    author = "Fomicheva, Marina  and
      Sun, Shuo  and
      Yankovskaya, Lisa  and
      Blain, Fr{\'e}d{\'e}ric  and
      Guzm{\'a}n, Francisco  and
      Fishel, Mark  and
      Aletras, Nikolaos  and
      Chaudhary, Vishrav  and
      Specia, Lucia",
    editor = "Johnson, Mark  and
      Roark, Brian  and
      Nenkova, Ani",
    journal = "Transactions of the Association for Computational Linguistics",
    volume = "8",
    year = "2020",
    address = "Cambridge, MA",
    publisher = "MIT Press",
    url = "https://aclanthology.org/2020.tacl-1.35/",
    doi = "10.1162/tacl_a_00330",
    pages = "539--555"
}

@inproceedings{kuhn2023semantic,
  author       = {Lorenz Kuhn and
                  Yarin Gal and
                  Sebastian Farquhar},
  title        = {Semantic Uncertainty: Linguistic Invariances for Uncertainty Estimation
                  in Natural Language Generation},
  booktitle    = {The Eleventh International Conference on Learning Representations,
                  {ICLR} 2023, Kigali, Rwanda, May 1-5, 2023},
  year         = {2023},
  url          = {https://openreview.net/pdf?id=VD-AYtP0dve},
  bibsource    = {dblp computer science bibliography, https://dblp.org}
}

@inproceedings{duan-etal-2024-shifting,
    title = "Shifting Attention to Relevance: Towards the Predictive Uncertainty Quantification of Free-Form Large Language Models",
    author = "Duan, Jinhao  and
      Cheng, Hao  and
      Wang, Shiqi  and
      Zavalny, Alex  and
      Wang, Chenan  and
      Xu, Renjing  and
      Kailkhura, Bhavya  and
      Xu, Kaidi",
    editor = "Ku, Lun-Wei  and
      Martins, Andre  and
      Srikumar, Vivek",
    booktitle = "Proceedings of the 62nd Annual Meeting of the Association for Computational Linguistics (Volume 1: Long Papers)",
    month = aug,
    year = "2024",
    address = "Bangkok, Thailand",
    publisher = "Association for Computational Linguistics",
    url = "https://aclanthology.org/2024.acl-long.276/",
    doi = "10.18653/v1/2024.acl-long.276",
    pages = "5050--5063"
}

@inproceedings{
    law,
    author = {Shu, Dong and Zhao, Haoran and Liu, Xukun and Demeter, David and Du, Mengnan and Zhang, Yongfeng},
    title = {{LawLLM}: Law Large Language Model for the {US} Legal System},
    year = {2024},
    publisher = {Association for Computing Machinery},
    url = {https://doi.org/10.1145/3627673.3680020},
    doi = {10.1145/3627673.3680020}, booktitle = {Proceedings of the 33rd ACM International Conference on Information and Knowledge Management},
    pages = {4882–4889},
    numpages = {8},
    address = {Boise, Idaho, USA}
}

@article{medicine,
  title = {Current applications and challenges in large language models for patient care: {A} systematic review},
  volume = {5},
  ISSN = {2730-664X},
  url = {http://dx.doi.org/10.1038/s43856-024-00717-2},
  DOI = {10.1038/s43856-024-00717-2},
  number = {1},
  journal = {Communications Medicine},
  publisher = {Springer Science and Business Media LLC},
  author = {Busch,  Felix and Hoffmann,  Lena and Rueger,  Christopher and van Dijk,  Elon HC and Kader,  Rawen and Ortiz-Prado,  Esteban and Makowski,  Marcus R. and Saba,  Luca and Hadamitzky,  Martin and Kather,  Jakob Nikolas and Truhn,  Daniel and Cuocolo,  Renato and Adams,  Lisa C. and Bressem,  Keno K.},
  year = {2025}
}

@article{education,
  title = {The Use of Large Language Models in Education},
  volume = {35},
  ISSN = {1560-4306},
  url = {http://dx.doi.org/10.1007/s40593-025-00457-x},
  DOI = {10.1007/s40593-025-00457-x},
  number = {2},
  journal = {International Journal of Artificial Intelligence in Education},
  publisher = {Springer Science and Business Media LLC},
  author = {Xing,  Wanli and Nixon,  Nia and Crossley,  Scott and Denny,  Paul and Lan,  Andrew and Stamper,  John and Yu,  Zhou},
  year = {2025},
  month = feb,
  pages = {439–443}
}

@inproceedings{yoo-etal-2022-detection,
    title = "Detection of Adversarial Examples in Text Classification: Benchmark and Baseline via Robust Density Estimation",
    author = "Yoo, KiYoon  and
      Kim, Jangho  and
      Jang, Jiho  and
      Kwak, Nojun",
    editor = "Muresan, Smaranda  and
      Nakov, Preslav  and
      Villavicencio, Aline",
    booktitle = "Findings of the Association for Computational Linguistics: ACL 2022",
    month = may,
    year = "2022",
    address = "Dublin, Ireland",
    publisher = "Association for Computational Linguistics",
    url = "https://aclanthology.org/2022.findings-acl.289/",
    doi = "10.18653/v1/2022.findings-acl.289",
    pages = "3656--3672"
}

@inproceedings{tian-etal-2023-just,
    title = "Just Ask for Calibration: Strategies for Eliciting Calibrated Confidence Scores from Language Models Fine-Tuned with Human Feedback",
    author = "Tian, Katherine  and
      Mitchell, Eric  and
      Zhou, Allan  and
      Sharma, Archit  and
      Rafailov, Rafael  and
      Yao, Huaxiu  and
      Finn, Chelsea  and
      Manning, Christopher",
    editor = "Bouamor, Houda  and
      Pino, Juan  and
      Bali, Kalika",
    booktitle = "Proceedings of the 2023 Conference on Empirical Methods in Natural Language Processing",
    month = dec,
    year = "2023",
    address = "Singapore",
    publisher = "Association for Computational Linguistics",
    url = "https://aclanthology.org/2023.emnlp-main.330/",
    doi = "10.18653/v1/2023.emnlp-main.330",
    pages = "5433--5442"
}

@inproceedings{xia-etal-2025-survey,
    title = "A Survey of Uncertainty Estimation Methods on Large Language Models",
    author = "Xia, Zhiqiu  and
      Xu, Jinxuan  and
      Zhang, Yuqian  and
      Liu, Hang",
    editor = "Che, Wanxiang  and
      Nabende, Joyce  and
      Shutova, Ekaterina  and
      Pilehvar, Mohammad Taher",
    booktitle = "Findings of the Association for Computational Linguistics: ACL 2025",
    month = jul,
    year = "2025",
    address = "Vienna, Austria",
    publisher = "Association for Computational Linguistics",
    url = "https://aclanthology.org/2025.findings-acl.1101/",
    doi = "10.18653/v1/2025.findings-acl.1101",
    pages = "21381--21396",
    ISBN = "979-8-89176-256-5"
}

@article{Baan2023UncertaintyIN,
  title={Uncertainty in Natural Language Generation: From Theory to Applications},
  author={Joris Baan and Nico Daheim and Evgenia Ilia and Dennis Ulmer and Haau-Sing Li and R. Fern{\'a}ndez and Barbara Plank and Rico Sennrich and Chrysoula Zerva and Wilker Aziz},
  journal={ArXiv},
  year={2023},
  volume={abs/2307.15703},
  url={https://api.semanticscholar.org/CorpusID:260316110}
}

@article{Xiao2019,
  title = {Quantifying Uncertainties in Natural Language Processing Tasks},
  volume = {33},
  ISSN = {2159-5399},
  url = {http://dx.doi.org/10.1609/aaai.v33i01.33017322},
  DOI = {10.1609/aaai.v33i01.33017322},
  number = {01},
  journal = {Proceedings of the AAAI Conference on Artificial Intelligence},
  publisher = {Association for the Advancement of Artificial Intelligence (AAAI)},
  author = {Xiao,  Yijun and Wang,  William Yang},
  year = {2019},
  pages = {7322–7329},
  address = {Honolulu, Hawaii, USA}
}

@article{Vijayakumar2018, title={Diverse Beam Search for Improved Description of Complex Scenes}, volume={32}, url={https://ojs.aaai.org/index.php/AAAI/article/view/12340}, DOI={10.1609/aaai.v32i1.12340}, abstractNote={ &lt;p&gt; A single image captures the appearance and position of multiple entities in a scene as well as their complex interactions. As a consequence, natural language grounded in visual contexts tends to be diverse---with utterances differing as focus shifts to specific objects, interactions, or levels of detail. Recently, neural sequence models such as RNNs and LSTMs have been employed to produce visually-grounded language. Beam Search, the standard work-horse for decoding sequences from these models, is an approximate inference algorithm that decodes the top-B sequences in a greedy left-to-right fashion. In practice, the resulting sequences are often minor rewordings of a common utterance, failing to capture the multimodal nature of source images. To address this shortcoming, we propose Diverse Beam Search (DBS), a diversity promoting alternative to BS for approximate inference. DBS produces sequences that are significantly different from each other by incorporating diversity constraints within groups of candidate sequences during decoding; moreover, it achieves this with minimal computational or memory overhead. We demonstrate that our method improves both diversity and quality of decoded sequences over existing techniques on two visually-grounded language generation tasks---image captioning and visual question generation---particularly on complex scenes containing diverse visual content. We also show similar improvements at language-only machine translation tasks, highlighting the generality of our approach. &lt;/p&gt; }, number={1}, journal={Proceedings of the AAAI Conference on Artificial Intelligence}, author={Vijayakumar, Ashwin and Cogswell, Michael and Selvaraju, Ramprasaath and Sun, Qing and Lee, Stefan and Crandall, David and Batra, Dhruv}, year={2018}, month={Apr.} }

@inproceedings{Holtzman2020The,
    title={The Curious Case of Neural Text Degeneration},
    author={Ari Holtzman and Jan Buys and Li Du and Maxwell Forbes and Yejin Choi},
    booktitle={International Conference on Learning Representations},
    year={2020},
    url={https://openreview.net/forum?id=rygGQyrFvH}
}

@inproceedings{hashimoto2025decodinguncertaintyimpactdecoding,
    title = "Decoding Uncertainty: The Impact of Decoding Strategies for Uncertainty Estimation in Large Language Models",
    author = "Hashimoto, Wataru  and
      Kamigaito, Hidetaka  and
      Watanabe, Taro",
    editor = "Christodoulopoulos, Christos  and
      Chakraborty, Tanmoy  and
      Rose, Carolyn  and
      Peng, Violet",
    booktitle = "Findings of the Association for Computational Linguistics: EMNLP 2025",
    month = nov,
    year = "2025",
    address = "Suzhou, China",
    publisher = "Association for Computational Linguistics",
    url = "https://aclanthology.org/2025.findings-emnlp.788/",
    doi = "10.18653/v1/2025.findings-emnlp.788",
    pages = "14601--14613",
    ISBN = "979-8-89176-335-7"
}

@inproceedings{daheim2025uncertaintyaware,
  title = {Uncertainty-Aware Decoding with Minimum {B}ayes Risk},
  author = {Daheim, Nico and Meister, Clara and M{\"o}llenhoff, Thomas and Gurevych, Iryna},
  booktitle = {The Thirteenth International Conference on Learning Representations},
  year = {2025},
  address = {Singapore},
  url = {https://openreview.net/forum?id=hPpyUv1XyQ}
}

@inproceedings{garces-arias-etal-2024-adaptive,
    title = "Adaptive Contrastive Search: Uncertainty-Guided Decoding for Open-Ended Text Generation",
    author = "Garces Arias, Esteban  and
      Rodemann, Julian  and
      Li, Meimingwei  and
      Heumann, Christian  and
      A{\ss}enmacher, Matthias",
    editor = "Al-Onaizan, Yaser  and
      Bansal, Mohit  and
      Chen, Yun-Nung",
    booktitle = "Findings of the Association for Computational Linguistics: EMNLP 2024",
    month = nov,
    year = "2024",
    address = "Miami, Florida, USA",
    publisher = "Association for Computational Linguistics",
    url = "https://aclanthology.org/2024.findings-emnlp.885/",
    doi = "10.18653/v1/2024.findings-emnlp.885",
    pages = "15060--15080"
}

@inproceedings{lee-etal-2025-uncertainty,
    title = "Uncertainty-Aware Contrastive Decoding",
    author = "Lee, Hakyung  and
      Park, Subeen  and
      Kim, Joowang  and
      Lim, Sungjun  and
      Song, Kyungwoo",
    editor = "Che, Wanxiang  and
      Nabende, Joyce  and
      Shutova, Ekaterina  and
      Pilehvar, Mohammad Taher",
    booktitle = "Findings of the Association for Computational Linguistics: ACL 2025",
    month = jul,
    year = "2025",
    address = "Vienna, Austria",
    publisher = "Association for Computational Linguistics",
    url = "https://aclanthology.org/2025.findings-acl.1352/",
    doi = "10.18653/v1/2025.findings-acl.1352",
    pages = "26376--26391",
    ISBN = "979-8-89176-256-5"
}

@inproceedings{ding-etal-2025-guard,
    title = "{GUARD}: Glocal Uncertainty-Aware Robust Decoding for Effective and Efficient Open-Ended Text Generation",
    author = "Ding, Yuanhao  and
      Garces Arias, Esteban  and
      Li, Meimingwei  and
      Rodemann, Julian  and
      A{\ss}enmacher, Matthias  and
      Chen, Danlu  and
      Fan, Gaojuan  and
      Heumann, Christian  and
      Zhang, Chongsheng",
    editor = "Christodoulopoulos, Christos  and
      Chakraborty, Tanmoy  and
      Rose, Carolyn  and
      Peng, Violet",
    booktitle = "Findings of the Association for Computational Linguistics: EMNLP 2025",
    month = nov,
    year = "2025",
    address = "Suzhou, China",
    publisher = "Association for Computational Linguistics",
    url = "https://aclanthology.org/2025.findings-emnlp.380/",
    doi = "10.18653/v1/2025.findings-emnlp.380",
    pages = "7202--7226",
    ISBN = "979-8-89176-335-7"
}
\bibliographystyle{iclr2026_conference}

\clearpage
\newpage

\appendix


\section{Ablation Studies}

\subsection{Different Sampling Strategies}
  \label{appendix:abl_sampling}

  This section studies how the proposed estimators behave under different sample generation strategies. In addition to multinomial sampling and beam search settings, we evaluate three additional families.

\paragraph{Diverse beam search.}
  We generate $M=10$ candidates using a diverse beam search~\citep{Vijayakumar2018} with group penalties $\lambda \in \{0.5, 1.0, 1.5, 2.0\}$ and group counts that split the ten candidates into $G \in \{2,5\}$ groups. As in the main beam setup, we apply the same self-normalized probability weights $w_i$ from equation~\eqref{eq:restricted-mass}.

\paragraph{Temperature sampling with importance weights.}
  For different temperatures $T$, we draw $M=10$ samples with temperature sampling $\{\yv_T^{(i)}\}_{i=1}^M$ and re-weight them via self-normalized importance weights
  \begin{equation}
    w^T_i = \frac{p \bigl(\yv_T^{(i)} \mid \xv \bigr)^{1 - 1/T}}{\sum_{j=1}^M p \bigl(\yv_T^{(j)} \mid \xv \bigr)^{1 - 1/T}}.
  \label{eq:importance_weights}
  \end{equation}

\paragraph{Hybrid multinomial–beam search.}
  We also consider a joint strategy: first draw $B$ beam candidates, then draw the remaining $M - B$ candidates via multinomial sampling while excluding the beam results. Beam candidates use autoregressive probability weights, and the residual probability mass is distributed uniformly over the multinomial samples. Let $\{\bv^{(i)}\}_{i=1}^{B}$ be the beam outputs and $\{\yv^{(j)}\}_{j=B + 1}^{M}$ the multinomial samples (with beam sequences masked out). We assign weights
  \begin{equation}
    w^H_i = p \bigl(\bv^{(i)} \mid \xv\bigr), \quad i=1, \dots, B,
    \qquad
    w^H_j = \frac{1 - \sum_{i=1}^{B} p \left(\bv^{(i)} \mid \xv\right)}{M-B},\quad j=B + 1,\dots,M,
  \label{eq:hybrid_weights}
  \end{equation}
  so that $\sum_{i=1}^{M} w^H_i = 1$. We test $B \in \{1,\dots,9\}$.

  \begin{table}[!b]

\caption{PRR (\(\uparrow\) is better) under different sampling strategies. Columns list methods (Dissimilarity, Eccentricity, EigVecDissimilarity) and four different model-dataset pairs; rows list strategies with their hyperparameters. Per column, top-1 is \textbf{bold}, second-best is \underline{underlined}.}

\centering
\small
\resizebox{\textwidth}{!}{%
\begin{tabular}{cc|ccc|ccc|ccc|ccc}
\toprule

\multicolumn{2}{c|}{} & \multicolumn{6}{c|}{\gemmab} & \multicolumn{6}{c}{\llamab} \\

\multicolumn{2}{c|}{} & \multicolumn{3}{c|}{TriviaQA} & \multicolumn{3}{c|}{CoQA} & \multicolumn{3}{c|}{TriviaQA} & \multicolumn{3}{c}{CoQA} \\

\multicolumn{2}{c|}{} & Dissim & Ecc & \multirowcell{EigVec\\Dissim} & Dissim & Ecc & \multirowcell{EigVec\\Dissim} & Dissim & Ecc & \multirowcell{EigVec\\Dissim} & Dissim & Ecc & \multirowcell{EigVec\\Dissim} \\

\midrule
\multicolumn{2}{c|}{Beam search} & .771 & .732 & .751 & \textbf{.561} & .483 & .488 & .623 & .581 & .598 & \underline{.502} & \textbf{.438} & \textbf{.454} \\
\midrule
\multirow{6}{*}{\multirowcell{Multinomial \\Sampling}}
& $T=0.7$ & .703 & .619 & .687 & .455 & .368 & .397 & .561 & .521 & .538 & .451 & .371 & .382 \\
& $T=0.9$ & .734 & .664 & .710 & .468 & .417 & .425 & .588 & .521 & .533 & .418 & .407 & .410 \\
& $T=1.0$ & .742 & .689 & .715 & .465 & .424 & .432 & .599 & .516 & .537 & .431 & .416 & .424 \\
& $T=1.2$ & .733 & .679 & .717 & .435 & .424 & .437 & .610 & .517 & .535 & .391 & .427 & .433 \\
& $T=1.5$ & .718 & .685 & .717 & .406 & .426 & .436 & .555 & .515 & .532 & .397 & .431 & .440 \\
& $T=1.7$ & .680 & .690 & .717 & .372 & .426 & .433 & .569 & .514 & .530 & .304 & .432 & \underline{.441} \\
\midrule
\multirow{8}{*}{\multirowcell{Diverse \\ Beam \\ Search}}
 & $G=2, \lambda=0.5$ & .753 & .528 & .693 & .498 & .405 & .452 & .623 & -.123 & .546 & .458 & .310 & .363 \\
 & $G=2, \lambda=1.0$ & .753 & .566 & .705 & .518 & .384 & .432 & .594 & -.138 & .539 & .466 & .285 & .355 \\
 & $G=2, \lambda=1.5$ & .763 & .522 & .714 & .537 & .420 & .441 & .618 & -.101 & .542 & .462 & .245 & .317 \\
 & $G=2, \lambda=2.0$ & .759 & .515 & .702 & .547 & .377 & .391 & .630 & -.130 & .546 & .452 & .215 & .287 \\
 & $G=5, \lambda=0.5$ & .758 & .546 & .736 & .493 & .401 & .453 & .591 & -.026 & .569 & .395 & .262 & .353 \\
 & $G=5, \lambda=1.0$ & .768 & .523 & .746 & .515 & .369 & .423 & .615 & -.086 & .563 & .447 & .274 & .376 \\
 & $G=5, \lambda=1.5$ & .761 & .453 & .723 & .513 & .391 & .427 & .623 & -.153 & .533 & .461 & .199 & .324 \\
 & $G=5, \lambda=2.0$ & .770 & .476 & .690 & .513 & .355 & .415 & .631 & -.093 & .548 & .453 & .132 & .254 \\
\midrule
\multirow{9}{*}{\multirowcell{Hybrid \\ Multinomial- \\ Beam}}
 & $B=1$ & .759 & .715 & .746 & .512 & .451 & .433 & .597 & .519 & .549 & .386 & .314 & .325 \\
 & $B=2$ & .765 & .731 & .745 & .519 & .470 & .435 & .617 & .564 & .586 & .445 & .338 & .345 \\
 & $B=3$ & \underline{.781} & .736 & .754 & .503 & .461 & .439 & .620 & .553 & .598 & \textbf{.519} & \underline{.436} & .424 \\
 & $B=4$ & \textbf{.784} & \textbf{.750} & \textbf{.769} & .516 & .467 & .428 & .622 & .572 & \underline{.617} & .436 & .388 & .382 \\
 & $B=5$ & .757 & .733 & .749 & .538 & .500 & .466 & \textbf{.655} & .578 & .609 & .470 & .412 & .419 \\
 & $B=6$ & .773 & .733 & .756 & .528 & \textbf{.512} & .488 & .635 & .586 & .613 & .486 & .421 & .415 \\
 & $B=7$ & .771 & .737 & .754 & .543 & .468 & .471 & .640 & .596 & .617 & .483 & .399 & .427 \\
 & $B=8$ & .764 & .733 & .755 & .548 & .504 & \textbf{.507} & \underline{.648} & \underline{.597} & .610 & .491 & .434 & .425 \\
 & $B=9$ & .772 & \underline{.747} & \underline{.765} & \underline{.551} & \underline{.509} & \underline{.497} & .646 & \textbf{.597} & \textbf{.618} & .501 & .427 & .439 \\
\bottomrule
\end{tabular}
}
\label{tab:abl_sampling}
\end{table}

  Evaluations use a subset of 500 examples from TriviaQA and CoQA with two base models, \gemmab and \llamab. Results are summarized in Table~\ref{tab:abl_sampling}.

  No single strategy dominates across datasets, models, or estimators. Temperature sampling (with importance weights) and diverse beam search systematically yield to beam search and hybrid multinomial-beam. Hybrid multinomial–beam strategy can reach top-1 for specific hyperparameter \(B\), but gains are not systematic and are sensitive to tuning. Given this variability and tuning cost, plain beam search with probability weighting is a reasonable default.

\subsection{Restricted-Mass Normalization}
\label{appendix:abl_normalization}
  Equation~\eqref{eq:restricted-mass} normalizes autoregressive sequence probabilities over the $M$ beam candidates. This choice can be sensitive to tail candidates whose probabilities are tiny and length-dependent. To test robustness, we introduce a floor $\epsilon$ on the per-candidate mass:
  \begin{equation}
    w_i^{\epsilon} = \frac{\max\!\left(\epsilon,\, p(\bv^{(i)} \mid \xv)\right)}{\sum_{j=1}^{M} \max\!\left(\epsilon,\, p(\bv^{(j)} \mid \xv)\right)}.
  \label{eq:epsilon_floor}
  \end{equation}
  The setting $\epsilon = 0$ recovers equation~\eqref{eq:restricted-mass}; $\epsilon = 1$ yields uniform weights $w_i^{1} = 1/M$. Intermediate $\epsilon$ values trade off fidelity to the model distribution against robustness to noisy, length-biased tails.

  We evaluate beam-guided probability-weighted methods for different $\epsilon$ on a subset of 500 examples from TriviaQA and CoQA with two base models, \gemmab and \llamab. Results are summarized in Table~\ref{tab:abl_normalization}.

  The results do not indicate a clear best choice of method and corresponding $\epsilon$ parameter.
  Determining the optimal $\epsilon$ is a case-dependent task.
  \begin{table}[h]

\caption{PRR (\(\uparrow\) is better) under restricted-mass normalization ablation. Columns group dataset–model pairs with methods (Dissim, Ecc, EigVecDissim). Rows vary the mass floor \(\epsilon\) in equation~\eqref{eq:epsilon_floor}: \(\epsilon=0\) recovers equation~\eqref{eq:restricted-mass}; \(\epsilon=1\) yields uniform weights \(w_i=1/M\). For each dataset–method, the top-1 score is \textbf{bold} and the second-best is \underline{underlined}.}

\centering
\small
\resizebox{\textwidth}{!}{%
\begin{tabular}{l|ccc|ccc|ccc|ccc}
\toprule
 & \multicolumn{6}{c|}{\gemmab} & \multicolumn{6}{c}{\llamab} \\
 & \multicolumn{3}{c|}{TriviaQA} & \multicolumn{3}{c|}{CoQA} & \multicolumn{3}{c|}{TriviaQA} & \multicolumn{3}{c}{CoQA} \\
 & Dissim & Ecc & \multirowcell{EigVec\\Dissim} & Dissim & Ecc & \multirowcell{EigVec\\Dissim} & Dissim & Ecc & \multirowcell{EigVec\\Dissim} & Dissim & Ecc & \multirowcell{EigVec\\Dissim} \\
\midrule
$\epsilon=1.0$ & .765 & \textbf{.741} & .744 & .536 & .487 & .461 & \textbf{.668} & .596 & .607 & .470 & .428 & .410 \\
$\epsilon=0.1$ & .765 & .727 & .745 & .556 & \textbf{.497} & .483 & \underline{.667} & \textbf{.612} & \textbf{.627} & \underline{.502} & \textbf{.447} & .446 \\
$\epsilon=0.05$ & .764 & .720 & .744 & .561 & \underline{.490} & .487 & .657 & \underline{.606} & \underline{.626} & \textbf{.509} & .435 & .451 \\
$\epsilon=0.01$ & .766 & .718 & .749 & .559 & .478 & \textbf{.489} & .630 & .584 & .602 & .496 & .437 & .452 \\
$\epsilon=0.001$ & \underline{.771} & .731 & \textbf{.751} & \textbf{.562} & .484 & \underline{.488} & .624 & .581 & .598 & .501 & \underline{.438} & .453 \\
$\epsilon=0.00001$ & \textbf{.771} & \underline{.732} & \underline{.751} & \underline{.561} & .483 & .488 & .623 & .581 & .598 & .502 & .438 & \textbf{.454} \\
$\epsilon=0$ & \textbf{.771} & \underline{.732} & \underline{.751} & \underline{.561} & .483 & .488 & .623 & .581 & .598 & .502 & .438 & \underline{.454} \\
\bottomrule
\end{tabular}
}
\label{tab:abl_normalization}
\end{table}

\subsection{Other Sampling-Based Methods Under Beam Search}
\label{appendix:abl_other_methods}
  Beyond Dissimilarity, Eccentricity, and EigVecDissimilarity, this ablation evaluates two other sampling-based methods under beam-generated candidates: \emph{Degree Matrix}~\citep{lin2023generating} and \emph{Semantic Entropy}~\citep{kuhn2023semantic}. We also provide probability-weighted beam formulations using the weights $w_i$ from equation~\eqref{eq:restricted-mass}.

\paragraph{Degree Matrix.}
  Given $M$ multinomial samples $\{\yv^{(i)}\}_{i=1}^{M}$, Degree Matrix estimates the average pairwise dissimilarity:
  \begin{equation}
    \widehat{U}_{DegMat}(\xv) = \frac{1}{M^2} \sum_{i=1}^{M} \sum_{j=1}^{M} \bigl(1 - s(\yv^{(i)}, \yv^{(j)})\bigr).
  \end{equation}
  For beam candidates \(\{\bv^{(i)}\}_{i=1}^{M}\), our mass-aware variant averages with weights:
  \begin{equation}
    \widehat{U}_{DegMat}^{b}(\xv) = \sum_{i=1}^{M} w_i \sum_{j=1}^{M} w_j \bigl(1 - s(\bv^{(i)}, \bv^{(j)})\bigr).
  \end{equation}

\paragraph{Semantic Entropy.}
  Multinomial samples are clustered into semantic equivalence classes \(C\). For each class, we calculate its probability
  \begin{equation}
    \hat{p}(c) = \frac{1}{M} \sum_{i=1}^{M} \mathbf{1}\{\yv^{(i)} \in c\} \quad \text{for } c \in C.
  \end{equation}
  Then Semantic Entropy calculates
  \begin{equation}
    \widehat{U}_{SemEnt}(\xv) = - \frac{1}{|C|} \sum_{c \in C} \log \hat{p}(c).
  \end{equation}
  For beam candidates, use cluster masses aggregated by \(w_i\):
  \begin{equation}
    \hat{p}^{b}(c) = \sum_{i=1}^{M} w_i \,\mathbf{1}\{\bv^{(i)} \in c\}, \qquad
    \widehat{U}_{SemEnt}^{b}(\xv) = - \sum_{c \in C} \hat{p}^{b}(c) \log \hat{p}^{b}(c).
  \end{equation}
  Note that these objectives score LLM uncertainty about the \emph{input} \(\xv\) as they are independent of a particular $\yv_*$.

  The results are summarized in Table~\ref{tab:abl_other_methods}. Beam search yields significant gains for Semantic Entropy and little to no improvement for Degree Matrix. Even with the beam-adapted formulations above, both objectives show worse results in terms of absolute PR-AUC compared to other methods. The primary reason is the target mismatch: as noted, these scores quantify uncertainty of the input $\xv$ and are independent of the produced answer $\yv_*$, whereas our main methods, Dissimilarity, Eccentricity and EigVecDissimilarity, focuses on ranking the correctness of $\yv_*$ itself.

  \begin{table}[t!]
    \caption{PR-AUC (\(\uparrow\) is better) on 6 datasets with \gemmab. Each method is shown as a pair: its multinomial-sampling variant and its beam-search variant; \(\uparrow\) denotes an improvement of the beam variant over its multinomial counterpart. Along main methods, the table includes input-uncertainty methods (Semantic Entropy, Lexical Similarity). For each dataset, the top-1 score is \textbf{bold} and the second-best is \underline{underlined}. The rightmost column reports the mean PR-AUC across datasets.}
    \centering
    \small
    \resizebox{\textwidth}{!}{
      \begin{tabular}{lccccccc}
\toprule
UQ Method & TriviaQA & \multirowcell{Web\\Questions} & CoQA & HotpotQA & \multirowcell{Common\\senceQA} & \multirowcell{ARC-\\Challenge} & Mean \\
\midrule
Semantic Entropy & .622 & .505 & .301 & .140 & .407 & .431 & .401 \\
Semantic Entropy + beamsearch & .685$\uparrow$ & .614$\uparrow$ & .365$\uparrow$ & .278$\uparrow$ & .436$\uparrow$ & .454$\uparrow$ & .472$\uparrow$ \\
\midrule
Degree Matrix & .682 & .605 & .385 & .311 & .409 & .419 & .469 \\
Degree Matrix + beamsearch & .673 & .642$\uparrow$ & .328 & .244 & .444$\uparrow$ & .473$\uparrow$ & .467 \\
\midrule
SAR & .656 & .571 & .347 & .296 & .183 & .264 & .386 \\
SAR + beamsearch & .671$\uparrow$ & .589$\uparrow$ & .329 & .266 & .209$\uparrow$ & .269$\uparrow$ & .372 \\
\midrule
Dissimilarity & \underline{.755} & \underline{.715} & \underline{.578} & \underline{.626} & .561 & .545 & .630 \\
Dissimilarity + beamsearch & \textbf{.766}$\uparrow$ & \textbf{.722}$\uparrow$ & \textbf{.600}$\uparrow$ & .611 & \textbf{.595}$\uparrow$ & .604$\uparrow$ & \textbf{.650}$\uparrow$ \\
\midrule
Eccentricity & .714 & .653 & .459 & .453 & .549 & .549 & .563 \\
Eccentricity + beamsearch & .739$\uparrow$ & .633 & .505$\uparrow$ & .514$\uparrow$ & \underline{.590}$\uparrow$ & \textbf{.636}$\uparrow$ & .603$\uparrow$ \\
\midrule
EigVecDissimilarity & .738 & .661 & .443 & .448 & .512 & .562 & .561 \\
EigVecDissimilarity + beamsearch & .753$\uparrow$ & .668$\uparrow$ & .497$\uparrow$ & .487$\uparrow$ & .562$\uparrow$ & \underline{.621}$\uparrow$ & .598$\uparrow$ \\
\midrule
CocoaMSP & .738 & .666 & .509 & .430 & .583 & .595 & .587 \\
CocoaMSP + beamsearch & .747$\uparrow$ & .679$\uparrow$ & .548$\uparrow$ & .523$\uparrow$ & .586$\uparrow$ & .606$\uparrow$ & .615$\uparrow$ \\
\midrule
CocoaPPL & .739 & .678 & .548 & .625 & .580 & .595 & .628 \\
CocoaPPL + beamsearch & .748$\uparrow$ & .694$\uparrow$ & .577$\uparrow$ & \textbf{.681}$\uparrow$ & .582$\uparrow$ & .610$\uparrow$ & \underline{.649}$\uparrow$ \\
\bottomrule
\end{tabular}
    }
  \label{tab:abl_other_methods}
  \end{table}

  \begin{figure}[h]
    \centering
    \includegraphics[width=.5\linewidth, trim=0 2em 0 2.5em]{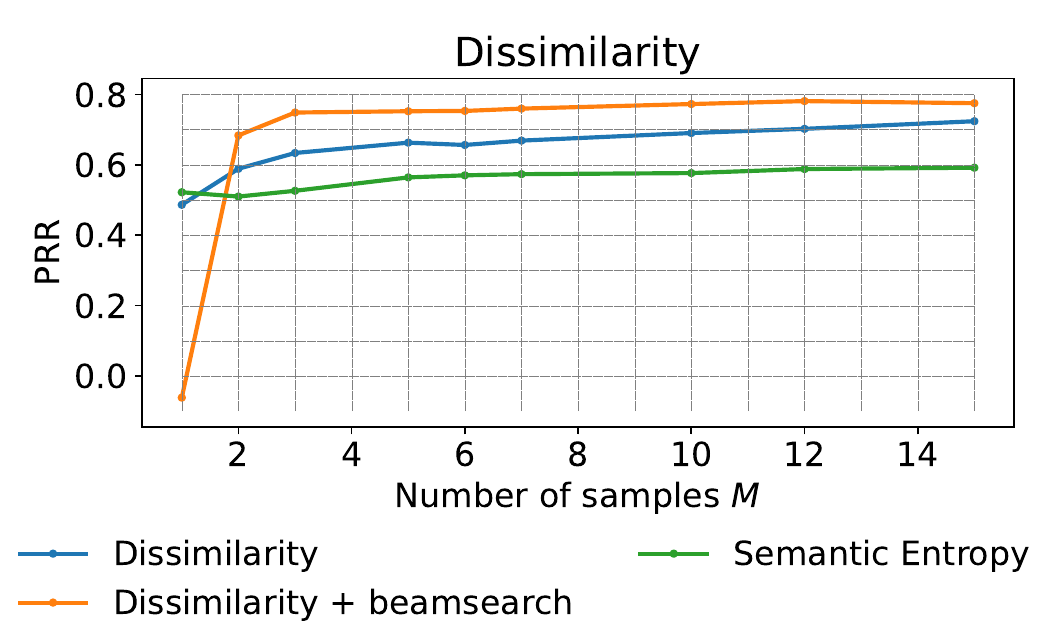}
    \caption{PRR ($\uparrow$ is better) as a function of the number of candidates $M$ on TriviaQA with \gemmab for 3 UQ methods: Semantic Entropy, and sampling and beam search versions of Dissimilarity.}
  \label{fig:sement_different_nsamples}
  \end{figure}

  To further assess performance under different numbers of samples $M$ used for UQ, we plot PRR as a function of $M$ for one selected baseline, Semantic Entropy, as well as for Dissimilarity (both sampling and beam-search variants) for reference. Figure~\ref{fig:sement_different_nsamples} presents the results, showing that for all $M>1$, Semantic Entropy underperforms both Dissimilarity variants. This occurs because Dissimilarity measures the targeted uncertainty of the specific generation $\yv_*$ rather than the overall uncertainty associated with $\xv$, measured by Semantic Entropy.

\subsection{Graph Laplacian Embedding Parameters}
\label{appendix:abl_ecc_params}
  Both multinomial and beam-guided versions of Eccentricity and EigVecDissimilarity depend on the threshold parameter $\alpha$, which selects eigenvectors of the Graph Laplacian $L = I - D^{-1/2} W D^{-1/2}$ used to form semantic embeddings. Specifically, after computing the eigenpairs $\{\lambda_i, \uv_i\}_{i=1}^{M+1}$, we retain those with $\lambda_i < \alpha$, yielding $K$ eigenvectors in total and embeddings $\vv_j = [\uv_{1j}, \uv_{2j}, \dots, \uv_{Kj}]$ of dimension $K$. All eigenvalues lie in $[0, 1]$; smaller values capture stronger semantic structure, whereas values closer to $1$ tend to reflect noise~\citep{lin2023generating}. In the main experiments we follow the original Eccentricity setting and use $\alpha=0.9$.

  Here we vary $\alpha$ and also test a fixed-$K$ strategy (i.e., keeping exactly $K$ leading low-spectrum eigenvectors irrespective of the threshold).

  \begin{table}[t]

\caption{PRR (\(\uparrow\) is better) for Eccentricity and EigVecDissimilarity under different Graph Laplacian embedding choices on four dataset–model pairs. Top block varies the eigenvalue threshold \(\alpha\) (retaining all \(\lambda_i < \alpha\)); bottom block fixes the embedding dimension \(K\). For each pair, the best score is \textbf{bold} and the second-best is \underline{underlined}.}

\centering

\small
\resizebox{0.7\textwidth}{!}{%
\begin{tabular}{l|cc|cc|cc|cc}
\toprule
 & \multicolumn{4}{c|}{\gemmab} & \multicolumn{4}{c}{\llamab} \\
 & \multicolumn{2}{c|}{TriviaQA} & \multicolumn{2}{c|}{CoQA} & \multicolumn{2}{c|}{TriviaQA} & \multicolumn{2}{c}{CoQA} \\
 & Ecc & \multirowcell{EigVec\\Dissim} & Ecc & \multirowcell{EigVec\\Dissim} & Ecc & \multirowcell{EigVec\\Dissim} & Ecc & \multirowcell{EigVec\\Dissim} \\
\midrule
$\alpha=0.3$ & .717 & .710 & .434 & .355 & .601 & .599 & .408 & .364 \\
$\alpha=0.5$ & \textbf{.752} & .740 & .497 & .460 & \textbf{.627} & .626 & .431 & .420 \\
$\alpha=0.7$ & .750 & .749 & \textbf{.539} & \textbf{.498} & \underline{.622} & \textbf{.643} & .438 & .441 \\
$\alpha=0.8$ & \underline{.751} & \textbf{.757} & \underline{.508} & .475 & .616 & \underline{.630} & \textbf{.444} & \underline{.450} \\
$\alpha=0.9$ & .732 & .751 & .483 & \underline{.488} & .581 & .598 & \underline{.438} & \textbf{.454} \\
$\alpha=0.99$ & .725 & \underline{.755} & .432 & .454 & .535 & .561 & .397 & .409 \\
\midrule
$K=1$ & .454 & .358 & .346 & .332 & .444 & .357 & .304 & .280 \\
$K=2$ & .510 & .532 & .383 & .361 & .474 & .469 & .318 & .338 \\
$K=3$ & .619 & .639 & .434 & .429 & .538 & .534 & .351 & .348 \\
$K=4$ & .645 & .643 & .418 & .412 & .519 & .514 & .359 & .352 \\
$K=5$ & .638 & .655 & .366 & .352 & .487 & .492 & .314 & .316 \\
$K=6$ & .529 & .545 & .244 & .236 & .363 & .368 & .210 & .211 \\
$K=7$ & .210 & .242 & -.101 & -.076 & .050 & .062 & -.208 & -.211 \\
$K=8$ & -.265 & -.209 & -.210 & -.171 & -.358 & -.349 & -.327 & -.308 \\
$K=9$ & -.566 & -.414 & -.268 & -.186 & -.462 & -.410 & -.339 & -.317 \\
$K=10$ & -.659 & -.484 & -.283 & -.189 & -.467 & -.397 & -.330 & -.249 \\
\bottomrule
\end{tabular}
}
\label{tab:abl_ecc_params}
\end{table}

  Table~\ref{tab:abl_ecc_params} reports the performance for Eccentricity and EigVecDissimilarity across $\alpha$ and $K$ on four dataset–model pairs. A fixed embedding size performs poorly: the optimal number of informative directions varies between candidate sets, so fixing $K$ either underfits or includes noisy directions. Thresholding is more robust: $\alpha \in [0.7, 0.9]$ consistently yields strong results across methods and pairs, supporting our default choice $\alpha=0.9$.

\subsection{Cross-Encoder Similarity}
\label{appendix:abl_crossencoder}
  In the main text, we instantiate the similarity function $s$ using an NLI score: the entailment probability from a DeBERTa model. CoCoA, however, originally used a RoBERTa-large \emph{cross-encoder} fine-tuned on the Semantic Textual Similarity benchmark~\citep{liu2020roberta}. Table~\ref{tab:abl_crossencoder} reports PRR for \gemmab when replacing the NLI-based $s$ with this cross-encoder; all other settings are unchanged.

  \begin{table}[h]
    \caption{PRR (\(\uparrow\) is better) on 6 datasets with \gemmab using a RoBERTa-large cross-encoder (STS) as the similarity function \(s\) in place of NLI. For each dataset, the top-1 is \textbf{bold} and the second-best is \underline{underlined}; \(\uparrow\) marks an improvement of a beam variant over its multinomial counterpart.}
    \centering
    \small
    \resizebox{\textwidth}{!}{
      \begin{tabular}{lccccccc}
\toprule
UQ Method & TriviaQA & \multirowcell{Web\\Questions} & CoQA & HotpotQA & \multirowcell{Common\\senceQA} & \multirowcell{ARC-\\Challenge} & Mean \\
\midrule
Dissimilarity & .725 & \underline{.683} & .497 & .597 & .481 & .421 & .567 \\
Dissimilarity + beamsearch & \textbf{.746}$\uparrow$ & \textbf{.693}$\uparrow$ & \textbf{.513}$\uparrow$ & \textbf{.654}$\uparrow$ & .505$\uparrow$ & .479$\uparrow$ & \underline{.598}$\uparrow$ \\
\midrule
Eccentricity & .722 & .647 & .489 & .544 & .455 & .500 & .560 \\
Eccentricity + beamsearch & .734$\uparrow$ & .647$\uparrow$ & .483 & .604$\uparrow$ & .362 & .421 & .542 \\
\midrule
EigVecDissimilarity & .737 & .649 & .453 & .523 & .489 & .529 & .563 \\
EigVecDissimilarity + beamsearch & \underline{.744}$\uparrow$ & .675$\uparrow$ & .484$\uparrow$ & .582$\uparrow$ & .439 & .496 & .570$\uparrow$ \\
\midrule
CocoaMSP & .731 & .642 & .438 & .397 & \underline{.553} & .577 & .556 \\
CocoaMSP + beamsearch & .740$\uparrow$ & .648$\uparrow$ & .462$\uparrow$ & .479$\uparrow$ & \textbf{.558}$\uparrow$ & \textbf{.593}$\uparrow$ & .580$\uparrow$ \\
\midrule
CocoaPPL & .728 & .653 & .488 & .607 & .546 & .567 & .598 \\
CocoaPPL + beamsearch & .737$\uparrow$ & .658$\uparrow$ & \underline{.498}$\uparrow$ & \underline{.650}$\uparrow$ & .548$\uparrow$ & \underline{.586}$\uparrow$ & \textbf{.613}$\uparrow$ \\
\bottomrule
\end{tabular}
    }
  \label{tab:abl_crossencoder}
  \end{table}

\subsection{Number of Samples Across Different Tasks}

  \begin{figure}[h]
    \centering
    \includegraphics[width=\linewidth, trim=0 2em 0 2.5em]{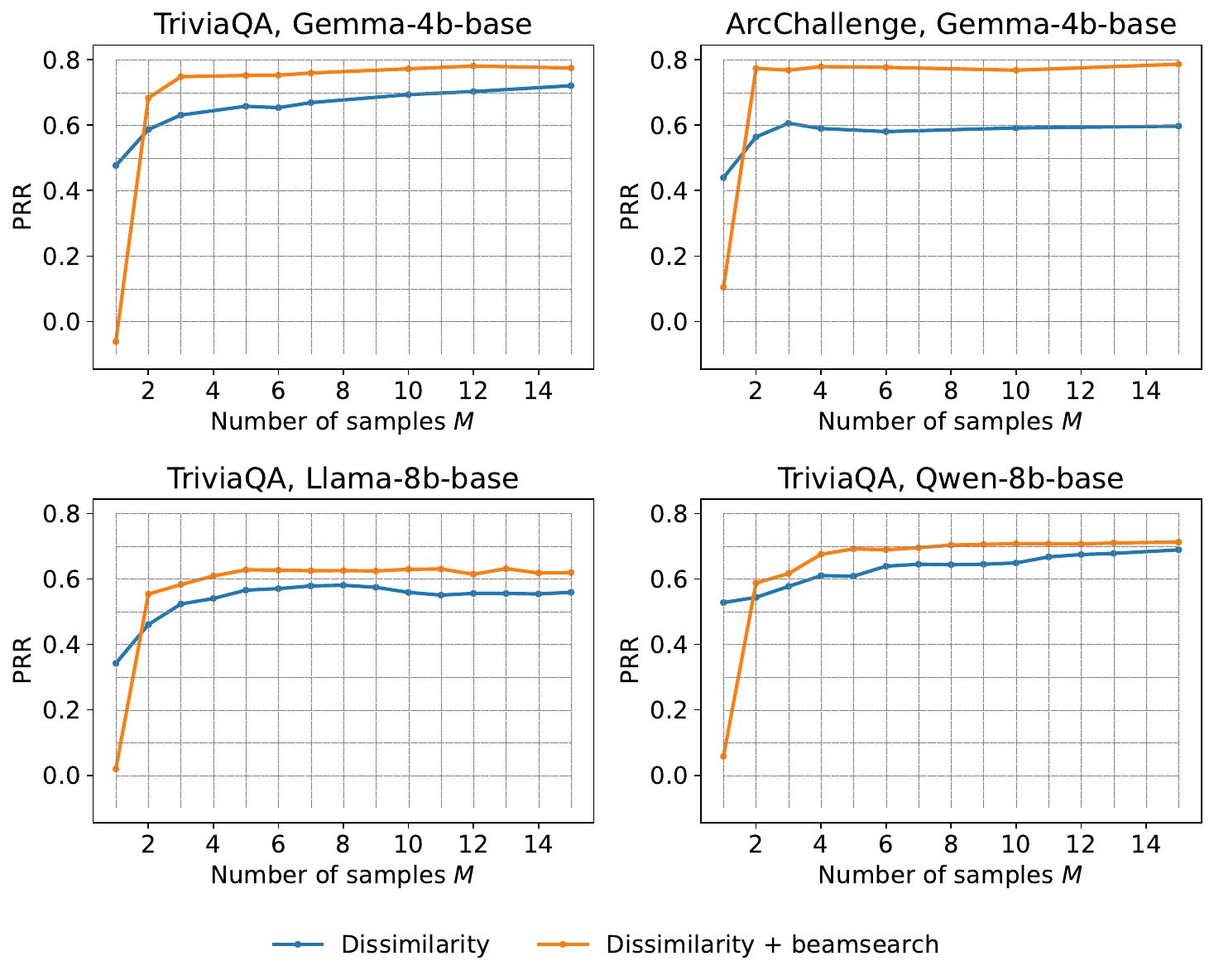}
    \caption{PRR ($\uparrow$ is better) as a function of the number of candidates $M$ across different datasets and models.}
    \label{fig:rebuttal_nsamples_across_data}
  \end{figure}

  To evaluate the performance of the proposed beam-search variations under different numbers of samples $M$ across models, we computed PRR for both the sampling and beam-search versions of Dissimilarity on 200 random subsamples of TriviaQA for three LLMs: \gemmab, \llamab, and \qwenb. To further assess performance across datasets, we additionally evaluated PRR for \gemmab on the 200 subsamples of ARC-Challenge dataset. All four resulting plots are shown in Figure~\ref{fig:rebuttal_nsamples_across_data}.

  The results show that for all budgets $M > 1$, beam search consistently outperforms sampling, yielding higher PRR. On the open-ended TriviaQA dataset, PRR increases steadily with $M$, with beam search reaching a plateau around $M = 5$ for all 3 models tested. On the multiple-choice ARC-Challenge dataset, PRR plateaus at a considerably smaller budget ($M = 2$), likely due to the small output space (i.e., a limited set of answer choices).

  Overall, these results indicate that the beam-search variant of Dissimilarity remains effective even at relatively small sample budgets: $M \approx 5$ for open-ended short-form generation tasks, and $M = 2$ for multiple-choice settings, where the constrained output space enables faster saturation.

\clearpage

\section{Analysis and Examples}

\subsection{Probability Mass Coverage}
\label{appendix:abl_prob_mass}

  \begin{figure}[h]
    \centering
    \includegraphics[width=\linewidth, trim=0 2em 0 2.5em]{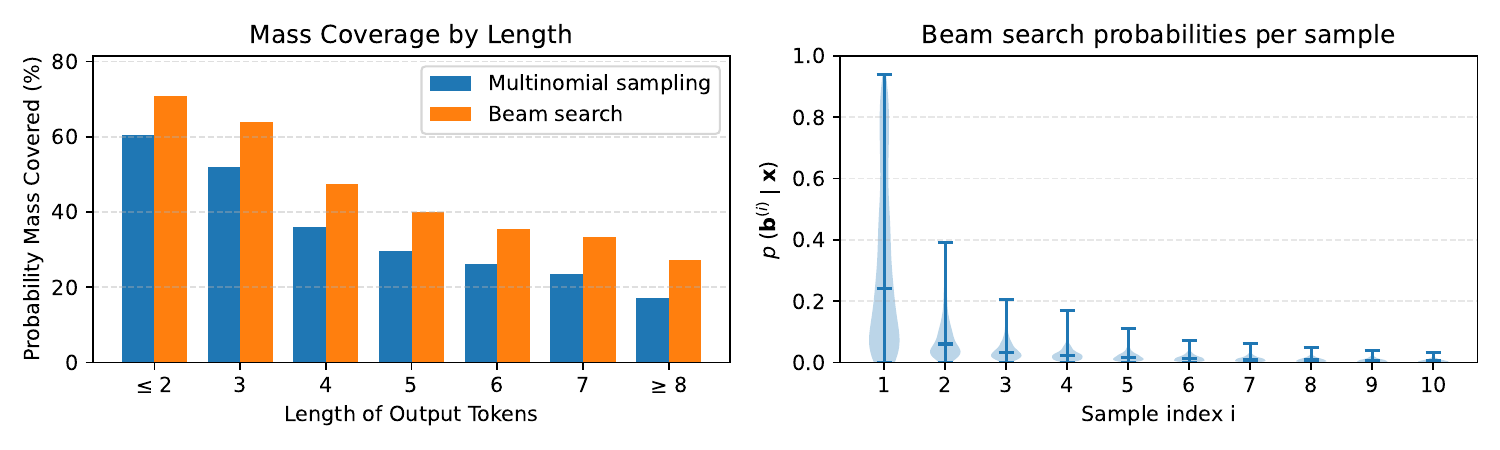}
    \caption{\textit{Left}: average probability mass covered by the candidate set ($M{=}10$) across output-length bins (averaged over examples in the bin) on TriviaQA with \gemmab. \textit{Right}: for beam search, distribution of sequence probabilities $p(\bv^{(i)} \mid \xv)$ by beam rank $i$ (1 = highest-probability text).}
  \label{fig:abl_prob_mass}
  \end{figure}

  Figure~\ref{fig:abl_prob_mass} summarizes two observations. First (left), beam search covers a larger share of the model's probability mass than multinomial sampling across length bins. Second (right), beam probabilities decay sharply with rank: the first few beams capture most of the mass, while lower-ranked beams contribute little. This motivates mass-aware weighting $w_i$ (see equation~\eqref{eq:restricted-mass}) and helps explain why probability-weighted beam variants are effective, especially at small candidate budgets.

\subsection{Examples}
  We include qualitative examples for \gemmab: two from TriviaQA, two from WebQuestions, and one from CoQA. Each panel shows the question, the greedy answer, ten multinomial samples, and ten beam-search samples with autoregressive probabilities, together with the corresponding uncertainty scores (e.g., Dissimilarity and its beam-guided variant). The cases illustrate how beam search reduces duplication and enhances uncertainty.

  \begin{figure}[h]
\centering
\begin{minipage}[t]{0.48\linewidth}
\begin{tcolorbox}[colback=black!1, colframe=black!12, arc=2mm, boxrule=0.4pt, left=6pt, right=6pt, top=6pt, bottom=6pt]
\small
\setlength{\tabcolsep}{6pt}
\textbf{Question:} What claimed the life of singer Kathleen Ferrier?\\
\textbf{Greedy:} breasts cancer\\[4pt]
\begin{tabular}{@{}p{0.35\linewidth}|p{0.6\linewidth}@{}}
\midrule
\multicolumn{1}{c|}{\multirowcell{Multinomial\\samples}} &
\multicolumn{1}{c}{\multirowcell{Beam-search\\samples}} \\
\midrule
cancer & cancer\hfill\textcolor{gray}{\footnotesize p=0.228}\\
breast cancer & tuberculosis\hfill\textcolor{gray}{\footnotesize p=0.154}\\
pulmonary & breast cancer\hfill\textcolor{gray}{\footnotesize p=0.089}\\
breast cancer & lung cancer\hfill\textcolor{gray}{\footnotesize p=0.041}\\
cancer & pneumonia\hfill\textcolor{gray}{\footnotesize p=0.039}\\
breast cancer & leukaemia\hfill\textcolor{gray}{\footnotesize p=0.034}\\
myx & myel\hfill\textcolor{gray}{\footnotesize p=0.023}\\
cancer & leuk\hfill\textcolor{gray}{\footnotesize p=0.011}\\
cancer & pulmonary\hfill\textcolor{gray}{\footnotesize p=0.011}\\
pneumonia & lymphoma\hfill\textcolor{gray}{\footnotesize p=0.010}\\
\midrule
\end{tabular}\\[4pt]
\textbf{Dissimilarity:} 0.330 \\
\textbf{Dissimilarity + beamsearch:} 0.533
\end{tcolorbox}
\end{minipage}
\hfill
\begin{minipage}[t]{0.48\linewidth}
\begin{tcolorbox}[colback=black!1, colframe=black!12, arc=2mm, boxrule=0.4pt, left=6pt, right=6pt, top=6pt, bottom=6pt]
\small
\setlength{\tabcolsep}{6pt}
\textbf{Question:} Which number Beethoven symphony is known as ‘The Pastoral’?\\
\textbf{Greedy:} 6\\[4pt]
\begin{tabular}{@{}p{0.35\linewidth}|p{0.6\linewidth}@{}}
\midrule
\multicolumn{1}{c|}{\multirowcell{Multinomial\\samples}} &
\multicolumn{1}{c}{\multirowcell{Beam-search\\samples}} \\
\midrule
six & sixth\hfill\textcolor{gray}{\footnotesize p=0.314}\\
seventh & 6\hfill\textcolor{gray}{\footnotesize p=0.169}\\
sixth & 6th\hfill\textcolor{gray}{\footnotesize p=0.104}\\
sixth & ninth\hfill\textcolor{gray}{\footnotesize p=0.061}\\
sixth & seventh\hfill\textcolor{gray}{\footnotesize p=0.037}\\
6 & 9\hfill\textcolor{gray}{\footnotesize p=0.027}\\
seventh & six\hfill\textcolor{gray}{\footnotesize p=0.023}\\
no & 9th\hfill\textcolor{gray}{\footnotesize p=0.021}\\
sixteenth & no.\hfill\textcolor{gray}{\footnotesize p=0.013}\\
n6 & 7\hfill\textcolor{gray}{\footnotesize p=0.008}\\
\midrule
\end{tabular}\\[4pt]
\textbf{Dissimilarity:} 0.634 \\
\textbf{Dissimilarity + beamsearch:} 0.561
\end{tcolorbox}
\end{minipage}
\caption{Two examples from \gemmab on TriviaQA. Each panel shows the question, greedy answer, multinomial and beam-search samples with autoregressive probabilities, plus dissimilarity and beamsearch-guided dissimilarity.}
\label{fig:triviaqa-gemma4b}
\end{figure}
  \begin{figure}[h]
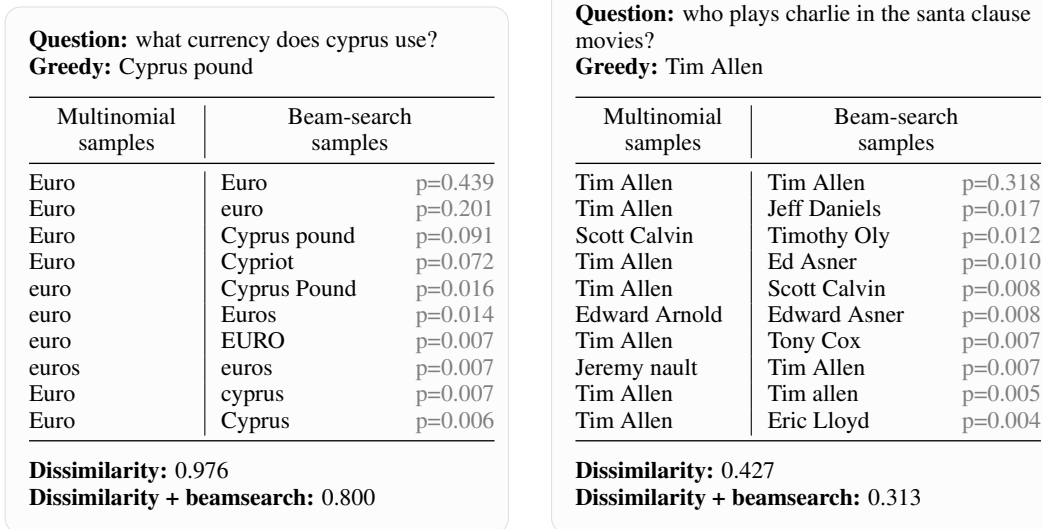

\centering
\begin{minipage}[t]{0.48\linewidth}
\begin{tcolorbox}[colback=black!1, colframe=black!12, arc=2mm, boxrule=0.4pt, left=6pt, right=6pt, top=6pt, bottom=6pt]
\small
\setlength{\tabcolsep}{6pt}
\textbf{Question:} what currency does cyprus use?\\
\textbf{Greedy:} Cyprus pound\\[4pt]
\begin{tabular}{@{}p{0.35\linewidth}|p{0.6\linewidth}@{}}
\midrule
\multicolumn{1}{c|}{\multirowcell{Multinomial\\samples}} &
\multicolumn{1}{c}{\multirowcell{Beam-search\\samples}} \\
\midrule
Euro & Euro\hfill\textcolor{gray}{\footnotesize p=0.439}\\
Euro & euro\hfill\textcolor{gray}{\footnotesize p=0.201}\\
Euro & Cyprus pound\hfill\textcolor{gray}{\footnotesize p=0.091}\\
Euro & Cypriot\hfill\textcolor{gray}{\footnotesize p=0.072}\\
euro & Cyprus Pound\hfill\textcolor{gray}{\footnotesize p=0.016}\\
euro & Euros\hfill\textcolor{gray}{\footnotesize p=0.014}\\
euro & EURO\hfill\textcolor{gray}{\footnotesize p=0.007}\\
euros & euros\hfill\textcolor{gray}{\footnotesize p=0.007}\\
Euro & cyprus\hfill\textcolor{gray}{\footnotesize p=0.007}\\
Euro & Cyprus\hfill\textcolor{gray}{\footnotesize p=0.006}\\
\midrule
\end{tabular}\\[4pt]
\textbf{Dissimilarity:} 0.976 \\
\textbf{Dissimilarity + beamsearch:} 0.800
\end{tcolorbox}
\end{minipage}
\hfill
\begin{minipage}[t]{0.48\linewidth}
\begin{tcolorbox}[colback=black!1, colframe=black!12, arc=2mm, boxrule=0.4pt, left=6pt, right=6pt, top=6pt, bottom=6pt]
\small
\setlength{\tabcolsep}{6pt}
\textbf{Question:} who plays charlie in the santa clause movies?\\
\textbf{Greedy:} Tim Allen\\[4pt]
\begin{tabular}{@{}p{0.35\linewidth}|p{0.6\linewidth}@{}}
\midrule
\multicolumn{1}{c|}{\multirowcell{Multinomial\\samples}} &
\multicolumn{1}{c}{\multirowcell{Beam-search\\samples}} \\
\midrule
Tim Allen & Tim Allen\hfill\textcolor{gray}{\footnotesize p=0.318}\\
Tim Allen & Jeff Daniels\hfill\textcolor{gray}{\footnotesize p=0.017}\\
Scott Calvin & Timothy Oly\hfill\textcolor{gray}{\footnotesize p=0.012}\\
Tim Allen & Ed Asner\hfill\textcolor{gray}{\footnotesize p=0.010}\\
Tim Allen & Scott Calvin\hfill\textcolor{gray}{\footnotesize p=0.008}\\
Edward Arnold & Edward Asner\hfill\textcolor{gray}{\footnotesize p=0.008}\\
Tim Allen & Tony Cox\hfill\textcolor{gray}{\footnotesize p=0.007}\\
Jeremy nault & Tim Allen\hfill\textcolor{gray}{\footnotesize p=0.007}\\
Tim Allen & Tim allen\hfill\textcolor{gray}{\footnotesize p=0.005}\\
Tim Allen & Eric Lloyd\hfill\textcolor{gray}{\footnotesize p=0.004}\\
\midrule
\end{tabular}\\[4pt]
\textbf{Dissimilarity:} 0.427 \\
\textbf{Dissimilarity + beamsearch:} 0.313
\end{tcolorbox}
\end{minipage}
\caption{Two examples from \gemmab on WebQ. Each panel shows the question, greedy answer, multinomial and beam-search samples with autoregressive probabilities, plus dissimilarity and beamsearch-guided dissimilarity.}
\label{fig:webq-gemma4b}
\end{figure}
  \begin{figure}[h]
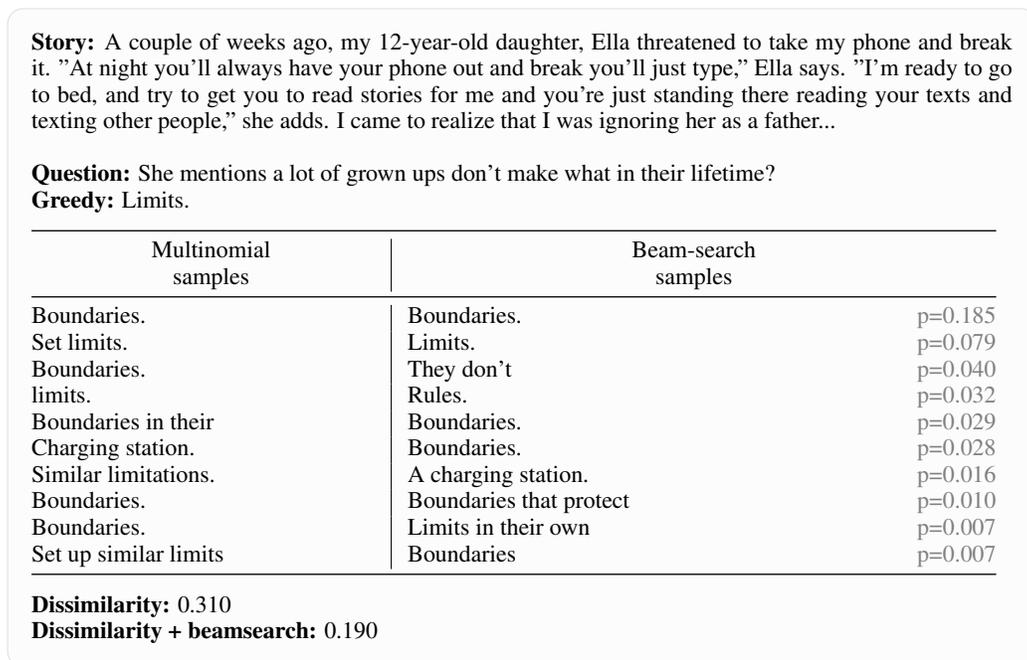

\centering
\begin{minipage}[t]{0.98\linewidth}
\begin{tcolorbox}[colback=black!1, colframe=black!12, arc=2mm, boxrule=0.4pt, left=6pt, right=6pt, top=6pt, bottom=6pt]
\small
\setlength{\tabcolsep}{6pt}
\textbf{Story:} A couple of weeks ago, my 12-year-old daughter, Ella threatened to take my phone and break it. "At night you'll always have your phone out and break you'll just type," Ella says. "I'm ready to go to bed, and try to get you to read stories for me and you're just standing there reading your texts and texting other people," she adds. I came to realize that I was ignoring her as a father...\\\\
\textbf{Question:} She mentions a lot of grown ups don't make what in their lifetime?\\
\textbf{Greedy:} Limits.\\[4pt]
\begin{tabular}{@{}p{0.35\linewidth}|p{0.6\linewidth}@{}}
\midrule
\multicolumn{1}{c|}{\multirowcell{Multinomial\\samples}} &
\multicolumn{1}{c}{\multirowcell{Beam-search\\samples}} \\
\midrule
Boundaries. & Boundaries.\hfill\textcolor{gray}{\footnotesize p=0.185}\\
Set limits. & Limits.\hfill\textcolor{gray}{\footnotesize p=0.079}\\
Boundaries. & They don't\hfill\textcolor{gray}{\footnotesize p=0.040}\\
limits. & Rules.\hfill\textcolor{gray}{\footnotesize p=0.032}\\
Boundaries in their & Boundaries.\hfill\textcolor{gray}{\footnotesize p=0.029}\\
Charging station. & Boundaries.\hfill\textcolor{gray}{\footnotesize p=0.028}\\
Similar limitations. & A charging station.\hfill\textcolor{gray}{\footnotesize p=0.016}\\
Boundaries. & Boundaries that protect\hfill\textcolor{gray}{\footnotesize p=0.010}\\
Boundaries. & Limits in their own\hfill\textcolor{gray}{\footnotesize p=0.007}\\
Set up similar limits & Boundaries\hfill\textcolor{gray}{\footnotesize p=0.007}\\
\midrule
\end{tabular}\\[4pt]
\textbf{Dissimilarity:} 0.310 \\
\textbf{Dissimilarity + beamsearch:} 0.190
\end{tcolorbox}
\end{minipage}
\caption{One example from \gemmab on CoQA. Shown are the question, greedy answer, multinomial and beam-search samples with autoregressive probabilities, plus dissimilarity and beamsearch-guided dissimilarity.}
\label{fig:coqa-gemma4b-one}
\end{figure}

\clearpage

\section{Datasets}
\label{appendix:prompts}
  Table~\ref{tab:prompts} lists the prompts used to form inputs for each dataset (separately for base and instruct models). Table~\ref{tab:base_accs} reports mean accuracy for each model–dataset pair. We measure accuracy as the fraction of predictions whose AlignScore with the gold answer exceeds $0.5$.
  
  \newcolumntype{S}{>{\hsize=.65\hsize\raggedright\arraybackslash}X} 
\newcolumntype{B}{>{\hsize=1.35\hsize\raggedright\arraybackslash}X}

\begin{table*}[h]

\caption{Prompt templates used for each dataset and model type. Few-shot exemplars are shown as placeholders (e.g., \texttt{<5 few-shot QA pairs>}); run-time inputs are denoted by \texttt{<question>}, \texttt{<context>}, \texttt{<title 1>}, etc.}

\centering
\small
\begin{tabularx}{\textwidth}{@{}l | S | B@{}}
\toprule
\textbf{Dataset} & \textbf{Base Prompt} & \textbf{Instruct Prompt} \\
\midrule

TriviaQA &
\begin{minipage}[t]{.65\linewidth}\ttfamily\scriptsize
<5 few-shot QA pairs>\\
Question: <question>\\
Answer:
\end{minipage}
&
\begin{minipage}[t]{1.35\linewidth}\ttfamily\scriptsize
Answer the following question as briefly as possible.\\
<5 few-shot QA pairs>\\
Now answer the following question:\\
Question: <question>\\
Answer:
\end{minipage}
\\
\midrule

\begin{minipage}[t]{0.03\linewidth}
Web\\
Questions
\end{minipage} &
\begin{minipage}[t]{.65\linewidth}\ttfamily\scriptsize
<5 few-shot QA pairs>\\
Question: <question>\\
Answer:
\end{minipage}
&
\begin{minipage}[t]{1.35\linewidth}\ttfamily\scriptsize
Below are questions with short factual answers.\\
Return only the short answer (a name, phrase, number, or year).\\
<5 few-shot QA pairs>\\
Now answer this.\\
Q: <question>\\
A:
\end{minipage}
\\
\midrule

CoQA &
\begin{minipage}[t]{.65\linewidth}\ttfamily\scriptsize
Story: <context>\\
<all preceding QA pairs>\\
Question: <question>\\
Answer:
\end{minipage}
&
\begin{minipage}[t]{1.35\linewidth}\ttfamily\scriptsize
Story: <context>\\
<all preceding QA pairs>\\
Answer the following question as briefly as possible.\\
Question: <question>\\
Answer:
\end{minipage}
\\
\midrule

HotpotQA &
\begin{minipage}[t]{.65\linewidth}\ttfamily\scriptsize
Title: <title 1>\\
<paragraph 1>\\[0.25em]
Title: <title 2>\\
<paragraph 2>\\[0.25em]
Question: <question>\\
Short answer:
\end{minipage}
&
\begin{minipage}[t]{1.35\linewidth}\ttfamily\scriptsize
Instruction: Read the context and answer with a short factual span (a few words) copied from the context. Reply with the short answer only.\\
Title: <title 1>\\
<paragraph 1>\\
Title: <title 2>\\
<paragraph 2>\\
Question: <question>\\
Short answer:
\end{minipage}
\\
\midrule

\begin{minipage}[t]{0.03\linewidth}
Common\\
senceQA
\end{minipage} &
\begin{minipage}[t]{.65\linewidth}\ttfamily\scriptsize
<2 few-shot QA pairs>\\
Question: <question>\\
Options:
<(A) - (D) options>\\
Answer:
\end{minipage}
&
\begin{minipage}[t]{1.35\linewidth}\ttfamily\scriptsize
Instruction: Choose the single best answer from the options. Answer with the option text only (not the letter).\\
<2 few-shot QA pairs>\\
Now answer this.\\
Question: <question>\\
Options:\\
<(A) - (D) options>\\
Answer:
\end{minipage}
\\
\midrule

\begin{minipage}[t]{0.05\linewidth}
ARC-Challenge
\end{minipage} &
\begin{minipage}[t]{.65\linewidth}\ttfamily\scriptsize
<2 few-shot QA pairs>\\
Question: <question>\\
Options:\\
<(A) - (D) options>\\
Answer:
\end{minipage}
&
\begin{minipage}[t]{1.35\linewidth}\ttfamily\scriptsize
Instruction: Choose the single best answer from the options. Answer with the option text only (not the letter).\\
<2 few-shot QA pairs>\\
Now answer this.\\
Question: <question>\\
Options:\\
<(A) - (D) options>\\
Answer:
\end{minipage}
\\
\bottomrule
\end{tabularx}
\label{tab:prompts}
\end{table*}

  \begin{table*}[h]

\caption{Mean accuracy (\%): proportion of predictions with AlignScore to the gold answer $> 0.5$.}

\centering
\small
\begin{tabular}{lcccccc}
\toprule
 &
\multicolumn{2}{c}{Closed-Book QA} &
\multicolumn{2}{c}{Open-Book QA} &
\multicolumn{2}{c}{Multiple Choice} \\
\cmidrule(lr){2-3}\cmidrule(lr){4-5}\cmidrule(lr){6-7}
 & TriviaQA & \multirowcell{Web\\Questions} & CoQA & HotpotQA & \multirowcell{Common\\senceQA} & \multirowcell{ARC-\\Challenge} \\
\midrule

\llamab & 63\% & 47\% & 74\% & 53\% & 74\% & 72\% \\
\llamait & 69\% & 40\% & 80\% & 72\% & 77\% & 76\% \\
\gemmab & 47\% & 33\% & 69\% & 41\% & 65\% & 70\% \\
\gemmait & 51\% & 35\% & 76\% & 66\% & 76\% & 77\% \\
\qwenb & 52\% & 48\% & 81\% & 47\% & 89\% & 91\% \\
\qwenit & 54\% & 42\% & 76\% & 76\% & 84\% & 88\% \\

\bottomrule
\end{tabular}
\label{tab:base_accs}
\end{table*}

\clearpage
\section{Additional Results}

\subsection{Scoring Top-Beam Output}
\label{appendix:topbeam_scoring}
  In the main text we score the greedy decode as the produced answer $\yv_*$. Table~\ref{tab:topbeam_scoring} complements these results by scoring the \emph{top-1 beam} as $\yv_*$, a natural choice when beam search is already used to obtain a higher-quality decode.
  The beam-weighted family of approaches achieves higher PRR than the original methods and baselines in the majority of cases.

  \begin{table}[h]
    \caption{PRR ($\uparrow$ is better) averaged over 6 datasets, when scoring the top-1 beam produced answer (instead of greedy). For each dataset, the top-1 score is \textbf{bold} and the second-best is \underline{underlined}. For beam-guided variants, we mark $\uparrow$ when the variant improves over its original multinomial-sampling counterpart.}
    \centering
    \small
    \resizebox{0.85\textwidth}{!}{
      \begin{tabular}{l|cc|cc|cc}
\toprule
\multirow{2}{*}{UQ Method} & \multicolumn{2}{c|}{\llama} & \multicolumn{2}{c|}{\gemma} & \multicolumn{2}{c}{\qwen} \\
 & base & instruct & base & instruct & base & instruct \\
\midrule
 \rowcolor{gray!20}
 \multicolumn{7}{c}{\textit{Baseline UQ methods}} \\
\midrule
Prob & .399 & .174 & .400 & .213 & .390 & .090 \\
MTE & .320 & .164 & .317 & .228 & .334 & .255 \\
Perplexity & .376 & .121 & .359 & .185 & .318 & .009 \\
CCP & .395 & .155 & .369 & .243 & .352 & .226 \\
SAR & .333 & .221 & .336 & .348 & .342 & .246 \\
P(True) & .019 & -.075 & .031 & .090 & .012 & -.080 \\
SemanticEntropy & .345 & .286 & .397 & .320 & .299 & .250 \\
LexicalSimilarity & .377 & .221 & .384 & .291 & .404 & .210 \\
EigValLaplacian & .366 & .209 & .402 & .307 & .384 & .223 \\
NumSemSets & .349 & .215 & .365 & .262 & .344 & .208 \\
\midrule
 \rowcolor{gray!20}
 \multicolumn{7}{c}{\textit{Consistency-based UQ: multinomial vs. beamsearch versions}} \\
\midrule
Dissimilarity & .437 & .229 & .424 & .333 & .446 & .272 \\
Dissimilarity + beamsearch & \underline{.455}$\uparrow$ & .266$\uparrow$ & \underline{.466}$\uparrow$ & .390$\uparrow$ & .440 & \textbf{.346}$\uparrow$ \\
\midrule
Eccentricity & .405 & .238 & .395 & .310 & .375 & .208 \\
Eccentricity + beamsearch & .444$\uparrow$ & \underline{.301}$\uparrow$ & .450$\uparrow$ & .348$\uparrow$ & .380$\uparrow$ & .308$\uparrow$ \\
\midrule
EigVecDissimilarity & .402 & .243 & .412 & .316 & .403 & .213 \\
EigVecDissimilarity + beamsearch & .446$\uparrow$ & \textbf{.316}$\uparrow$ & .457$\uparrow$ & .366$\uparrow$ & .415$\uparrow$ & .334$\uparrow$ \\
\midrule
CocoaMSP & .447 & .284 & .450 & .347 & \underline{.454} & .272 \\
CocoaMSP + beamsearch & \textbf{.471}$\uparrow$ & .290$\uparrow$ & \textbf{.478}$\uparrow$ & \textbf{.407}$\uparrow$ & \textbf{.459}$\uparrow$ & \underline{.345}$\uparrow$ \\
\midrule
CocoaPPL & .440 & .251 & .433 & .340 & .422 & .261 \\
CocoaPPL + beamsearch & .450$\uparrow$ & .273$\uparrow$ & .444$\uparrow$ & \underline{.395}$\uparrow$ & .410 & .318$\uparrow$ \\
\bottomrule
\end{tabular}
    }

  \label{tab:topbeam_scoring}
  \end{table}

\subsection{ROC-AUC and PR-AUC}
\label{appendix:rocauc_and_prauc}
  In the main text we report PRR. Tables~\ref{tab:abl_rocauc} and~\ref{tab:abl_prauc} complements these results with ROC-AUC and PR-AUC on \gemmab. We binarize by marking an answer as correct if its AlignScore to the gold answer exceeds \(0.5\), and incorrect otherwise (the positive class for PR-AUC is the incorrect label). The pattern mirrors PRR: beam-guided variants generally match or outperform multinomial sampling.

  \begin{table}[h]
    \caption{ROC-AUC$\uparrow$ for 6 datasets with \gemmab. For each dataset, the top-1 method is \textbf{bold} and the second-best is \underline{underlined}. Beam-guided and probability-weighted variants are marked with \(\uparrow\) when they improve over their multinomial-sampling baseline. The two rightmost columns report the mean ROC-AUC across datasets.}
    \centering
    \small
    \resizebox{\textwidth}{!}{
      \begin{tabular}{lccccccc}
\toprule
UQ Method & TriviaQA & \multirowcell{Web\\Questions} & CoQA & HotpotQA & \multirowcell{Common\\senceQA} & \multirowcell{ARC-\\Challenge} & Mean \\
\midrule
 \rowcolor{gray!20}
 \multicolumn{8}{c}{\textit{Baseline UQ methods}} \\
\midrule
Prob & .863 & .768 & .698 & .632 & .796 & .821 & .763 \\
MTE & .867 & .793 & .710 & .721 & .737 & .753 & .763 \\
Perplexity & .863 & .785 & .729 & .735 & .796 & .820 & .788 \\
CCP & .881 & .781 & .698 & .660 & .775 & .793 & .764 \\
SAR & .867 & .776 & .701 & .713 & .653 & .696 & .748 \\
P(True) & .642 & .473 & .524 & .513 & .571 & .545 & .545 \\
SemanticEntropy & .849 & .758 & .690 & .591 & .755 & .774 & .736 \\
LexicalSimilarity & .842 & .766 & .713 & .656 & .739 & .756 & .745 \\
EigValLaplacian & .867 & .766 & .701 & .633 & .739 & .775 & .747 \\
NumSemSets & .856 & .754 & .653 & .639 & .702 & .757 & .727 \\
\midrule
 \rowcolor{gray!20}
 \multicolumn{8}{c}{\textit{Consistency-based UQ: multinomial vs. beamsearch versions}} \\
\midrule
Dissimilarity & .916 & .836 & .822 & .809 & .817 & .818 & .836 \\
Dissimilarity + beamsearch & \textbf{.923}$\uparrow$ & \textbf{.852}$\uparrow$ & \textbf{.826}$\uparrow$ & \underline{.814}$\uparrow$ & \underline{.831}$\uparrow$ & .841$\uparrow$ & \textbf{.848}$\uparrow$ \\
\midrule
Eccentricity & .897 & .808 & .768 & .737 & .809 & .821 & .806 \\
Eccentricity + beamsearch & .911$\uparrow$ & .816$\uparrow$ & .790$\uparrow$ & .771$\uparrow$ & \textbf{.833}$\uparrow$ & \textbf{.859}$\uparrow$ & .830$\uparrow$ \\
\midrule
EigVecDissimilarity & .902 & .813 & .761 & .728 & .798 & .825 & .805 \\
EigVecDissimilarity + beamsearch & \underline{.920}$\uparrow$ & .827$\uparrow$ & .787$\uparrow$ & .763$\uparrow$ & .820$\uparrow$ & \underline{.856}$\uparrow$ & .829$\uparrow$ \\
\midrule
CocoaMSP & .904 & .823 & .791 & .726 & .826 & .839 & .818 \\
CocoaMSP + beamsearch & .910$\uparrow$ & .836$\uparrow$ & .811$\uparrow$ & .779$\uparrow$ & .827$\uparrow$ & .847$\uparrow$ & .835$\uparrow$ \\
\midrule
CocoaPPL & .907 & .832 & .810 & .799 & .825 & .837 & .835 \\
CocoaPPL + beamsearch & .912$\uparrow$ & \underline{.845}$\uparrow$ & \underline{.823}$\uparrow$ & \textbf{.828}$\uparrow$ & .825$\uparrow$ & .844$\uparrow$ & \underline{.846}$\uparrow$ \\
\bottomrule
\end{tabular}
    }
  \label{tab:abl_rocauc}
  \end{table}

  \begin{table}[h]
    \caption{PR-AUC$\uparrow$ for 6 datasets with \gemmab. For each dataset, the top-1 method is \textbf{bold} and the second-best is \underline{underlined}. Beam-guided and probability-weighted variants are marked with \(\uparrow\) when they improve over their multinomial-sampling baseline. The two rightmost columns report the mean PR-AUC across datasets.}
    \centering
    \small
    \resizebox{\textwidth}{!}{
      \begin{tabular}{lccccccc}
\toprule
UQ Method & TriviaQA & \multirowcell{Web\\Questions} & CoQA & HotpotQA & \multirowcell{Common\\senceQA} & \multirowcell{ARC-\\Challenge} & Mean \\
\midrule
 \rowcolor{gray!20}
 \multicolumn{8}{c}{\textit{Baseline UQ methods}} \\
\midrule
Prob & .855 & .838 & .477 & .678 & .623 & .628 & .683 \\
MTE & .875 & .874 & .545 & .799 & .558 & .540 & .699 \\
Perplexity & .860 & .861 & .539 & .814 & .629 & .632 & .722 \\
CCP & .866 & .853 & .475 & .715 & .676 & .641 & .704 \\
SAR & .865 & .861 & .484 & .753 & .437 & .422 & .646 \\
P(True) & .657 & .662 & .326 & .634 & .410 & .355 & .507 \\
SemanticEntropy & .838 & .823 & .456 & .649 & .572 & .511 & .642 \\
LexicalSimilarity & .833 & .848 & .514 & .711 & .545 & .509 & .660 \\
EigValLaplacian & .865 & .855 & .481 & .682 & .565 & .532 & .663 \\
NumSemSets & .841 & .825 & .427 & .685 & .508 & .497 & .631 \\
\midrule
 \rowcolor{gray!20}
 \multicolumn{8}{c}{\textit{Consistency-based UQ: multinomial vs. beamsearch versions}} \\
\midrule
Dissimilarity & .911 & .904 & \textbf{.715} & .838 & .722 & .648 & .789 \\
Dissimilarity + beamsearch & \textbf{.919}$\uparrow$ & \textbf{.915}$\uparrow$ & .660 & .822 & \textbf{.754}$\uparrow$ & .693$\uparrow$ & \underline{.794}$\uparrow$ \\
\midrule
Eccentricity & .888 & .887 & .561 & .758 & .685 & .625 & .734 \\
Eccentricity + beamsearch & .906$\uparrow$ & .884 & .576$\uparrow$ & .789$\uparrow$ & \underline{.744}$\uparrow$ & \textbf{.717}$\uparrow$ & .769$\uparrow$ \\
\midrule
EigVecDissimilarity & .902 & .889 & .573 & .766 & .677 & .651 & .743 \\
EigVecDissimilarity + beamsearch & \underline{.916}$\uparrow$ & .900$\uparrow$ & .588$\uparrow$ & .784$\uparrow$ & .717$\uparrow$ & .689$\uparrow$ & .766$\uparrow$ \\
\midrule
CocoaMSP & .897 & .894 & .605 & .761 & .711 & .680 & .758 \\
CocoaMSP + beamsearch & .907$\uparrow$ & .904$\uparrow$ & .632$\uparrow$ & .801$\uparrow$ & .715$\uparrow$ & .691$\uparrow$ & .775$\uparrow$ \\
\midrule
CocoaPPL & .902 & .902 & .672 & \underline{.861} & .712 & .686 & .789 \\
CocoaPPL + beamsearch & .909$\uparrow$ & \underline{.910}$\uparrow$ & \underline{.690}$\uparrow$ & \textbf{.881}$\uparrow$ & .718$\uparrow$ & \underline{.695}$\uparrow$ & \textbf{.801}$\uparrow$ \\
\bottomrule
\end{tabular}
    }
  \label{tab:abl_prauc}
  \end{table}

\subsection{Detailed Results for Each Dataset}
\label{appendix:other_llms}
  Complementing the main-table results in Table~\ref{tab:qa_mean_prr}, Tables~\ref{tab:qa_gemma4b}--\ref{tab:qa_qwen8bit} report PRR for six datasets separately for \gemmab, \gemmait, \llamab, \llamait, \qwenb, and \qwenit.

  \begin{table}[h]
    \caption{PRR ($\uparrow$ is better) for 6 datasets with \gemmab. For each dataset, the top-1 method is \textbf{bold} and the second-best is \underline{underlined}. Beam-guided and probability-weighted variants are marked with $\uparrow$ when they improve over their multinomial-sampling baseline.}
    \centering
    \small
    \resizebox{\textwidth}{!}{
      \begin{tabular}{lcccccc}
\toprule
Method & TriviaQA & \multirowcell{Web\\Questions} & CoQA & HotpotQA & \multirowcell{Common\\senceQA} & \multirowcell{ARC-\\Challenge} \\
\midrule
 \rowcolor{gray!20}
 \multicolumn{7}{c}{\textit{Baseline UQ methods}} \\
\midrule
Prob & .659 \tiny{$\pm$ 0.018} & .521 \tiny{$\pm$ 0.031} & .312 \tiny{$\pm$ 0.024} & .274 \tiny{$\pm$ 0.014} & .511 \tiny{$\pm$ 0.025} & .548 \tiny{$\pm$ 0.077} \\
MTE & .670 \tiny{$\pm$ 0.013} & .583 \tiny{$\pm$ 0.029} & .363 \tiny{$\pm$ 0.02} & .494 \tiny{$\pm$ 0.034} & .364 \tiny{$\pm$ 0.031} & .381 \tiny{$\pm$ 0.052} \\
Perplexity & .647 \tiny{$\pm$ 0.024} & .553 \tiny{$\pm$ 0.022} & .369 \tiny{$\pm$ 0.02} & .527 \tiny{$\pm$ 0.023} & .503 \tiny{$\pm$ 0.022} & .547 \tiny{$\pm$ 0.062} \\
CCP & .686 \tiny{$\pm$ 0.021} & .569 \tiny{$\pm$ 0.031} & .326 \tiny{$\pm$ 0.022} & .337 \tiny{$\pm$ 0.025} & .506 \tiny{$\pm$ 0.034} & .527 \tiny{$\pm$ 0.062} \\
SAR & .656 \tiny{$\pm$ 0.02} & .571 \tiny{$\pm$ 0.028} & .347 \tiny{$\pm$ 0.023} & .296 \tiny{$\pm$ 0.018} & .183 \tiny{$\pm$ 0.037} & .264 \tiny{$\pm$ 0.055} \\
P(True) & .272 \tiny{$\pm$ 0.026} & -.004 \tiny{$\pm$ 0.034} & .031 \tiny{$\pm$ 0.026} & .075 \tiny{$\pm$ 0.025} & .090 \tiny{$\pm$ 0.028} & .090 \tiny{$\pm$ 0.048} \\
SemanticEntropy & .622 \tiny{$\pm$ 0.021} & .505 \tiny{$\pm$ 0.022} & .301 \tiny{$\pm$ 0.019} & .140 \tiny{$\pm$ 0.022} & .407 \tiny{$\pm$ 0.028} & .431 \tiny{$\pm$ 0.051} \\
Lexical Similarity & .602 \tiny{$\pm$ 0.017} & .540 \tiny{$\pm$ 0.032} & .349 \tiny{$\pm$ 0.025} & .286 \tiny{$\pm$ 0.016} & .386 \tiny{$\pm$ 0.032} & .392 \tiny{$\pm$ 0.054} \\
EigValLaplacian & .666 \tiny{$\pm$ 0.014} & .555 \tiny{$\pm$ 0.028} & .320 \tiny{$\pm$ 0.036} & .246 \tiny{$\pm$ 0.024} & .386 \tiny{$\pm$ 0.027} & .452 \tiny{$\pm$ 0.046} \\
NumSemSets & .656 \tiny{$\pm$ 0.017} & .538 \tiny{$\pm$ 0.028} & .257 \tiny{$\pm$ 0.027} & .268 \tiny{$\pm$ 0.019} & .338 \tiny{$\pm$ 0.03} & .454 \tiny{$\pm$ 0.042} \\
\midrule
 \rowcolor{gray!20}
 \multicolumn{7}{c}{\textit{Consistency-based UQ: multinomial vs. beamsearch versions}} \\
\midrule
Dissimilarity & \underline{.755} \tiny{$\pm$ 0.019} & \underline{.715} \tiny{$\pm$ 0.03} & \underline{.578} \tiny{$\pm$ 0.022} & \underline{.626} \tiny{$\pm$ 0.016} & .561 \tiny{$\pm$ 0.04} & .545 \tiny{$\pm$ 0.062} \\
Dissimilarity + beamsearch & \textbf{.766} $\uparrow$\tiny{$\pm$ 0.023} & \textbf{.722} $\uparrow$\tiny{$\pm$ 0.028} & \textbf{.600} $\uparrow$\tiny{$\pm$ 0.016} & .611 \tiny{$\pm$ 0.021} & \textbf{.595} $\uparrow$\tiny{$\pm$ 0.028} & .604 $\uparrow$\tiny{$\pm$ 0.052} \\
\midrule
Eccentricity & .714 \tiny{$\pm$ 0.012} & .653 \tiny{$\pm$ 0.029} & .459 \tiny{$\pm$ 0.02} & .453 \tiny{$\pm$ 0.026} & .549 \tiny{$\pm$ 0.034} & .549 \tiny{$\pm$ 0.054} \\
Eccentricity + beamsearch & .739 $\uparrow$\tiny{$\pm$ 0.019} & .633 \tiny{$\pm$ 0.035} & .505 $\uparrow$\tiny{$\pm$ 0.025} & .514 $\uparrow$\tiny{$\pm$ 0.027} & \underline{.590} $\uparrow$\tiny{$\pm$ 0.024} & \textbf{.636} $\uparrow$\tiny{$\pm$ 0.066} \\
\midrule
EigVecDissimilarity & .738 \tiny{$\pm$ 0.021} & .661 \tiny{$\pm$ 0.027} & .443 \tiny{$\pm$ 0.031} & .448 \tiny{$\pm$ 0.02} & .512 \tiny{$\pm$ 0.032} & .562 \tiny{$\pm$ 0.035} \\
EigVecDissimilarity + beamsearch & .753 $\uparrow$\tiny{$\pm$ 0.028} & .668 $\uparrow$\tiny{$\pm$ 0.032} & .497 $\uparrow$\tiny{$\pm$ 0.021} & .487 $\uparrow$\tiny{$\pm$ 0.016} & .562 $\uparrow$\tiny{$\pm$ 0.028} & \underline{.621} $\uparrow$\tiny{$\pm$ 0.06} \\
\midrule
CocoaMSP & .738 \tiny{$\pm$ 0.023} & .666 \tiny{$\pm$ 0.028} & .509 \tiny{$\pm$ 0.021} & .430 \tiny{$\pm$ 0.028} & .583 \tiny{$\pm$ 0.03} & .595 \tiny{$\pm$ 0.052} \\
CocoaMSP + beamsearch & .747 $\uparrow$\tiny{$\pm$ 0.02} & .679 $\uparrow$\tiny{$\pm$ 0.02} & .548 $\uparrow$\tiny{$\pm$ 0.02} & .523 $\uparrow$\tiny{$\pm$ 0.027} & .586 $\uparrow$\tiny{$\pm$ 0.029} & .606 $\uparrow$\tiny{$\pm$ 0.072} \\
\midrule
CocoaPPL & .739 \tiny{$\pm$ 0.015} & .678 \tiny{$\pm$ 0.025} & .548 \tiny{$\pm$ 0.019} & .625 \tiny{$\pm$ 0.023} & .580 \tiny{$\pm$ 0.031} & .595 \tiny{$\pm$ 0.039} \\
CocoaPPL + beamsearch & .748 $\uparrow$\tiny{$\pm$ 0.024} & .694 $\uparrow$\tiny{$\pm$ 0.024} & .577 $\uparrow$\tiny{$\pm$ 0.024} & \textbf{.681} $\uparrow$\tiny{$\pm$ 0.019} & .582 $\uparrow$\tiny{$\pm$ 0.035} & .610 $\uparrow$\tiny{$\pm$ 0.048} \\
\bottomrule
\end{tabular}
    }
    \label{tab:qa_gemma4b}
  \end{table}

  \begin{table}[h]
    \caption{PRR ($\uparrow$ is better) for 6 datasets with \gemmait. For each dataset, the top-1 method is \textbf{bold} and the second-best is \underline{underlined}. Beam-guided and probability-weighted variants are marked with $\uparrow$ when they improve over their multinomial-sampling baseline.}
    \centering
    \small
    \resizebox{\textwidth}{!}{
      \begin{tabular}{lcccccc}
\toprule
UQ Method & TriviaQA & \multirowcell{Web\\Questions} & CoQA & HotpotQA & \multirowcell{Common\\senceQA} & \multirowcell{ARC-\\Challenge} \\
\midrule
 \rowcolor{gray!20}
 \multicolumn{7}{c}{\textit{Baseline UQ methods}} \\
\midrule
Prob & .442 \tiny{$\pm$ .018} & .425 \tiny{$\pm$ .031} & .162 \tiny{$\pm$ .024} & .220 \tiny{$\pm$ .014} & .254 \tiny{$\pm$ .025} & .252 \tiny{$\pm$ .077} \\
MTE & .534 \tiny{$\pm$ .013} & .465 \tiny{$\pm$ .029} & .161 \tiny{$\pm$ .02} & .232 \tiny{$\pm$ .034} & .253 \tiny{$\pm$ .031} & .256 \tiny{$\pm$ .052} \\
Perplexity & .422 \tiny{$\pm$ .024} & .419 \tiny{$\pm$ .022} & .157 \tiny{$\pm$ .02} & .223 \tiny{$\pm$ .023} & .252 \tiny{$\pm$ .022} & .256 \tiny{$\pm$ .062} \\
CCP & .533 \tiny{$\pm$ .021} & \textbf{.478} \tiny{$\pm$ .031} & .117 \tiny{$\pm$ .022} & .303 \tiny{$\pm$ .025} & .264 \tiny{$\pm$ .034} & \textbf{.290} \tiny{$\pm$ .062} \\
SAR & .533 \tiny{$\pm$ .02} & .426 \tiny{$\pm$ .028} & .176 \tiny{$\pm$ .023} & .214 \tiny{$\pm$ .018} & .033 \tiny{$\pm$ .037} & .050 \tiny{$\pm$ .055} \\
P(True) & -.076 \tiny{$\pm$ .026} & -.155 \tiny{$\pm$ .034} & -.161 \tiny{$\pm$ .026} & -.090 \tiny{$\pm$ .025} & -.046 \tiny{$\pm$ .028} & -.047 \tiny{$\pm$ .048} \\
SemanticEntropy & .449 \tiny{$\pm$ .021} & .415 \tiny{$\pm$ .022} & .166 \tiny{$\pm$ .019} & .223 \tiny{$\pm$ .022} & .254 \tiny{$\pm$ .028} & .252 \tiny{$\pm$ .051} \\
Lexical Similarity & .527 \tiny{$\pm$ .017} & .427 \tiny{$\pm$ .032} & .176 \tiny{$\pm$ .025} & .127 \tiny{$\pm$ .016} & .052 \tiny{$\pm$ .032} & .172 \tiny{$\pm$ .054} \\
EigValLaplacian & \textbf{.578} \tiny{$\pm$ .014} & .472 \tiny{$\pm$ .028} & .190 \tiny{$\pm$ .036} & .134 \tiny{$\pm$ .024} & .014 \tiny{$\pm$ .027} & .010 \tiny{$\pm$ .046} \\
NumSemSets & .556 \tiny{$\pm$ .017} & .442 \tiny{$\pm$ .028} & .123 \tiny{$\pm$ .027} & .106 \tiny{$\pm$ .019} & .046 \tiny{$\pm$ .03} & .153 \tiny{$\pm$ .042} \\
\midrule
 \rowcolor{gray!20}
 \multicolumn{7}{c}{\textit{Consistency-based UQ: multinomial vs. beamsearch versions}} \\
\midrule
Dissimilarity & .549 \tiny{$\pm$ .019} & .415 \tiny{$\pm$ .03} & .111 \tiny{$\pm$ .022} & .068 \tiny{$\pm$ .016} & .024 \tiny{$\pm$ .04} & .070 \tiny{$\pm$ .062} \\
Dissimilarity + beamsearch & .413 \tiny{$\pm$ .023} & .321 \tiny{$\pm$ .028} & .204 $\uparrow$\tiny{$\pm$ .016} & .273 $\uparrow$\tiny{$\pm$ .021} & .218 $\uparrow$\tiny{$\pm$ .028} & .085 $\uparrow$\tiny{$\pm$ .052} \\
\midrule
Eccentricity & .540 \tiny{$\pm$ .012} & .429 \tiny{$\pm$ .029} & .167 \tiny{$\pm$ .02} & .175 \tiny{$\pm$ .026} & -.020 \tiny{$\pm$ .034} & .094 \tiny{$\pm$ .054} \\
Eccentricity + beamsearch & .441 \tiny{$\pm$ .019} & .367 \tiny{$\pm$ .035} & .235 $\uparrow$\tiny{$\pm$ .025} & \textbf{.314} $\uparrow$\tiny{$\pm$ .027} & .246 $\uparrow$\tiny{$\pm$ .024} & .108 $\uparrow$\tiny{$\pm$ .066} \\
\midrule
EigVecDissimilarity & \underline{.561} \tiny{$\pm$ .021} & .437 \tiny{$\pm$ .027} & .169 \tiny{$\pm$ .031} & .173 \tiny{$\pm$ .02} & -.017 \tiny{$\pm$ .032} & .095 \tiny{$\pm$ .035} \\
EigVecDissimilarity + beamsearch & .478 \tiny{$\pm$ .028} & .416 \tiny{$\pm$ .032} & \textbf{.240} $\uparrow$\tiny{$\pm$ .021} & \underline{.308} $\uparrow$\tiny{$\pm$ .016} & .253 $\uparrow$\tiny{$\pm$ .028} & .113 $\uparrow$\tiny{$\pm$ .06} \\
\midrule
CocoaMSP & .531 \tiny{$\pm$ .023} & .456 \tiny{$\pm$ .028} & .183 \tiny{$\pm$ .021} & .198 \tiny{$\pm$ .028} & .252 \tiny{$\pm$ .03} & .266 \tiny{$\pm$ .052} \\
CocoaMSP + beamsearch & .535 $\uparrow$\tiny{$\pm$ .02} & \underline{.473} $\uparrow$\tiny{$\pm$ .02} & \underline{.237} $\uparrow$\tiny{$\pm$ .02} & .287 $\uparrow$\tiny{$\pm$ .027} & \textbf{.282} $\uparrow$\tiny{$\pm$ .029} & .258 \tiny{$\pm$ .072} \\
\midrule
CocoaPPL & .523 \tiny{$\pm$ .015} & .454 \tiny{$\pm$ .025} & .174 \tiny{$\pm$ .019} & .201 \tiny{$\pm$ .023} & .247 \tiny{$\pm$ .031} & \underline{.271} \tiny{$\pm$ .039} \\
CocoaPPL + beamsearch & .522 \tiny{$\pm$ .024} & .467 $\uparrow$\tiny{$\pm$ .024} & .222 $\uparrow$\tiny{$\pm$ .024} & .285 $\uparrow$\tiny{$\pm$ .019} & \underline{.277} $\uparrow$\tiny{$\pm$ .035} & .264 \tiny{$\pm$ .048} \\
\bottomrule
\end{tabular}
    }
    \label{tab:qa_gemma4bit}
  \end{table}

  \begin{table}[h]
    \caption{PRR ($\uparrow$ is better) for 6 datasets with \llamab. For each dataset, the top-1 method is \textbf{bold} and the second-best is \underline{underlined}. Beam-guided and probability-weighted variants are marked with $\uparrow$ when they improve over their multinomial-sampling baseline.}
    \centering
    \small
    \resizebox{\textwidth}{!}{
      \begin{tabular}{lcccccc}
\toprule
UQ Method & TriviaQA & \multirowcell{Web\\Questions} & CoQA & HotpotQA & \multirowcell{Common\\senceQA} & \multirowcell{ARC-\\Challenge} \\
\midrule
 \rowcolor{gray!20}
 \multicolumn{7}{c}{\textit{Baseline UQ methods}} \\
\midrule
Prob & .517 \tiny{$\pm$ .019} & .414 \tiny{$\pm$ .029} & .310 \tiny{$\pm$ .022} & .213 \tiny{$\pm$ .024} & .504 \tiny{$\pm$ .029} & .505 \tiny{$\pm$ .043} \\
MTE & .544 \tiny{$\pm$ .018} & .420 \tiny{$\pm$ .015} & .286 \tiny{$\pm$ .022} & .327 \tiny{$\pm$ .02} & .448 \tiny{$\pm$ .029} & .511 \tiny{$\pm$ .055} \\
Perplexity & .507 \tiny{$\pm$ .015} & .441 \tiny{$\pm$ .027} & .316 \tiny{$\pm$ .031} & .375 \tiny{$\pm$ .018} & .501 \tiny{$\pm$ .027} & .570 \tiny{$\pm$ .047} \\
CCP & .575 \tiny{$\pm$ .016} & .420 \tiny{$\pm$ .026} & .276 \tiny{$\pm$ .024} & .247 \tiny{$\pm$ .029} & .442 \tiny{$\pm$ .023} & .446 \tiny{$\pm$ .031} \\
SAR & .548 \tiny{$\pm$ .017} & .452 \tiny{$\pm$ .028} & .331 \tiny{$\pm$ .03} & .263 \tiny{$\pm$ .031} & .189 \tiny{$\pm$ .021} & .330 \tiny{$\pm$ .044} \\
P(True) & -.055 \tiny{$\pm$ .021} & .059 \tiny{$\pm$ .023} & -.020 \tiny{$\pm$ .018} & -.223 \tiny{$\pm$ .026} & .034 \tiny{$\pm$ .024} & .292 \tiny{$\pm$ .044} \\
SemanticEntropy & .538 \tiny{$\pm$ .019} & .409 \tiny{$\pm$ .023} & .330 \tiny{$\pm$ .021} & .199 \tiny{$\pm$ .024} & .492 \tiny{$\pm$ .023} & .514 \tiny{$\pm$ .05} \\
Lexical Similarity & .467 \tiny{$\pm$ .018} & .396 \tiny{$\pm$ .03} & .366 \tiny{$\pm$ .024} & .289 \tiny{$\pm$ .026} & .437 \tiny{$\pm$ .028} & .511 \tiny{$\pm$ .041} \\
EigValLaplacian & .569 \tiny{$\pm$ .019} & .418 \tiny{$\pm$ .022} & .377 \tiny{$\pm$ .023} & .247 \tiny{$\pm$ .025} & .449 \tiny{$\pm$ .035} & .499 \tiny{$\pm$ .047} \\
NumSemSets & .550 \tiny{$\pm$ .014} & .409 \tiny{$\pm$ .033} & .319 \tiny{$\pm$ .019} & .241 \tiny{$\pm$ .028} & .378 \tiny{$\pm$ .025} & .477 \tiny{$\pm$ .044} \\
\midrule
 \rowcolor{gray!20}
 \multicolumn{7}{c}{\textit{Consistency-based UQ: multinomial vs. beamsearch versions}} \\
\midrule
Dissimilarity & .576 \tiny{$\pm$ .02} & .445 \tiny{$\pm$ .024} & \underline{.473} \tiny{$\pm$ .023} & .446 \tiny{$\pm$ .02} & .449 \tiny{$\pm$ .028} & .640 \tiny{$\pm$ .056} \\
Dissimilarity + beamsearch & \textbf{.654} $\uparrow$\tiny{$\pm$ .017} & \textbf{.504} $\uparrow$\tiny{$\pm$ .023} & \textbf{.485} $\uparrow$\tiny{$\pm$ .019} & .424 \tiny{$\pm$ .024} & .510 $\uparrow$\tiny{$\pm$ .023} & \textbf{.683} $\uparrow$\tiny{$\pm$ .044} \\
\midrule
Eccentricity & .555 \tiny{$\pm$ .016} & .404 \tiny{$\pm$ .025} & .405 \tiny{$\pm$ .023} & .297 \tiny{$\pm$ .021} & .464 \tiny{$\pm$ .028} & .591 \tiny{$\pm$ .038} \\
Eccentricity + beamsearch & .613 $\uparrow$\tiny{$\pm$ .021} & .458 $\uparrow$\tiny{$\pm$ .019} & .429 $\uparrow$\tiny{$\pm$ .017} & .361 $\uparrow$\tiny{$\pm$ .023} & .512 $\uparrow$\tiny{$\pm$ .025} & .657 $\uparrow$\tiny{$\pm$ .031} \\
\midrule
EigVecDissimilarity & .570 \tiny{$\pm$ .015} & .452 \tiny{$\pm$ .022} & .409 \tiny{$\pm$ .019} & .289 \tiny{$\pm$ .02} & .469 \tiny{$\pm$ .04} & .587 \tiny{$\pm$ .038} \\
EigVecDissimilarity + beamsearch & .630 $\uparrow$\tiny{$\pm$ .019} & .492 $\uparrow$\tiny{$\pm$ .022} & .427 $\uparrow$\tiny{$\pm$ .019} & .357 $\uparrow$\tiny{$\pm$ .02} & .506 $\uparrow$\tiny{$\pm$ .035} & .650 $\uparrow$\tiny{$\pm$ .032} \\
\midrule
CocoaMSP & .595 \tiny{$\pm$ .013} & .458 \tiny{$\pm$ .021} & .463 \tiny{$\pm$ .023} & .366 \tiny{$\pm$ .021} & .510 \tiny{$\pm$ .028} & .641 \tiny{$\pm$ .038} \\
CocoaMSP + beamsearch & \underline{.631} $\uparrow$\tiny{$\pm$ .019} & .487 $\uparrow$\tiny{$\pm$ .023} & .465 $\uparrow$\tiny{$\pm$ .027} & .372 $\uparrow$\tiny{$\pm$ .027} & \textbf{.532} $\uparrow$\tiny{$\pm$ .022} & .639 \tiny{$\pm$ .041} \\
\midrule
CocoaPPL & .587 \tiny{$\pm$ .017} & .464 \tiny{$\pm$ .024} & .464 \tiny{$\pm$ .023} & \textbf{.465} \tiny{$\pm$ .02} & .501 \tiny{$\pm$ .031} & .660 \tiny{$\pm$ .034} \\
CocoaPPL + beamsearch & .616 $\uparrow$\tiny{$\pm$ .016} & \underline{.498} $\uparrow$\tiny{$\pm$ .029} & .459 \tiny{$\pm$ .024} & \underline{.456} \tiny{$\pm$ .018} & \underline{.525} $\uparrow$\tiny{$\pm$ .028} & \underline{.661} $\uparrow$\tiny{$\pm$ .046} \\
\bottomrule
\end{tabular}
    }
    \label{tab:qa_llama8b}
  \end{table}

  \begin{table}[h]
    \caption{PRR ($\uparrow$ is better) for 6 datasets with \llamait. For each dataset, the top-1 method is \textbf{bold} and the second-best is \underline{underlined}. Beam-guided and probability-weighted variants are marked with $\uparrow$ when they improve over their multinomial-sampling baseline.}
    \centering
    \small
    \resizebox{\textwidth}{!}{
      \begin{tabular}{lcccccc}
\toprule
UQ Method & TriviaQA & \multirowcell{Web\\Questions} & CoQA & HotpotQA & \multirowcell{Common\\senceQA} & \multirowcell{ARC-\\Challenge} \\
\midrule
 \rowcolor{gray!20}
 \multicolumn{7}{c}{\textit{Baseline UQ methods}} \\
\midrule
Prob & .524 \tiny{$\pm$ .023} & .357 \tiny{$\pm$ .036} & .327 \tiny{$\pm$ .021} & .213 \tiny{$\pm$ .022} & .283 \tiny{$\pm$ .026} & .363 \tiny{$\pm$ .044} \\
MTE & .604 \tiny{$\pm$ .015} & .424 \tiny{$\pm$ .028} & .307 \tiny{$\pm$ .02} & .253 \tiny{$\pm$ .031} & .260 \tiny{$\pm$ .027} & .339 \tiny{$\pm$ .055} \\
Perplexity & .498 \tiny{$\pm$ .018} & .367 \tiny{$\pm$ .025} & .262 \tiny{$\pm$ .025} & .221 \tiny{$\pm$ .03} & .255 \tiny{$\pm$ .028} & .332 \tiny{$\pm$ .053} \\
CCP & .576 \tiny{$\pm$ .023} & .406 \tiny{$\pm$ .028} & .291 \tiny{$\pm$ .018} & .265 \tiny{$\pm$ .022} & .248 \tiny{$\pm$ .034} & .402 \tiny{$\pm$ .048} \\
SAR & .599 \tiny{$\pm$ .021} & .420 \tiny{$\pm$ .029} & .338 \tiny{$\pm$ .024} & .236 \tiny{$\pm$ .02} & .301 \tiny{$\pm$ .025} & .418 \tiny{$\pm$ .04} \\
P(True) & .236 \tiny{$\pm$ .023} & .012 \tiny{$\pm$ .031} & .018 \tiny{$\pm$ .035} & .045 \tiny{$\pm$ .024} & -.011 \tiny{$\pm$ .024} & .135 \tiny{$\pm$ .051} \\
SemanticEntropy & .591 \tiny{$\pm$ .016} & .381 \tiny{$\pm$ .027} & .335 \tiny{$\pm$ .032} & .231 \tiny{$\pm$ .029} & .301 \tiny{$\pm$ .038} & .418 \tiny{$\pm$ .061} \\
Lexical Similarity & .566 \tiny{$\pm$ .023} & .395 \tiny{$\pm$ .029} & .347 \tiny{$\pm$ .024} & .232 \tiny{$\pm$ .03} & .275 \tiny{$\pm$ .032} & .380 \tiny{$\pm$ .045} \\
EigValLaplacian & .615 \tiny{$\pm$ .021} & .389 \tiny{$\pm$ .026} & .355 \tiny{$\pm$ .029} & .238 \tiny{$\pm$ .023} & .252 \tiny{$\pm$ .029} & .377 \tiny{$\pm$ .051} \\
NumSemSets & .569 \tiny{$\pm$ .021} & .363 \tiny{$\pm$ .031} & .228 \tiny{$\pm$ .03} & .180 \tiny{$\pm$ .023} & .208 \tiny{$\pm$ .035} & .368 \tiny{$\pm$ .051} \\
\midrule
 \rowcolor{gray!20}
 \multicolumn{7}{c}{\textit{Consistency-based UQ: multinomial vs. beamsearch versions}} \\
\midrule
Dissimilarity & .616 \tiny{$\pm$ .016} & .382 \tiny{$\pm$ .031} & .349 \tiny{$\pm$ .018} & .270 \tiny{$\pm$ .021} & .277 \tiny{$\pm$ .037} & .378 \tiny{$\pm$ .061} \\
Dissimilarity + beamsearch & \underline{.662} $\uparrow$\tiny{$\pm$ .015} & .411 $\uparrow$\tiny{$\pm$ .029} & .358 $\uparrow$\tiny{$\pm$ .029} & \textbf{.349} $\uparrow$\tiny{$\pm$ .019} & .288 $\uparrow$\tiny{$\pm$ .032} & .434 $\uparrow$\tiny{$\pm$ .054} \\
\midrule
Eccentricity & .598 \tiny{$\pm$ .021} & .379 \tiny{$\pm$ .032} & .319 \tiny{$\pm$ .016} & .248 \tiny{$\pm$ .031} & .273 \tiny{$\pm$ .035} & .389 \tiny{$\pm$ .058} \\
Eccentricity + beamsearch & .620 $\uparrow$\tiny{$\pm$ .016} & .396 $\uparrow$\tiny{$\pm$ .027} & .330 $\uparrow$\tiny{$\pm$ .021} & .281 $\uparrow$\tiny{$\pm$ .021} & .306 $\uparrow$\tiny{$\pm$ .031} & \underline{.451} $\uparrow$\tiny{$\pm$ .047} \\
\midrule
EigVecDissimilarity & .611 \tiny{$\pm$ .019} & .378 \tiny{$\pm$ .033} & .325 \tiny{$\pm$ .025} & .249 \tiny{$\pm$ .029} & .264 \tiny{$\pm$ .037} & .390 \tiny{$\pm$ .061} \\
EigVecDissimilarity + beamsearch & .640 $\uparrow$\tiny{$\pm$ .017} & .425 $\uparrow$\tiny{$\pm$ .028} & .347 $\uparrow$\tiny{$\pm$ .027} & .291 $\uparrow$\tiny{$\pm$ .022} & \textbf{.318} $\uparrow$\tiny{$\pm$ .034} & \textbf{.461} $\uparrow$\tiny{$\pm$ .046} \\
\midrule
CocoaMSP & .629 \tiny{$\pm$ .018} & .409 \tiny{$\pm$ .023} & \underline{.366} \tiny{$\pm$ .03} & .278 \tiny{$\pm$ .02} & \underline{.314} \tiny{$\pm$ .029} & .426 \tiny{$\pm$ .051} \\
CocoaMSP + beamsearch & \textbf{.665} $\uparrow$\tiny{$\pm$ .016} & \textbf{.428} $\uparrow$\tiny{$\pm$ .029} & \textbf{.378} $\uparrow$\tiny{$\pm$ .017} & \underline{.344} $\uparrow$\tiny{$\pm$ .019} & .302 \tiny{$\pm$ .036} & .439 $\uparrow$\tiny{$\pm$ .041} \\
\midrule
CocoaPPL & .626 \tiny{$\pm$ .022} & .410 \tiny{$\pm$ .03} & .354 \tiny{$\pm$ .024} & .278 \tiny{$\pm$ .024} & .299 \tiny{$\pm$ .038} & .413 \tiny{$\pm$ .056} \\
CocoaPPL + beamsearch & .653 $\uparrow$\tiny{$\pm$ .018} & \underline{.427} $\uparrow$\tiny{$\pm$ .032} & .356 $\uparrow$\tiny{$\pm$ .021} & .334 $\uparrow$\tiny{$\pm$ .018} & .285 \tiny{$\pm$ .04} & .419 $\uparrow$\tiny{$\pm$ .056} \\
\bottomrule
\end{tabular}
    }
    \label{tab:qa_llama8bit}
  \end{table}

  \begin{table}[h]
    \caption{PRR ($\uparrow$ is better) for 6 datasets with \qwenb. For each dataset, the top-1 method is \textbf{bold} and the second-best is \underline{underlined}. Beam-guided and probability-weighted variants are marked with $\uparrow$ when they improve over their multinomial-sampling baseline.}
    \centering
    \small
    \resizebox{\textwidth}{!}{
      \begin{tabular}{lcccccc}
\toprule
UQ Method & TriviaQA & \multirowcell{Web\\Questions} & CoQA & HotpotQA & \multirowcell{Common\\senceQA} & \multirowcell{ARC-\\Challenge} \\
\midrule
 \rowcolor{gray!20}
 \multicolumn{7}{c}{\textit{Baseline UQ methods}} \\
\midrule
Prob & .617 \tiny{$\pm$ .017} & .449 \tiny{$\pm$ .025} & .267 \tiny{$\pm$ .025} & .111 \tiny{$\pm$ .033} & .337 \tiny{$\pm$ .039} & .475 \tiny{$\pm$ .085} \\
MTE & .602 \tiny{$\pm$ .022} & .409 \tiny{$\pm$ .027} & .267 \tiny{$\pm$ .023} & .279 \tiny{$\pm$ .023} & \textbf{.443} \tiny{$\pm$ .044} & .444 \tiny{$\pm$ .077} \\
Perplexity & .597 \tiny{$\pm$ .018} & .426 \tiny{$\pm$ .028} & .278 \tiny{$\pm$ .023} & .256 \tiny{$\pm$ .026} & .294 \tiny{$\pm$ .045} & .381 \tiny{$\pm$ .065} \\
CCP & .640 \tiny{$\pm$ .018} & .406 \tiny{$\pm$ .028} & .213 \tiny{$\pm$ .028} & .153 \tiny{$\pm$ .025} & .296 \tiny{$\pm$ .048} & .421 \tiny{$\pm$ .09} \\
SAR & .617 \tiny{$\pm$ .023} & .457 \tiny{$\pm$ .023} & .323 \tiny{$\pm$ .03} & .243 \tiny{$\pm$ .028} & .220 \tiny{$\pm$ .042} & .317 \tiny{$\pm$ .066} \\
P(True) & .322 \tiny{$\pm$ .021} & .282 \tiny{$\pm$ .025} & .005 \tiny{$\pm$ .031} & .168 \tiny{$\pm$ .024} & -.043 \tiny{$\pm$ .045} & -.074 \tiny{$\pm$ .069} \\
Semantic Entropy & .549 \tiny{$\pm$ .018} & .411 \tiny{$\pm$ .02} & .247 \tiny{$\pm$ .025} & .173 \tiny{$\pm$ .023} & .230 \tiny{$\pm$ .026} & .305 \tiny{$\pm$ .058} \\
Lexical Similarity & .595 \tiny{$\pm$ .023} & .430 \tiny{$\pm$ .024} & .338 \tiny{$\pm$ .019} & .310 \tiny{$\pm$ .025} & .367 \tiny{$\pm$ .042} & .508 \tiny{$\pm$ .076} \\
EigValLaplacian & .602 \tiny{$\pm$ .015} & .423 \tiny{$\pm$ .027} & .301 \tiny{$\pm$ .027} & .284 \tiny{$\pm$ .028} & .349 \tiny{$\pm$ .032} & .475 \tiny{$\pm$ .081} \\
NumSemSets & .593 \tiny{$\pm$ .016} & .403 \tiny{$\pm$ .029} & .268 \tiny{$\pm$ .024} & .250 \tiny{$\pm$ .023} & .311 \tiny{$\pm$ .039} & .367 \tiny{$\pm$ .069} \\
\midrule
 \rowcolor{gray!20}
 \multicolumn{7}{c}{\textit{Consistency-based UQ: multinomial vs. beamsearch versions}} \\
\midrule
Dissimilarity & .668 \tiny{$\pm$ .014} & .462 \tiny{$\pm$ .024} & \underline{.406} \tiny{$\pm$ .023} & \textbf{.531} \tiny{$\pm$ .017} & .315 \tiny{$\pm$ .038} & .476 \tiny{$\pm$ .086} \\
Dissimilarity + beamsearch & \textbf{.680} $\uparrow$\tiny{$\pm$ .019} & .484 $\uparrow$\tiny{$\pm$ .024} & \textbf{.409} $\uparrow$\tiny{$\pm$ .03} & \underline{.504} \tiny{$\pm$ .019} & .335 $\uparrow$\tiny{$\pm$ .044} & .457 \tiny{$\pm$ .088} \\
\midrule
Eccentricity & .615 \tiny{$\pm$ .016} & .416 \tiny{$\pm$ .023} & .320 \tiny{$\pm$ .022} & .319 \tiny{$\pm$ .024} & .266 \tiny{$\pm$ .053} & .440 \tiny{$\pm$ .068} \\
Eccentricity + beamsearch & .640 $\uparrow$\tiny{$\pm$ .013} & .437 $\uparrow$\tiny{$\pm$ .025} & .368 $\uparrow$\tiny{$\pm$ .02} & .407 $\uparrow$\tiny{$\pm$ .026} & .243 \tiny{$\pm$ .04} & .366 \tiny{$\pm$ .072} \\
\midrule
EigVecDissimilarity & .628 \tiny{$\pm$ .014} & .454 \tiny{$\pm$ .028} & .325 \tiny{$\pm$ .026} & .314 \tiny{$\pm$ .027} & .373 \tiny{$\pm$ .035} & .456 \tiny{$\pm$ .071} \\
EigVecDissimilarity + beamsearch & .660 $\uparrow$\tiny{$\pm$ .016} & .460 $\uparrow$\tiny{$\pm$ .024} & .380 $\uparrow$\tiny{$\pm$ .025} & .394 $\uparrow$\tiny{$\pm$ .025} & .353 \tiny{$\pm$ .045} & .453 \tiny{$\pm$ .086} \\
\midrule
CocoaMSP & .667 \tiny{$\pm$ .019} & \underline{.492} \tiny{$\pm$ .019} & .385 \tiny{$\pm$ .025} & .320 \tiny{$\pm$ .025} & .378 \tiny{$\pm$ .028} & \textbf{.523} \tiny{$\pm$ .071} \\
CocoaMSP + beamsearch & \underline{.678} $\uparrow$\tiny{$\pm$ .018} & \textbf{.498} $\uparrow$\tiny{$\pm$ .028} & .391 $\uparrow$\tiny{$\pm$ .02} & .378 $\uparrow$\tiny{$\pm$ .029} & \underline{.385} $\uparrow$\tiny{$\pm$ .037} & \underline{.510} \tiny{$\pm$ .079} \\
\midrule
CocoaPPL & .665 \tiny{$\pm$ .015} & .478 \tiny{$\pm$ .019} & .388 \tiny{$\pm$ .035} & .397 \tiny{$\pm$ .024} & .353 \tiny{$\pm$ .038} & .484 \tiny{$\pm$ .081} \\
CocoaPPL + beamsearch & .667 $\uparrow$\tiny{$\pm$ .016} & .486 $\uparrow$\tiny{$\pm$ .021} & .387 \tiny{$\pm$ .036} & .437 $\uparrow$\tiny{$\pm$ .026} & .339 \tiny{$\pm$ .044} & .450 \tiny{$\pm$ .06} \\
\bottomrule
\end{tabular}
    }
    \label{tab:qa_qwen8b}
  \end{table}

  \begin{table}[h]
    \caption{PRR ($\uparrow$ is better) for 6 datasets with \qwenit. For each dataset, the top-1 method is \textbf{bold} and the second-best is \underline{underlined}. Beam-guided and probability-weighted variants are marked with $\uparrow$ when they improve over their multinomial-sampling baseline.}
    \centering
    \small
    \resizebox{\textwidth}{!}{
      \begin{tabular}{lcccccc}
\toprule
UQ Method & TriviaQA & \multirowcell{Web\\Questions} & CoQA & HotpotQA & \multirowcell{Common\\senceQA} & \multirowcell{ARC-\\Challenge} \\
\midrule
 \rowcolor{gray!20}
 \multicolumn{7}{c}{\textit{Baseline UQ methods}} \\
\midrule
Prob & .564 \tiny{$\pm$ .017} & .353 \tiny{$\pm$ .032} & .215 \tiny{$\pm$ .02} & .250 \tiny{$\pm$ .026} & .174 \tiny{$\pm$ .034} & .181 \tiny{$\pm$ .078} \\
MTE & .564 \tiny{$\pm$ .018} & .345 \tiny{$\pm$ .028} & .164 \tiny{$\pm$ .025} & .251 \tiny{$\pm$ .028} & .183 \tiny{$\pm$ .03} & .272 \tiny{$\pm$ .095} \\
Perplexity & .491 \tiny{$\pm$ .023} & .341 \tiny{$\pm$ .036} & .169 \tiny{$\pm$ .026} & .250 \tiny{$\pm$ .028} & .175 \tiny{$\pm$ .037} & .229 \tiny{$\pm$ .058} \\
CCP & .563 \tiny{$\pm$ .02} & .383 \tiny{$\pm$ .029} & .169 \tiny{$\pm$ .018} & .258 \tiny{$\pm$ .029} & .173 \tiny{$\pm$ .034} & .202 \tiny{$\pm$ .068} \\
SAR & .590 \tiny{$\pm$ .016} & .425 \tiny{$\pm$ .036} & .146 \tiny{$\pm$ .026} & .159 \tiny{$\pm$ .029} & .201 \tiny{$\pm$ .033} & .233 \tiny{$\pm$ .051} \\
P(True) & -.105 \tiny{$\pm$ .023} & -.222 \tiny{$\pm$ .035} & -.126 \tiny{$\pm$ .017} & .018 \tiny{$\pm$ .021} & -.083 \tiny{$\pm$ .03} & -.164 \tiny{$\pm$ .071} \\
Semantic Entropy & .597 \tiny{$\pm$ .016} & .404 \tiny{$\pm$ .034} & .214 \tiny{$\pm$ .022} & .231 \tiny{$\pm$ .026} & .174 \tiny{$\pm$ .041} & .176 \tiny{$\pm$ .08} \\
Lexical Similarity & .530 \tiny{$\pm$ .023} & .425 \tiny{$\pm$ .029} & .193 \tiny{$\pm$ .031} & .101 \tiny{$\pm$ .026} & .121 \tiny{$\pm$ .039} & .053 \tiny{$\pm$ .06} \\
EigValLaplacian & .626 \tiny{$\pm$ .015} & .417 \tiny{$\pm$ .04} & .196 \tiny{$\pm$ .026} & .083 \tiny{$\pm$ .024} & .134 \tiny{$\pm$ .031} & .134 \tiny{$\pm$ .066} \\
NumSemSets & .608 \tiny{$\pm$ .021} & \underline{.437} \tiny{$\pm$ .036} & .110 \tiny{$\pm$ .019} & .096 \tiny{$\pm$ .024} & .113 \tiny{$\pm$ .041} & .154 \tiny{$\pm$ .065} \\
\midrule
 \rowcolor{gray!20}
 \multicolumn{7}{c}{\textit{Consistency-based UQ: multinomial vs. beamsearch versions}} \\
\midrule
Dissimilarity & .588 \tiny{$\pm$ .017} & .382 \tiny{$\pm$ .03} & .165 \tiny{$\pm$ .02} & .187 \tiny{$\pm$ .025} & \textbf{.246} \tiny{$\pm$ .038} & \textbf{.394} \tiny{$\pm$ .072} \\
Dissimilarity + beamsearch & \underline{.637} $\uparrow$\tiny{$\pm$ .018} & .386 $\uparrow$\tiny{$\pm$ .026} & .269 $\uparrow$\tiny{$\pm$ .019} & .264 $\uparrow$\tiny{$\pm$ .026} & .213 \tiny{$\pm$ .031} & \underline{.362} \tiny{$\pm$ .083} \\
\midrule
Eccentricity & .565 \tiny{$\pm$ .019} & .367 \tiny{$\pm$ .034} & .167 \tiny{$\pm$ .025} & .125 \tiny{$\pm$ .023} & .150 \tiny{$\pm$ .026} & .132 \tiny{$\pm$ .078} \\
Eccentricity + beamsearch & .600 $\uparrow$\tiny{$\pm$ .016} & .392 $\uparrow$\tiny{$\pm$ .034} & \underline{.288} $\uparrow$\tiny{$\pm$ .029} & \underline{.291} $\uparrow$\tiny{$\pm$ .022} & .211 $\uparrow$\tiny{$\pm$ .035} & .285 $\uparrow$\tiny{$\pm$ .084} \\
\midrule
EigVecDissimilarity & .590 \tiny{$\pm$ .024} & .385 \tiny{$\pm$ .031} & .169 \tiny{$\pm$ .026} & .121 \tiny{$\pm$ .032} & .143 \tiny{$\pm$ .033} & .131 \tiny{$\pm$ .066} \\
EigVecDissimilarity + beamsearch & \textbf{.645} $\uparrow$\tiny{$\pm$ .016} & \textbf{.439} $\uparrow$\tiny{$\pm$ .032} & \textbf{.328} $\uparrow$\tiny{$\pm$ .019} & \textbf{.297} $\uparrow$\tiny{$\pm$ .017} & \underline{.242} $\uparrow$\tiny{$\pm$ .029} & .306 $\uparrow$\tiny{$\pm$ .058} \\
\midrule
CocoaMSP & .607 \tiny{$\pm$ .015} & .394 \tiny{$\pm$ .03} & .204 \tiny{$\pm$ .016} & .272 \tiny{$\pm$ .023} & .230 \tiny{$\pm$ .042} & .298 \tiny{$\pm$ .061} \\
CocoaMSP + beamsearch & .635 $\uparrow$\tiny{$\pm$ .02} & .404 $\uparrow$\tiny{$\pm$ .024} & .263 $\uparrow$\tiny{$\pm$ .023} & .282 $\uparrow$\tiny{$\pm$ .025} & .206 \tiny{$\pm$ .029} & .290 \tiny{$\pm$ .061} \\
\midrule
CocoaPPL & .581 \tiny{$\pm$ .02} & .389 \tiny{$\pm$ .032} & .179 \tiny{$\pm$ .024} & .272 \tiny{$\pm$ .022} & .232 \tiny{$\pm$ .032} & .309 \tiny{$\pm$ .082} \\
CocoaPPL + beamsearch & .609 $\uparrow$\tiny{$\pm$ .02} & .395 $\uparrow$\tiny{$\pm$ .031} & .233 $\uparrow$\tiny{$\pm$ .025} & .282 $\uparrow$\tiny{$\pm$ .026} & .207 \tiny{$\pm$ .03} & .299 \tiny{$\pm$ .084} \\
\bottomrule
\end{tabular}
    }
    \label{tab:qa_qwen8bit}
  \end{table}

\clearpage 

\section{Detailed Description of Uncertainty Quantification Methods}
\label{sec:appendix_methods}
  In this section, we describe the uncertainty quantification methods used in our experiments.

\textbf{Sequence Probability (Prob)} is the most straightforward approach to uncertainty quantification. We define it formally as the negative log-probability of the generating sequence:
  \begin{equation}
    U_{\text{SP}}(\yv \mid \xv) = - \log P(\yv \mid \xv).
    \label{eq:msp}
  \end{equation}

\textbf{Mean Token Entropy (MTE)} measures an average entropy of tokens in a sequence: 
  \begin{equation}
    U_{\text{MTE}}(\yv \mid \xv) = \frac{1}{L} \sum_{l = 1}^L \HC (y_l \mid \yv_{<l}, \xv),
    \label{eq:entropy}
  \end{equation}

  where $\HC(y_l \mid \yv_{<l}, \xv) = -\sum_{v} P(y_l = v \mid \yv_{<l}, \xv) \log P(y_l = v \mid \yv_{<l}, \xv)$.

\textbf{Perplexity} computes negative average log-likelihood of tokens in a sequence: 
  \begin{equation}
    U_{\text{PPL}}(\yv \mid \xv) = -\frac{1}{L} \log P(\yv \mid \xv),
    \label{eq:ppl}
  \end{equation}
  
\textbf{Claim Conditioned Probability (CCP)}, introduced in~\citep{fadeeva-etal-2024-fact}, measures uncertainty on a claim level by perturbing claim's tokens with alternative generations: 
  \begin{equation}
    U_{\text{CCP}}(C \mid \xv) = 1 - \prod_{j \in C} \text{CCP}(y_j \mid y_{<j}, \xv).
  \end{equation}

  Where $\text{CCP}(y_j \mid \yv_{<j}, \xv) =
  \frac{
  \sum_{k : \text{NLI}(y_j^k, y_j) = 'e'} P(y_j^k \mid \yv_{<j}, \xv)
  }{
  \sum_{k : \text{NLI}(y_j^k, y_j) \in \{'e', 'c'\}} P(y_j^k \mid \yv_{<j}, \xv)
  }$

\textbf{Shifting Attention to Relevance (SAR)} is a method combining TokenSAR and SentenceSAR, as introduced by~\citet{duan-etal-2024-shifting}. SentenceSAR is defined as follows:
  \begin{equation}
    U_\mathrm{SentSAR}(\xv) = -\frac{1}{M} \sum_{i = 1}^M \log \Bigl(p(\yv^{(i)} \mid \xv) + \frac{1}{t} \mathrm{R}_S (\yv^{(i)}, \xv)\Bigr),
  \end{equation}

  Here, $\mathrm{R}_S (\yv^{(j)}, \xv) \! = \sum_{k \neq j} s\bigl(\yv^{(j)}, \yv^{(k)}\bigr) p\bigl(\yv^{(k)} \mid \xv \bigr)$. To obtain SAR score, the generative probability $p(\yv \mid \xv)$ is replaced with relevance-reweighted probability on a sequence level.
  \textit{TokenSAR} is defined as:
  \begin{equation}    
    U_{\text{TokenSAR}}(\xv)  =
    - \sum_{l=1}^{L} \tilde{R}_T(y_l, \yv, \xv) \log P(y_l \mid \yv_{<l}, \xv),
  \label{eq:tokensar}
  \end{equation}
  where $R_T(\cdot)$ denotes some token relevance function and relevance weight for token $y_l$ is given by $\tilde{R}_T(y_k, \yv, \xv) = \frac{R_T(y_k, \yv, \xv)}{\sum_{l=1}^L R_T(y_l, \yv, \xv)}$ .

\textbf{P(True)}, introduced in~\citep{kadavath2022language}, evaluates the confidence in a generation by asking the model the original question and answer, then asking if it is true or false. We then use the negative log-probability of the token ``True'' as an uncertainty score.

  
\textbf{Lexical Similarity}, introduced in~\citep{fomicheva-etal-2020-unsupervised}, measures average pairwise similarity between $M$ sampled generations using some similarity function $s(\yv, \yv')$:  
  \begin{equation}
    U_{\text{LSRL}}(\xv) = 1 - \frac{2}{M(M-1)} \sum_{i < j} \text{s}\bigl(\yv^{(i)}, \yv^{(j)}\bigr).
  \label{eq:lsrl}
  \end{equation} 

\textbf{Number of Semantic Sets}, introduced in~\citep{lin2023generating}, estimates how many distinct meanings the model produces by clustering its outputs with an NLI model. Two answers are placed in the same cluster if they mutually entail each other more than they contradict and the final number of distinct clusters serves as an uncertainty score $U_{\text{NumSemSets}}$.

\textbf{Sum of Eigenvalues of Laplacian}, introduced in~\citep{lin2023generating}, constructs a similarity matrix among the sampled outputs and computes a uncertainty score from the eigenvalues of the Laplacian of that similarity matrix:
  \begin{equation}
    U_{\text{EigV}}(\xv) = \sum_{i = 1}^M \max\bigl(0, 1 - \lambda_i(\xv)\bigr).
  \end{equation}

\section{Computational Budget}
  All experiments were run on 2$\times$NVIDIA A100 (80 GB). Evaluating a single model across all six datasets took approximately 2 wall-clock days on this setup (4 GPU-days); with six models, this amounts to 12 wall-clock days (24 GPU-days). Additional ablations (sampling strategies, top-1 beam scoring, and other objectives) required a further 5 wall-clock days on the same hardware (10 GPU-days). In total, the study used about 34 GPU-days.

\section{The Usage of LLMs}
  In this study, large language models are examined primarily as the focus of analysis. For practical tasks such as programming and writing, we also make limited use of LLM-based assistants (e.g., ChatGPT) to support grammar correction and code debugging, with all usage carefully monitored by humans.

\end{document}